\documentclass[11pt]{article} 

\usepackage{fullpage}
\usepackage{authblk}
\usepackage{color,psfrag}
\usepackage{algorithm, algorithmic}
\usepackage{amsmath,amsbsy,amsfonts,amssymb,dsfont,units,amsthm}
\usepackage{tikz}
\usetikzlibrary{calc,shapes}
\usepackage{subfig} 
\usepackage{float}
\usepackage{cases}

\newtheorem{corollary}{Corollary}[section]

\newtheorem{lemma}{Lemma}[section]
\newtheorem{proposition}{Proposition}[section]
\newtheorem{theorem}{Theorem}[section]
\theoremstyle{definition}

\newcommand{\bc}{\begin{center}}
\newcommand{\ec}{\end{center}} 
\def\tcw{\textcolor{white}}
\def\st{{s.t.  }}

\usepackage{natbib}

\DeclareMathOperator{\Ber}{Ber}

\def\tcr{\textcolor{red}}
\def\tcb{\textcolor{blue}}
\def\tha{{\mbox{\tiny th}}}
\def\Ibb{{\mathbb I}}

\DeclareMathOperator{\thres}{Thres}
\DeclareMathOperator{\Var}{Var}
\DeclareMathOperator\Diag{Diag}

 \DeclareMathOperator*{\argmax}{arg\,max}
\newcommand\E{\mathbb{E}}
\newcommand\R{\mathbb{R}}
\renewcommand\t{{\scriptscriptstyle\top}}

\newcommand\inner[1]{\ensuremath{\left<#1\right>}}
\newcommand\norm[1]{\left\|#1\right\|}
\newcommand\bigO{O}

\def\t{{\scriptscriptstyle\top}}
\def\tl{\tilde}
\def\h{\hat}
\def\halpha{\widehat{\alpha}}
\def\tT{\tilde{T}}

\def\eps{\epsilon}
\def\veps{\varepsilon}
\def\simiid{{\overset{iid}{\sim}}}
\def\Pbb{{\mathbb P}}

\def\Ac{{\cal A}}
\def\Bc{{\cal B}}

\renewcommand\th[1]{\ensuremath{\theta_{#1}}}
\newcommand\hth[1]{\ensuremath{\hat{\theta}_{#1}}}
\newcommand\teps{\ensuremath{\tilde{\eps}}}
\newcommand\hv{\ensuremath{\hat{v}}}
\newcommand\hlambda{\ensuremath{\hat{\lambda}}}
\newcommand\lambdamin{\ensuremath{\lambda_{\min}}}
\newcommand\lambdamax{\ensuremath{\lambda_{\max}}}
\newcommand\deflate{\mathcal{E}}
 \newcommand\tlambdamin{\ensuremath{\tilde{\lambda}_{\min}}}

\DeclareMathOperator{\range}{range}

\DeclareMathOperator{\Triples}{T}

\newcommand\Dir{\operatorname{Dir}}

\newcommand{\bp}{\begin{psfrags}}
\newcommand{\ep}{\end{psfrags}}
\newcommand{\bprfof}{\begin{proof_of}}
\newcommand{\eprfof}{\end{proof_of}}
\newcommand{\bprf}{\begin{myproof}}
\newcommand{\eprf}{\end{myproof}}

\newenvironment{myproof}{\noindent{\em Proof:} \hspace*{1em}}{
    \hspace*{\fill} $\Box$ }
\newenvironment{proof_of}[1]{\noindent {\em Proof of #1: }}{\hspace*{\fill} $\Box$ }

\def\viz{{viz.,\ \/}}
\def\Ebb{{\mathbb E}}
\def\nn{\nonumber}
\def\beq{\begin{equation}}
\def\eeq{\end{equation}\noindent}
\def\beqn{\begin{eqnarray}}
\def\eeqn{\end{eqnarray} \noindent}
\def\beqnn{  \begin{eqnarray*}}
\def\eeqnn{\end{eqnarray*}  \noindent}
\def\bcase{  \begin{numcases}}
\def\ecase{\end{numcases}   \noindent}
 
\title{A Tensor Spectral Approach to Learning Mixed Membership Community Models}
\author[1]{Anima Anandkumar}
\author[2]{Rong Ge}
\author[3]{Daniel Hsu}
\author[3]{Sham M. Kakade}
\affil[1]{a.anandkumar@uci.edu, University of California, Irvine}
\affil[2]{rongge@cs.princeton.edu, Princeton University}
\affil[3]{dahsu/skakade@microsoft.com, Microsoft Research, New England}

\begin{document}

\maketitle

\begin{abstract}Community detection is the task of detecting hidden communities from observed interactions. Guaranteed community detection has so far been  mostly limited to models with non-overlapping communities such as the stochastic block model. In this paper, we remove this restriction, and provide guaranteed community detection for   a family of probabilistic network models with overlapping communities, termed as the mixed membership Dirichlet model, first introduced   by~\cite{ABFX08}. This model allows for nodes to have fractional memberships in multiple communities and assumes that the community memberships are drawn from a Dirichlet distribution. Moreover, it contains the stochastic block model as a special case. We propose a unified  approach to learning these models via a   tensor spectral decomposition method. Our estimator is based on  low-order  moment tensor of the observed network, consisting of  $3$-star counts. Our learning method is fast and is based on   simple linear algebraic operations, e.g. singular value decomposition and tensor power iterations. We provide guaranteed recovery of community memberships and model parameters and present a careful finite sample analysis of our learning method. As an important special case, our results  match the best known scaling requirements for the (homogeneous) stochastic block model.
\end{abstract}


\section{Introduction}\label{sec:intro}


Studying communities forms an integral part of social network analysis. A community generally refers to a group of individuals with shared interests (e.g. music, sports), or relationships (e.g. friends, co-workers).  Community formation in social networks has been studied by many sociologists, e.g.~\citep{moreno1934shall,lazarsfeld1954friendship,mcpherson2001birds,currarini2009economic}, starting with the seminal work of~\cite{moreno1934shall}. They posit various  factors such as {\em homophily}\footnote{The term {\em homophily} refers to the tendency that individuals belonging to the same community tend to connect more than individuals in different communities.} among the individuals to be responsible for community formation.
Various probabilistic and non-probabilistic  network models attempt to explain community formation. In addition, they also attempt to quantify   interactions and the extent of overlap between different communities,   relative sizes among the communities, and various other network properties. Studying such community models are also of interest in other domains, e.g. in biological networks.

While  there exists a vast literature  on community  models, learning these models is typically challenging, and various heuristics such as Markov Chain Monte Carlo (MCMC) or variational expectation maximization (EM) are employed in practice. Such heuristics tend to  scale poorly for large networks. On the other hand, community models with guaranteed learning methods tend to be restrictive. A popular class of probabilistic models, termed as {\em stochastic blockmodels}, have been widely studied and enjoy strong theoretical learning guarantees, e.g.~\citep{white1976social,holland1983stochastic,fienberg1985statistical,wang1987stochastic,snijders1997estimation,McSherry01}. On the other hand, they posit that an individual belongs to a single community, which does not hold in most real settings~\citep{palla2005uncovering}.

In this paper, we consider a class of mixed membership community models, originally introduced by~\cite{ABFX08}, and recently  employed by~\cite{xing2010state} and \cite{gopalan2012scalable}. The model has been shown to be effective in many real-world settings, but so far, no learning approach exists with provable guarantees. In this paper, we provide a novel learning approach for learning these mixed membership models and prove that these methods succeed under a set of sufficient conditions.

The mixed membership community model of~\cite{ABFX08} has a number of attractive properties.  It  retains many of the convenient properties of the  stochastic
block model. For instance, conditional independence of the edges is assumed, given the community memberships of the nodes in the network.  At the same time, it  allows for communities to overlap, and for   every individual to  be fractionally involved in different communities.  It includes the stochastic block model as a special case (corresponding to zero overlap among the different communities). This enables us to compare our learning guarantees  with existing works for stochastic block models and also study how the extent of overlap among different communities affects  the learning performance.

 %

\subsection{Summary of Results}\label{sec:summary}
We now summarize the main contributions of this paper. We propose a novel approach for learning mixed membership   community models of~\cite{ABFX08}. Our approach is a method of moments estimator  and incorporates tensor spectral decomposition.  We provide guarantees for our approach under a set of sufficient conditions. Finally, we compare our results to existing ones for the special case of  the stochastic block model, where nodes belong to a single community.

\paragraph{Learning Mixed Membership Models: }
We present a tensor-based   approach  for learning   the mixed membership stochastic block  model (MMSB) proposed by~\citet{ABFX08}.  In the MMSB model, the community membership vectors are drawn from the Dirichlet distribution, denoted by $\Dir(\alpha)$, where $\alpha$ is known the Dirichlet concentration vector. Employing the Dirichlet distribution results in sparse community memberships in certain regimes of $\alpha$, which is realistic. The   extent of overlap between different communities under the MMSB model   is  controlled (roughly) via a single scalar parameter, $\alpha_0:=\sum_i \alpha_i$, where  $\alpha:=[\alpha_i]$ is the Dirichlet concentration vector.    When $\alpha_0\to 0$, the mixed membership model degenerates to a stochastic block model and we have non-overlapping communities.

We  propose a unified tensor-based learning method for the MMSB model and establish
recovery guarantees under a set of sufficient conditions. These conditions are in  in terms of    the network size $n$, the  number of communities $k$, extent of
community overlaps (through $\alpha_0$), and the average edge connectivity across various communities. Below, we present an overview of our guarantees for the special case of equal sized communities (each of size $n/k$) and homogeneous community connectivity: let  $p$  be the probability for any intra-community edge to occur, and $q$ be the probability for any inter-community edge. Let $\Pi$ be the community membership matrix, where $\Pi^{(i)}$ denotes the $i^{\tha}$ row, which is the vector of membership weights of the nodes for the $i^{\tha}$ community. Let $P$ be the community connectivity matrix such that $P(i,i)=p$ and $P(i,j)=q$ for $i\neq j$.
 
\begin{theorem}[Main Result]\label{thm:intro}For an MMSB model with network size $n$, number of communities $k$, connectivity parameters $p, q$ and community overlap parameter $\alpha_0$, when\footnote{The notation $\tl{\Omega}(\cdot), \tl{O}(\cdot)$ denotes $\Omega(\cdot), O(\cdot)$ up to poly-log factors.}
\beq \label{eqn:condspecial-intro}n = \tl{\Omega}(k^2 (\alpha_0+1)^2), \qquad \frac{p-q}{\sqrt{p}} =
\tl{\Omega}\left(\frac{(\alpha_0+1)k}{n^{1/2}}\right),\eeq our estimated  community membership matrix $\h{\Pi}$ and the edge connectivity matrix $\h{P}$ satisfy with high probability (w.h.p.)\begin{align}\label{eqn:epspi-intro}
\frac{\veps_{\pi,\ell_1}}{n}:= \frac{1}{n}\max_{i\in [n]} \| \h{\Pi}^i - \Pi^i\|_1&=   \tl{O}\left(\frac{(\alpha_0+1)^{3/2}\sqrt{ p}}{(p-q)\sqrt{n}}\right)\\ \label{eqn:epsP-intro}
\veps_P :=\max_{i,j\in [k]}|\h{P}_{i,j} - P_{i,j}|
&=  \tl{O}\left(\frac{(\alpha_0+1)^{3/2}k\sqrt{p}}{\sqrt{n}}\right)
.\end{align}Further, our support estimates $\h{S}$ satisfy
w.h.p.,
\beq \label{eqn:support-intro}\Pi(i,j)\geq \xi \Rightarrow \h{S}(i,j)=1 \quad\mbox{ and }\quad \Pi(i,j) \leq \frac{\xi}{2} \Rightarrow \h{S}(i,j)=0, \quad \forall i\in [k],j\in [n],\eeq where $\Pi$ is the true community membership matrix and the threshold is chosen as $\xi=\Omega(\eps_P)$.
\end{theorem}

The complete details are in Section~\ref{sec:sample}. We first provide some intuitions behind the sufficient conditions  in \eqref{eqn:condspecial-intro}.  We require the network size $n$ to be large enough compared to the number of communities $k$, and for the separation $p-q$ to be large enough, so that the learning method  can distinguish the different communities. This is natural since a zero separation $(p=q)$ implies that the communities are indistinguishable.
Moreover,  we see that the scaling requirements become more stringent   as $\alpha_0$ increases. This is intuitive since it is harder to learn communities with more overlap, and we quantify this scaling. For the Dirichlet distribution, it can be shown that the number of ``significant'' entries is roughly $O(\alpha_0)$ with high probability, and in many settings of practical interest, nodes may have significant memberships in only a few communities, and thus, $\alpha_0$ is a constant (or growing slowly) in many instances.

In addition, we quantify the error bounds for estimating various parameters of the mixed membership model in \eqref{eqn:epspi-intro} and \eqref{eqn:epsP-intro}. These errors decay under the sufficient conditions in \eqref{eqn:condspecial-intro}.
Lastly, we establish zero-error guarantees for support recovery in \eqref{eqn:support-intro}: our learning method correctly identifies (w.h.p) all the significant memberships of a node and also identifies the set of communities where a node does not have a strong presence, and we quantify the threshold $\xi$ in Theorem~\ref{thm:intro}. Further, we present the results for a general (non-homogeneous) MMSB model in Section~\ref{sec:generalresults}.

\paragraph{Identifiability Result for the  MMSB model: }
A byproduct of our analysis yields novel identifiability results for the MMSB model based on low order graph moments. We establish that the MMSB model is identifiable, given access to third order moments in the form of  counts of $3$-star subgraphs, i.e. a star subgraph consisting of three leaves, for each triplet of leaves, when the community connectivity matrix $P$ is full rank. Our learning approach involves decomposition of this third order tensor.  Previous identifiability results required access to high order moments and were limited to the stochastic block model setting; see Section~\ref{sec:related} for details.

\paragraph{Implications on Learning Stochastic Block Models: } 
Our results have implications for learning stochastic block models, which is a special case of the MMSB model with $\alpha_0\to 0$. 
In this case, the sufficient conditions in   \eqref{eqn:condspecial-intro} reduce to  \beq \label{eqn:condspecialblock-intro}n = \tl{\Omega}(k^2), \qquad  \frac{p-q}{\sqrt{p}} =\tl{\Omega}\left(\frac{k}{n^{1/2}}\right),\eeq
The   scaling requirements in \eqref{eqn:condspecialblock-intro} match with the best known bounds\footnote{There are many methods which achieve the best known scaling for $n$ in \eqref{eqn:condspecialblock-intro}, but have worse scaling for the separation $p-q$. This includes variants of the spectral clustering method, e.g.~\cite{CCT12}. See~\cite{ChenSanghaviXu} for a detailed comparison.} (up to poly-log factors) for learning uniform stochastic block models  and  were previously achieved by~\cite{ChenSanghaviXu}  via convex optimization involving semi-definite programming (SDP).  In contrast, we propose  an iterative non-convex approach involving tensor power iterations and linear algebraic techniques, and obtain similar guarantees.
For a detailed comparison of learning guarantees under various methods for learning (homogeneous) stochastic block models, see~\cite{ChenSanghaviXu}.

%
%

Thus, we establish learning guarantees explicitly in terms of the extent of overlap among the different communities for general MMSB models. Many real-world networks  involve sparse community memberships and  the total number of communities is typically much larger than the extent of membership of a single individual, e.g. hobbies/interests of a person, university/company networks that a person belongs to, the set of transcription factors regulating a gene, and so on. Thus, we see that in this regime of practical interest, where $\alpha_0=\Theta(1)$, the scaling requirements in
\eqref{eqn:condspecial-intro} match those for the stochastic block model in \eqref{eqn:condspecialblock-intro} (up to polylog factors) without any degradation in learning performance. Thus, we establish that learning community models with sparse community memberships is akin to learning stochastic block models and we present a unified approach and analysis for learning these models.

To the best of our knowledge, this work is the first to establish polynomial time learning guarantees for probabilistic network models with overlapping communities and we provide a fast and an iterative learning approach through linear algebraic techniques and tensor power iterations. While the results of this paper are mostly limited to a theoretical analysis of the tensor method for learning overlapping communities, we note recent results which show that this method (with improvements and modifications) is very accurate in practice on real datasets from social networks, and is   scalable to graphs with millions of nodes~\citep{AnandkumarEtal:communityimplementation13}.

\subsection{Overview of Techniques}

We now describe the main techniques employed in our learning approach and in establishing the recovery guarantees.

\paragraph{Method of moments and subgraph counts: }
We propose an efficient learning algorithm based on   low order  moments, \viz counts of small subgraphs. Specifically, we employ a third-order tensor which   counts the number of $3$-stars in the observed network. A $3$-star is a star graph with three leaves  (see figure~\ref{fig:star}) and we count the occurrences of such $3$-stars across different partitions. We establish that (an adjusted) $3$-star count tensor has a simple relationship with the model parameters, when the network is drawn from a mixed membership model. We propose a multi-linear transformation using edge-count matrices (also termed as the process of whitening), which reduces the  problem of learning mixed membership models to the \emph{canonical polyadic (CP) decomposition} 
of  an  orthogonal symmetric tensor, for which tractable decomposition exists, as described below.
Note that the decomposition of  a general tensor into  its rank-one components is referred to as its CP decomposition~\citep{kolda2009tensor} and is in general NP-hard~\citep{hillar2009most}.  However, the decomposition is tractable in the special case of an orthogonal symmetric tensor considered here.



\paragraph{Tensor spectral decomposition via power iterations: }Our tensor decomposition method is based on the popular power iterations (e.g. see~\cite{AFHKL12}). It is a simple iterative method to compute the stable eigen-pairs of a tensor. In this paper, we propose   various modifications to the basic power method  to strengthen the recovery guarantees under perturbations. For instance, we introduce  adaptive deflation techniques (which involves subtracting out the eigen-pairs previously estimated). Moreover, we initialize the tensor power method with (whitened) neighborhood vectors from the observed network, as opposed to random initialization. In the regime, where the community overlaps are small,  this leads to an improved performance. Additionally, we incorporate thresholding as a post-processing operation, which again, leads to improved guarantees for sparse community memberships, i.e., when the overlap among different communities is small. We theoretically establish that all these modifications lead to improvement in performance guarantees and we discuss comparisons with the basic power method in Section~\ref{sec:compare}.

\paragraph{Sample analysis: }We establish that our learning approach correctly recovers the model parameters and the community memberships of all nodes under exact moments. We then carry out a careful analysis of the empirical graph moments, computed using the network observations. We establish tensor concentration bounds and also control the perturbation of the various quantities used by our learning algorithm via matrix Bernstein's inequality~\citep[thm. 1.4]{tropp2012user} and other inequalities.  We impose the scaling requirements in \eqref{eqn:condspecial-intro} for various concentration bounds to hold.




\subsection{Related Work}\label{sec:related}
There is extensive work on modeling communities and various algorithms and heuristics for discovering them. We mostly limit our focus to works with  theoretical guarantees.



\paragraph{Method of moments: }The method of moments approach dates back to~\cite{Pearson94} and has been applied for learning various community models. Here, the moments correspond to counts of various subgraphs in the network. They   typically   consist of aggregate   quantities, e.g.,  number of star subgraphs, triangles etc. in the network. For instance,~\cite{bickel2011method} analyze the moments of a stochastic block model and  establish that the subgraph counts of certain structures, termed as   ``wheels'' (a family of trees), are sufficient  for identifiability under some natural non-degeneracy conditions. In contrast,   we establish that  moments up to third order (corresponding to edge and $3$-star   counts) are sufficient for identifiability of the stochastic block model, and also more generally, for the mixed membership Dirichlet model. We employ   subgraph count tensors,  corresponding to the number of subgraphs (such as stars) over a set of labeled vertices, while the work of~\cite{bickel2011method}  considers only  aggregate (i.e. scalar) counts. Considering tensor moments  allows us to use simple subgraphs (edges and $3$ stars) corresponding to low order moments, rather than more complicated graphs (e.g. wheels considered by~\cite{bickel2011method}) with larger number of nodes, for learning the community model.

The method of moments is also relevant for the family of random graph models termed as {\em exponential random graph models}~\citep{holland1981exponential,frank1986markov}. Subgraph counts of fixed graphs such as stars and triangles serve as sufficient statistics for these models. However,
parameter estimation given the subgraph counts is in general NP-hard, due to the normalization constant in the likelihood (the partition function) and the model suffers from degeneracy issues; see~\cite{rinaldo2009geometry,chatterjee2011estimating} for detailed discussion. In contrast, we establish in this paper that the   mixed membership model  is amenable to simple estimation methods through linear algebraic operations and tensor power iterations using  subgraph counts of $3$-stars.

\paragraph{Stochastic block models: }Many algorithms provide learning guarantees for stochastic block models.
For a detailed comparison of these methods, see the recent work by~\cite{ChenSanghaviXu}. A popular method is based on spectral clustering~\citep{McSherry01}, where community memberships are inferred through projection onto the spectrum of the Laplacian matrix (or its variants). This method is fast and easy to implement (via singular value decomposition). There are many variants of this method, e.g. the work of~\cite{CCT12} employs normalized Laplacian matrix to handle degree heterogeneities. In contrast, the work of~\cite{ChenSanghaviXu}  uses convex optimization techniques via semi-definite programming learning block models.  For a detailed comparison of learning guarantees under various methods for learning stochastic block models, see~\cite{ChenSanghaviXu}.

\paragraph{Non-probabilistic approaches: }
The classical approach to community detection tries to directly
exploit the properties of the graph to define communities, without assuming a probabilistic model.
\cite{girvan2002community} use betweenness to remove edges until only communities are left. However,~\cite{bickel2009nonparametric} show  that these algorithms are (asymptotically) biased and
that using modularity scores can lead to the discovery of an incorrect
community structure, even for large graphs.~\cite{jalali2011clustering} define community structure as the structure that satisfies the maximum number of edge constraints
(whether two individuals like/dislike each other). However, these   models   assume that every individual belongs to  a single
community.

Recently, some non-probabilistic approaches have been introduced with overlapping community models by~\cite{AGSS12} and~\cite{BBBCT12}. The analysis of~\cite{AGSS12} is mostly limited to dense graphs (i.e. $\Theta(n^2)$ edges for a $n$ node graph), while our analysis provides learning guarantees for much sparser graphs (as seen by the scaling requirements in  \eqref{eqn:condspecial-intro}). Moreover, the running time of the method of~\cite{AGSS12} is {\em quasipolynomial time} (i.e. $O(n^{\log n})$) for the general case, and is based on a combinatorial learning approach. In contrast, our learning approach is based on simple linear algebraic techniques and the running time is a low-order polynomial (roughly it is $O(n^2 k)$ for a $n$ node network with $k$ communities under a serial computation model and $O(n+k^3)$ under a parallel computation model).  The work of~\cite{BBBCT12} assumes endogenously formed communities, by constraining the fraction of edges within a community compared to the outside. They provide a polynomial time algorithm for finding all such ``self-determined'' communities and the running time is $n^{O(\log 1/\alpha)/\alpha}$, where $\alpha$ is the fraction of edges within a self-determined community, and this bound is improved to linear time when $\alpha>1/2$. On the other hand, the running time of our algorithm is mostly independent of the parameters of the assumed model, (and is roughly $O(n^2 k)$).
Moreover, both these works are limited to homophilic models, where there are more edges within each community, than between any two different communities. However, our learning approach is not limited to this setting and also does not assume homogeneity in edge connectivity across different communities (but instead it makes  probabilistic assumptions on community formation). In addition, we  provide improved guarantees for homophilic models by considering additional post-processing steps in our algorithm. Recently,~\cite{ACKS12} provide an algorithm for approximating the parameters
of an Euclidean log-linear model in polynomial time. However, there setting is considerably different than the one in this paper.

\paragraph{Inhomogeneous random graphs, graph limits and weak regularity lemma: }Inhomogeneous random graphs have been analyzed in a variety of settings
(e.g.,~\cite{bollobas2007phase,lovasz2009very}) and are   generalizations of the stochastic block model. Here,   the probability of an edge between any two nodes is characterized by a general function (rather than by a $k\times k$ matrix as in the stochastic block model with $k$ blocks). Note that the mixed membership model considered in this work is a special instance of this general framework.  These models arise as the limits of   convergent (dense) graph sequences and for this reason, the functions are also termed as ``graphons'' or graph limits~\citep{lovasz2009very}. A deep result in this context is the regularity lemma and its variants. The weak regularity lemma   proposed by~\cite{FK99},  showed that any convergent dense graph can be approximated by a stochastic block model. Moreover, they propose an algorithm to learn such a block model based on the so-called $d_2$ distance. The $d_2$ distance between two nodes measures similarity with respect to their ``two-hop'' neighbors and the block model is obtained by thresholding the $d_2$ distances. However, the method is limited to learning block models and not overlapping communities.



%

\paragraph{Learning Latent Variable Models (Topic Models): }The community models considered in this paper are closely related to the probabilistic topic models~\citep{blei2012probabilistic}, employed for text modeling and document categorization. Topic models posit  the occurrence  of words in a corpus of documents, through the presence of multiple latent topics in each document.
Latent Dirichlet allocation (LDA) is perhaps the most popular topic model, where the topic mixtures are assumed to be drawn from the Dirichlet distribution. In each document, a topic mixture is drawn from the Dirichlet distribution, and the words are drawn in a conditional independent manner, given the topic mixture.
The mixed membership community model considered in this paper can be interpreted as a  generalization of the LDA model, where a node in the community model can function both as a document and a word. For instance, in the directed community model, when the outgoing links of a node are considered, the node functions as a document, and its outgoing neighbors can be interpreted as the words occurring  in that document. Similarly, when the incoming links of a node in the network are considered, the node can be interpreted as a word, and its incoming links,  as documents  containing that particular word.
In particular, we establish that certain graph moments under the mixed membership model have similar structure as the observed word moments under the LDA model.
This allows us to leverage the recent developments from Anandkumar et. al.~\citep{AnandkumarHsuKakade:COLT12,AFHKL12,AGHKT12} for
learning  topic models,  based on the method of moments. These works establish guaranteed learning using second- and third-order observed moments   through linear algebraic and tensor-based techniques. In particular, in this paper, we exploit the tensor power iteration method of~\cite{AGHKT12}, and propose additional improvements to obtain stronger recovery guarantees. Moreover, the sample analysis is quite different (and more challenging) in the community setting, compared to topic models analyzed in~\cite{AnandkumarHsuKakade:COLT12,AFHKL12,AGHKT12}. We clearly spell out the similarities and differences between the community model and other latent variable models in Section~\ref{sec:compare}.

\paragraph{Lower Bounds: }
The work of~\cite{FGRVX12} provides lower bounds on the complexity of statistical algorithms,  and shows that  for cliques of size $O(n^{1/2-\delta})$, for any constant $\delta>0$, at least $n^{\Omega(\log \log n)}$ queries are needed to find the cliques. There
are works relating the hardness of finding hidden cliques and the use of   higher order moment tensors for this purpose.~\cite{FK08} relate the problem of finding a hidden clique to finding the top eigenvector of the third order tensor, corresponding to the maximum spectral norm. ~\cite{BV09} extend the result to arbitrary $r^{\tha} $-order tensors and the cliques have to be size $\Omega(n^{1/r})$ to enable  recovery from $r^{\tha}$-order moment tensors in a $n$ node network. However, this problem (finding the top eigenvector of a tensor) is known to be NP-hard in general~\citep{hillar2009most}. Thus, tensors are useful for  finding smaller hidden cliques in network (albeit by solving a computationally hard problem). In contrast, we consider tractable tensor decomposition through reduction to orthogonal  tensors (under the scaling requirements of \eqref{eqn:condspecial-intro}),  and our learning method is a fast and an iterative approach based on tensor power iterations and linear algebraic operations.~\cite{mossel2012stochastic} provide lower bounds on the separation $p-q$, the  edge connectivity between intra-community and inter-community, for identifiability of communities in  stochastic block models in the sparse regime (when $p, q\sim n^{-1}$), when the number of communities is a constant $k = O(1)$. Our method achieves the lower bounds on separation of edge connectivity up to poly-log factors.

\paragraph{Likelihood-based Approaches to Learning MMSB: }Another class of approaches for learning MMSB models are based on optimizing the observed likelihood.
 Traditional approaches such as Gibbs sampling or expectation maximization (EM) can be too expensive apply in practice for MMSB models. Variational approaches which optimize the so-called evidence lower bound~\citep{hoffman2012stochastic,gopalan2012scalable}, which is a lower bound on the marginal likelihood of the observed data (typically by applying a mean-field approximation), are efficient for practical implementation. Stochastic versions of the variational approach provide even further gains in efficiency and are state-of-art practical learning methods for MMSB models~\citep{gopalan2012scalable}. However, these methods lack theoretical guarantees; since they optimize a bound on the likelihood, they are not guaranteed to recover the underlying communities consistently. A recent work~\citep{celisse2012consistency} establishes  consistency of maximum likelihood and variational estimators for  stochastic block models, which are special cases of the MMSB model. However, it is not known if the results extend to general MMSB models. Moreover, the framework of~\citet{celisse2012consistency} assumes a fixed number of communities and growing network size, and provide only asymptotic consistency guarantees. Thus, they do not allow for high-dimensional settings, where the parameters of the learning problem also grow as the observed dimensionality grows.  In contrast, in this paper, we allow for the number of communities to grow, and provide precise constraints on the scaling bounds for consistent estimation under finite samples. It is an open problem to obtain such bounds for maximum likelihood and variational estimators.
 On the practical side, a recent work deploying the tensor approach proposed in this paper by~\citet{AnandkumarEtal:communityimplementation13} shows that the tensor approach is more than an order of magnitude faster in recovering the communities than the variational approach, is scalable to networks with millions of nodes, and also has better accuracy in recovering the communities.

\section{Community Models and Graph Moments}

\subsection{Community Membership Models}\label{sec:model}

In this section, we describe the mixed membership community model based on Dirichlet priors for the community draws by the individuals.
We first introduce the special case of the popular stochastic block model, where each node
belongs to a single community.

\paragraph{Notation: }We consider networks with $n$ nodes and let $[n] :=
\{1,2,\dotsc,n\}$. Let $G$   be the $\{0,1\}$ adjacency\footnote{Our analysis can easily be extended to weighted adjacency matrices with bounded entries.} matrix for the random network and let $G_{A,B}$ be the submatrix of $G$ corresponding to rows $A \subseteq
[n]$ and columns $B \subseteq [n]$. We consider models with $k$ underlying (hidden) communities. For node $i$, let $\pi_i\in \R^k$ denote its \emph{community membership vector}, i.e., the vector is supported on the communities to which the node belongs. In the special  case of the popular stochastic block model described below, $\pi_i$ is a basis coordinate vector, while the more general mixed membership model relaxes this assumption and a node can be in multiple communities with fractional memberships.
Define $\Pi:=[\pi_1| \pi_2| \dotsb | \pi_n]\in \R^{k \times n}$. and let
$\Pi_A:=[\pi_i : i \in A] \in \R^{k \times |A|}$ denote the set of column
vectors restricted to   $A\subseteq [n]$. For a matrix $M$, let $(M)_i$ and $(M)^i$ denote its $i^{\tha}$ column and row respectively. For a matrix $M$ with singular value decomposition (SVD) $M = UD V^\top$, let $(M)_{k-svd}:= U \tl{D} V^\top$ denote the $k$-rank SVD of $M$, where $\tl{D}$ is limited to top-$k$ singular values of $M$. Let $M^\dagger$ denote the Moore–Penrose pseudo-inverse of $M$.
Let $\Ibb(\cdot)$ be the indicator function. Let $\Diag(v)$ denote a diagonal matrix with diagonal entries given by a vector $v$.   We use the term high probability to mean with probability $1-n^{-c}$ for any constant $c>0$.

\paragraph{Stochastic block model (special case):}
In this model, each individual is independently assigned to a single community, chosen at random: each node $i$ chooses community $j$ independently with probability
$\halpha_j$, for $i\in [n], j \in [k]$,  and we assign $\pi_i=e_j$ in this case, where $e_j \in
\{0,1\}^k$ is the $j^{\tha}$  coordinate basis vector.
Given the community assignments $\Pi$, every directed\footnote{We limit our discussion to directed networks in
this paper, but note that the results also hold for undirected community models, where $P$ is a symmetric matrix, and an edge $(u,v)$ is formed with probability $\pi_u^\top P \pi_v = \pi_v^\top P \pi_u$.} edge in the network is independently drawn: if node $u$ is in community $i$ and  node $v$ is in community $j$ (and $u\neq v$),
then the probability of having the edge $(u,v)$ in the network is   $P_{i,j}$. Here,   $P\in [0,1]^{k \times k}$ and we refer to it as  the \emph{community
connectivity matrix}.  This implies that given the community membership vectors $\pi_u$ and $\pi_v$, the probability of an edge from $u$ to $v$ is  $\pi_u^\top P \pi_v$ (since when $\pi_u =e_i$ and $\pi_v=e_j$, we have $\pi_u^\top P \pi_v=P_{i,j}$.). The stochastic model has been extensively studied and can be learnt efficiently through various methods, e.g. spectral clustering~\citep{McSherry01}, convex optimization~\citep{ChenSanghaviXu}. and so on. Many of these methods rely on conditional independence assumptions of the edges in the block model for guaranteed learning.


\paragraph{Mixed membership model: }We now consider the extension of the stochastic block model which
allows for an individual to belong to multiple communities and yet preserves some of the convenient independence assumptions of the block model. In this model, the community membership vector $\pi_u$ at node $u$ is a probability vector, i.e., $\sum_{i\in [k]} \pi_u(i)=1$, for all $u\in [n]$. Given the community membership vectors, the generation of the edges is identical to the block model:  given   vectors $\pi_u$ and $\pi_v$, the probability of an edge from $u$ to $v$ is  $\pi_u^\top P \pi_v$, and the edges are independently drawn.
This  formulation allows for the nodes to be in multiple communities, and at the same time, preserves the conditional independence of the edges, given the community memberships of the nodes.


\paragraph{Dirichlet prior for community membership: }The only aspect left to be specified for the mixed membership model is the distribution from which the community membership vectors $\Pi$ are drawn. We consider the popular setting of~\cite{ABFX08},   where  the community vectors $\{\pi_u\}$ are i.i.d.~draws from the Dirichlet distribution, denoted by $\Dir(\alpha)$, with   parameter vector $\alpha \in
\R_{>0}^k$. The probability density  function of the Dirichlet distribution is given by
\begin{equation}
\Pbb[\pi] = \frac{\prod_{i=1}^k \Gamma(\alpha_i)}{\Gamma(\alpha_0)} \prod_{i=1}^k
\pi_i^{\alpha_i-1}, \quad \pi\sim \Dir(\alpha), \alpha_0:=\sum_i \alpha_i\label{eqn:dirichlet},
\end{equation} where $\Gamma(\cdot)$ is the Gamma function and the ratio of the Gamma function serves as the normalization constant.

%

The Dirichlet distribution is widely employed for specifying priors in Bayesian statistics, e.g. latent Dirichlet allocation~\citep{BNJ03}. The Dirichlet distribution  is the conjugate prior of the multinomial distribution which makes it attractive for Bayesian inference.

Let $\halpha$ denote the normalized parameter vector $\alpha/\alpha_0$, where
$\alpha_0:=\sum_i \alpha_i$.
In particular, note that $\halpha$ is a probability vector: $\sum_i \halpha_i=1$.
Intuitively, $\halpha$ denotes the relative expected sizes of the
communities (since
$ \Ebb[n^{-1}\sum_{u\in [n]}  \pi_u[i]] = \halpha_i$).
Let $\halpha_{\max}$ be the largest entry in $\halpha$, and
$\halpha_{\min}$ be the smallest entry. Our learning guarantees will depend on these parameters.

The stochastic block model is a   limiting case of the mixed
membership model when the Dirichlet parameter  is  $\alpha =
\alpha_0 \cdot \halpha$, where the probability vector $\halpha$ is held fixed and $\alpha_0\to 0$. In the other extreme when $\alpha_0 \to \infty$, the Dirichlet
distribution becomes peaked around a single point, for instance, if $\alpha_i\equiv c $ and $c\to \infty$, the Dirichlet distribution is peaked at $k^{-1} \cdot\vec{1}$, where $\vec{1}$ is the all-ones vector. Thus, the parameter $\alpha_0$ serves as a  measure of the average sparsity of the Dirichlet draws or equivalently, of how concentrated  the Dirichlet measure is along the different coordinates.
This in effect, controls the extent of overlap among different communities.

\paragraph{Sparse regime of Dirichlet distribution: }When the Dirichlet parameter vector satisfies\footnote{The assumption that the Dirichlet distribution be in the sparse regime is not strictly needed. Our results can be extended to general Dirichlet distributions, but with worse scaling requirements on the network size $n$ for guaranteed learning.} $\alpha_i<1$, for all $i\in [k]$, the Dirichlet distribution $\Dir(\alpha)$ generates  ``sparse'' vectors with high probability\footnote{Roughly the number of entries in $\pi$ exceeding a threshold $\tau$ is at most $O(\alpha_0\log(1/\tau)) $ with high probability, when $\pi \sim \Dir(\alpha)$.}; see~\cite{Matus-sparse} (and in the extreme case of the block model where $\alpha_0\to 0$, it generates $1$-sparse vectors). Many real-world settings involve sparse community membership and  the total number of communities is typically much larger than the extent of membership of a single individual, e.g. hobbies/interests of a person, university/company networks that a person belongs to, the set of transcription factors regulating a gene, and so on.
Our learning guarantees are limited to the sparse regime of the Dirichlet model.



\subsection{Graph Moments Under Mixed Membership Models}

Our approach for learning a mixed membership community model relies on the form of the graph moments\footnote{We interchangeably use the term first order moments for edge counts and third order moments for $3$-star counts.} under the mixed membership model. We now describe the specific graph moments used by our learning algorithm (based on $3$-star  and edge counts) and provide explicit forms for the   moments, assuming draws from a mixed membership model.

\begin{figure}\centering{\bp\psfrag{x}[c]{$x$}
\psfrag{u}[c]{$u$}\psfrag{v}[c]{$v$}\psfrag{w}[c]{$w$}
\psfrag{A}[c]{\tcb{$A$}}\psfrag{B}[c]{\tcb{$B$}}
\psfrag{C}[c]{\tcb{$C$}}\psfrag{X}[c]{\tcr{$X$}}
\includegraphics[width=1.7in]{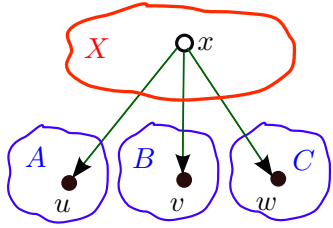}\ep}
\caption{Our moment-based learning algorithm uses  $3$-star count tensor from set $X$ to  sets $A, B, C$ (and the roles of the sets are interchanged to get various estimates). Specifically,   $\Triples$ is a third order tensor, where $\Triples(u,v,w)$ is the normalized  count  of the $3$-stars with $u,v,w$ as leaves over all  $x\in X$.}\label{fig:star}
\end{figure}

\subsubsection*{Notations}
Recall that $G$ denotes the adjacency matrix and that $G_{X,A}$ denotes the submatrix corresponding to edges going from $X$ to $A$. Recall that $P\in [0,1]^{k\times k}$ denotes the community connectivity matrix.
Define
\begin{equation}
F:= \Pi^\top P^\top= [\pi_1| \pi_2| \dotsb | \pi_n]^\top P^\top.
\label{eqn:F}
\end{equation}For a subset $A \subseteq [n]$ of individuals, let $F_A \in \R^{|A| \times
k}$ denote the submatrix of $F$ corresponding to nodes in $A$,
\emph{i.e.}, $F_A:= \Pi^\top_A P^\top$. We will subsequently show that $F_A$ is linear map which takes any community vector $\pi_i$ as input and outputs  the corresponding neighborhood vector $G^\top_{i, A}$ in expectation.


Our learning algorithm uses  moments up to the
third-order, represented as a tensor.
A third-order tensor $T$ is a three-dimensional array whose $(p,q,r)$-th
entry denoted by $T_{p,q,r}$.
The symbol $\otimes$ denotes the standard Kronecker product: if $u$, $v$, $w$ are three vectors, then
\begin{equation}
(u\otimes v\otimes w)_{p,q,r}:=u_p\cdot v_q \cdot w_r.
\label{eqn:otimes}
\end{equation}
A tensor of the   form $u\otimes v \otimes w$ is referred to as a rank-one tensor. The decomposition of  a general tensor into a sum of its rank-one components is referred to as \emph{canonical polyadic (CP) decomposition}~\cite{kolda2009tensor}. We will subsequently see that the graph moments can be expressed as a tensor and that the CP decomposition of the graph-moment tensor yields the model parameters and the community vectors under the mixed membership community model.


\subsubsection{Graph moments  under Stochastic Block Model}

We first analyze the graph moments in the special case of a stochastic block model (i.e., $\alpha_0=\sum_i \alpha_i \to0$ in the Dirichlet prior in \eqref{eqn:dirichlet}) and then extend it to general mixed membership model. We provide explicit expressions for the graph moments corresponding to edge counts and $3$-star counts.
 We  later establish in Section~\ref{sec:algorithm} that these moments   are sufficient to learn the community memberships of the nodes and the model parameters of the block model.

\paragraph{$3$-star counts: }
The primary quantity of interest is a third-order tensor which   counts the number of $3$-stars. A $3$-star is a star graph with three leaves $\{a,b,c\}$  and we refer to the internal node $x$ of the star as its ``head'', and denote the structure by $x\rightarrow \{a,b,c\}$ (see figure~\ref{fig:star}).
We  partition the network into four\footnote{For sample complexity analysis, we require dividing the graph into more than four partitions to deal with statistical dependency issues, and we outline it in Section~\ref{sec:algorithm}.} parts and consider $3$-stars such that each node in the $3$-star belongs to a different partition. This is necessary to obtain a simple form of the moments, based on the conditional independence assumptions of the block model, see Proposition~\ref{prop:blockmoments}. Specifically, consider\footnote{To establish our theoretical guarantees, we assume that the partitions $A,B,C,X$ are randomly chosen and are of size $\Theta(n)$.} a partition $A,B,C,X$ of the network. We count the number of $3$-stars
 from $X$ to $A,B,C$ and our quantity of interest is \beq\label{eqn:triples}
\Triples_{X\rightarrow \{A,B,C\}}  := \frac{1}{|X|}\sum_{i\in X}[G^\top_{i,A}\otimes G^\top_{i,B}\otimes G^\top_{i,C}],\eeq where $\otimes$ is the Kronecker product, defined in \eqref{eqn:otimes}
 and $G_{i,A}$ is the row vector supported on the set of neighbors of $i$ belonging to set $A$. $\Triples\in \R^{|A|\times |B|\times |C|}$ is a third order tensor, and an element of the tensor is given by\beq \Triples_{X\rightarrow \{A,B,C\}}(a,b,c) = \frac{1}{|X|}\sum_{x\in X} G(x,a)G(x,b)G(x,c), \quad \forall a\in A, b\in B, c\in C,\eeq
which is the normalized count of the number of $3$-stars with leaves $a,b,c$ such that its  ``head''  is in set $X$.

We now relate the tensor $\Triples_{X\rightarrow \{A,B,C\}}$ to the parameters of the stochastic block model, \viz the community connectivity matrix $P$ and the community probability vector $\halpha$, where $\halpha_i$ is the probability of choosing community $i$.

\begin{proposition}[Moments in Stochastic Block Model]\label{prop:blockmoments}Given partitions $A,B,C, X$, and $F:=\Pi^\top P^\top$, where $P$ is the community connectivity   matrix and $\Pi$ is the matrix of community membership vectors, we have \begin{align}
\Ebb[G_{X,A}^\top|\Pi_A, \Pi_X] & = F_A \Pi_X,\label{eqn:GXA}\\
\label{eqn:triplestensorform} \Ebb[\Triples_{X\rightarrow \{A,B,C\}} |\Pi_A, \Pi_B, \Pi_C]&=\sum_{i\in [k]}  \halpha_i (F_A)_i \otimes (F_B)_i \otimes (F_C)_i, \end{align} where $\halpha_i$ is the probability for a node to select  community $i$. \end{proposition}

\paragraph{Remark 1 (Linear model): }In Equation~\eqref{eqn:GXA}, we see that the edge generation occurs under a linear model, and more precisely, the matrix $F_A\in \R^{|A|\times k  }$ is a linear map which takes a community vector $\pi_i \in \R^k$ to a neighborhood vector $G_{i,A}^\top \in \R^{|A|}$ in expectation.

\paragraph{Remark 2 (Identifiability under third order moments): }
Note the form of the $3$-star count tensor $\Triples$ in \eqref{eqn:triplestensorform}. It provides a CP decomposition of  $\Triples$ since each term in the summation, \viz $\halpha_i (F_A)_i \otimes (F_B)_i \otimes (F_C)_i$, is a rank one tensor. Thus, we can learn the matrices $F_A, F_B, F_C$ and the vector $\halpha$ through CP decomposition of tensor $\Triples$. Once these parameters are learnt, learning the communities is straight-forward under exact moments: by exploiting \eqref{eqn:GXA}, we find  $\Pi_X$ as \[\Pi_X =
F_A^\dagger\cdot\Ebb[G_{X,A}^\top|\Pi_A, \Pi_X] .\]
Similarly, we can consider another tensor consisting of $3$-stars from $A$ to $X,B,C$, and obtain matrices $F_X, F_B$ and $F_C$ through a CP decomposition, and so on. Once we obtain  matrices $F$  and $\Pi$ for the entire set of nodes in this manner, we can obtain the community connectivity matrix $P$, since $F:=\Pi^\top P^\top$. Thus, in principle, we are able to learn all the model parameters ($\halpha$ and $P$) and the community membership matrix $\Pi$ under the stochastic block model, given exact moments. This establishes identifiability of the model given moments up to third order and forms a high-level approach for learning the communities.  When only samples are available, we establish that the empirical versions are close to the exact moments considered above, and we modify the basic learning approach to obtain robust guarantees. See Section~\ref{sec:algorithm} for details.


\paragraph{Remark 3 (Significance of conditional independence relationships): }The main property exploited in proving the tensor form in \eqref{eqn:triplestensorform} is the   conditional-independence assumption under the stochastic block model: the realization of the edges in each  $3$-star, say in  $x\rightarrow \{a,b,c\}$, is conditionally  independent given the community membership vector $\pi_x$,  when $x\neq a\neq b\neq c$. This is because the community membership vectors $\Pi$ are assumed to be drawn independently at the different nodes and the edges are drawn independently given the community vectors.
Considering $3$-stars from $X$ to $A,B,C$ where   $X, A, B, C$ form a partition ensures that this conditional independence is satisfied for all the $3$-stars in tensor $\Triples$. \\

\bprf Recall that the probability of an edge from $u$ to $v$ given $\pi_u, \pi_v$ is
\[ \Ebb[G_{u,v}|\pi_u, \pi_v]= \pi_u^\top P \pi_v = \pi_v^\top P^\top \pi_u = F_v \pi_u,\] and   $\Ebb[G_{X,A}|\Pi_A,\Pi_X] = \Pi_X^\top P \Pi_A = \Pi_X^\top F_A^\top$ and thus \eqref{eqn:GXA} holds.
For the tensor form, first consider an element of the tensor, with $a\in A, b\in B, c\in C$,\[\Ebb\left[\Triples_{X\rightarrow \{A,B,C\}}(a,b,c)| \pi_a, \pi_b, \pi_c, \pi_x \right]= \frac{1}{|X|}\sum_{x\in X} F_a \pi_x \cdot F_b \pi_x \cdot F_c \pi_x,  \]
The equation follows from the conditional-independence assumption of the edges (assuming $a\neq b\neq c$). Now taking expectation over the nodes in $X$, we have
\begin{align*} \Ebb\left[\Triples_{X\rightarrow \{A,B,C\}}(a,b,c)| \pi_a, \pi_b, \pi_c\right]&=
\frac{1}{|X|}\sum_{x\in X}\Ebb\left[ F_a \pi_x \cdot F_b \pi_x \cdot F_c \pi_x| \pi_a, \pi_b, \pi_c\right] \\ &=\Ebb\left[ F_a \pi \cdot F_b \pi \cdot F_c \pi|  \pi_a, \pi_b, \pi_c\right]\\ &= \sum_{j\in [k]} \halpha_j (F_a)_j \cdot (F_b)_j \cdot (F_c)_j,\end{align*}where the last step follows from the fact that $\pi= e_j$ with probability $\halpha_j$ and the result holds when $x\neq a,b,c$. Recall that $(F_a)_j$ denotes the $j^{\tha}$ column of $F_a$ (since $F_a e_j = (F_a)_j$). Collecting all the elements of the tensor, we obtain the desired result. \eprf


\subsubsection{Graph Moments under Mixed Membership Dirichlet Model}

We now analyze the graph moments for the general mixed membership Dirichlet model. Instead of the raw moments (i.e. edge and $3$-star counts), we consider modified moments to obtain similar expressions as in the case of the stochastic block model.

Let $\mu_{X\rightarrow A}\in \R^{|A|}$ denote a vector which gives the normalized count of edges from $X$ to $A$: \beq \label{eqn:mean} \mu_{X\rightarrow A}:= \frac{1}{|X|}\sum_{i\in X}[ G_{i, A}^\top]. \eeq
We now define a modified adjacency matrix\footnote{To compute the modified moments $G^{\alpha_0}$, and $\Triples^{\alpha_0}$, we need to know the value of the scalar $\alpha_0:=\sum_i \alpha_i$, which is the concentration parameter of the Dirichlet distribution and is a measure of the extent of overlap between the communities. We assume its knowledge here.}
$ G_{X,A}^{\alpha_0}$ as \beq G_{X,A}^{\alpha_0}:= \left(\sqrt{\alpha_0+1} G_{X, A} - (\sqrt{\alpha_0+1}-1)  \vec{1}\mu_{X\rightarrow A}^\top\right).\label{eqn:GXAalphadef}\eeq In the special case of the stochastic block model $(\alpha_0\to0)$, $ G_{X,A}^{\alpha_0} = G_{X,A}$ is the submatrix of the adjacency matrix $G$. Similarly, we define modified third-order statistics,
\begin{align}\nn & \Triples_{X\rightarrow \{A,B,C\}} ^{\alpha_0} := (\alpha_0+1)(\alpha_0+2)\Triples_{X\rightarrow \{A,B,C\}} + 2\,\alpha_0^2\, \mu_{X\rightarrow A}\otimes \mu_{X\rightarrow B}\otimes \mu_{X\rightarrow C}\\ & -  \frac{\alpha_0(\alpha_0+1)}{|X|}\sum_{i\in X}\left[G^\top_{i,A}\otimes G^\top_{i,B}\otimes \mu_{X \rightarrow C} +G^\top_{i,A}\otimes \mu_{X\rightarrow B}\otimes G^\top_{i,C}
+\mu_{X\rightarrow A}\otimes G^\top_{i,B}\otimes G^\top_{i,C}\right]
 \label{eqn:triplesalpha},\end{align} and it reduces to (a scaled version of) the $3$-star count   $\Triples_{X \rightarrow \{A,B,C\}}$ defined in \eqref{eqn:triples} for the stochastic block model  $(\alpha_0\to0)$.
 The modified adjacency matrix and the $3$-star count tensor can be viewed as a form of ``centering'' of the raw moments which simplifies the expressions for the moments.
 The following relationships hold between the modified  graph moments $G^{\alpha_0}_{X,A}$, $\Triples^{\alpha_0}$ and the model parameters $P$ and $\halpha$ of the mixed membership model.

\begin{proposition}[Moments in Mixed Membership Model]\label{prop:dirichletmoments}Given partitions $A,B,C, X$ and $G^{\alpha_0}_{X,A}$ and $\Triples^{\alpha_0}$, as in \eqref{eqn:GXAalphadef} and \eqref{eqn:triplesalpha}, normalized Dirichlet concentration vector $\halpha$,  and $F:=\Pi^\top P^\top$, where $P$ is the community connectivity   matrix and $\Pi$ is the matrix of community memberships, we have
\begin{align}\label{eqn:GXAalpha} \Ebb[(G^{\alpha_0}_{X, A})^\top|\Pi_A, \Pi_X] &= F_A \Diag(\halpha^{1/2}) \Psi_X,  \\ \label{eqn:triplestensorformalpha} \Ebb[\Triples^{\alpha_0}_{X\rightarrow \{A,B,C\}} |\Pi_A, \Pi_B, \Pi_C]&=\sum_{i=1}^k \halpha_i (F_A)_i \otimes (F_B)_i \otimes (F_C)_i, \end{align} where $(F_A)_i$ corresponds to $i^{\tha}$ column of $F_A$ and $\Psi_X$
relates to the community membership matrix $\Pi_X$ as
\[\Psi_X:=\Diag(\halpha^{-1/2})
\left(\sqrt{\alpha_0+1}\Pi_X - (\sqrt{\alpha_0+1}-1) \left(\frac{1}{|X|}\sum_{i\in X }\pi_i\right) \vec{1}^\top\right).\] Moreover, we have that
\begin{equation} \label{eqn:nocorrelation}|X|^{-1}\Ebb_{\Pi_X}[\Psi_X \Psi_X^\top] =   I.\end{equation}
\end{proposition}

\paragraph{Remark 1: }The $3$-star count tensor $T^{\alpha_0}$ is carefully chosen so that the CP decomposition of the tensor directly yields  the matrices  $F_A, F_B, F_C$ and $\halpha_i$, as in the case of the stochastic block model.
 Similarly, the modified adjacency matrix  $(G^{\alpha_0}_{X,A})^\top$ is carefully chosen to eliminate second-order correlation in the Dirichlet distribution and we have that $|X|^{-1}\Ebb_{\Pi_X}[\Psi \Psi^\top] =  I$ is the identity matrix. These properties will be exploited by our learning algorithm in Section~\ref{sec:algorithm}. \\

\paragraph{Remark 2: }Recall that $\alpha_0$ quantifies the extent of overlap among the communities. The computation of the modified moment $T^{\alpha_0}$ requires the knowledge of $\alpha_0$, which is assumed to be known.   Since this is a scalar quantity, in practice, we can easily tune this parameter via cross validation.

\bprf The proof is on lines of Proposition~\ref{prop:blockmoments} for stochastic block models $(\alpha_0\to 0)$ but more involved due to the form of Dirichlet moments. Recall  $\Ebb[G_{i,A}^\top|\pi_i, \Pi_A] = F_A \pi_i$ for a mixed membership model, and $\mu_{X\rightarrow A} := \frac{1}{|X|} \sum_{i\in X} G_{i,A}^\top$, therefore $\Ebb[\mu_{X\rightarrow A}|\Pi_A, \Pi_X] = F_A
\left(\frac{1}{|X|}\sum_{i\in X }\pi_i\right) \vec{1}^\top$. Equation (\ref{eqn:GXAalpha}) follows directly.
For Equation (\ref{eqn:nocorrelation}), we note the Dirichlet moment, $\Ebb[\pi\pi^\top] = \frac{1}{\alpha_0+1} \Diag(\halpha)
+ \frac{\alpha_0}{\alpha_0+1} \halpha \halpha^\top$, when $\pi\sim \Dir(\alpha)$ and
\begin{align*}|X|^{-1}\Ebb[\Psi_X \Psi_X^\top] & =  \Diag(\halpha^{-1/2}) \left[(\alpha_0+1)\Ebb[\pi\pi^\top] + ( - 2\sqrt{\alpha_0+1}(\sqrt{\alpha_0+1}-1)
\right.
\\
& \left.\quad + (\sqrt{\alpha_0+1}-1)^2) \Ebb[\pi]\Ebb[\pi]^\top\right]\Diag(\halpha^{-1/2}) \\
& = \Diag(\halpha^{-1/2})  \left(\Diag(\halpha) + \alpha_0 \halpha \halpha^\top + (-\alpha_0) \halpha\halpha^\top\right)\Diag(\halpha^{-1/2}) \\
& =  I.
\end{align*}
On lines of the proof of Proposition~\ref{prop:blockmoments} for the block model,  the expectation in \eqref{eqn:triplestensorformalpha} involves multi-linear map of the expectation of the tensor products $\pi\otimes \pi \otimes \pi$ among other terms. Collecting  these terms, we have that
\begin{align*}&
(\alpha_0+1)(\alpha_0+2)\E[\pi\otimes \pi\otimes \pi] - (\alpha_0)(\alpha_0+1) (\E[\pi\otimes \pi\otimes \E[\pi]]
\\+&\E[\pi\otimes \E[\pi]\otimes \pi]+\E[\E[\pi]\otimes \pi\otimes \pi]) + 2\alpha_0^2\E[\pi]\otimes \E[\pi]\otimes \E[\pi]
\end{align*}
is a diagonal tensor, in the sense that its $(p,p,p)$-th entry is $\halpha_p$, and its $(p,q,r)$-th entry is 0 when $p,q,r$
are not all equal. With this, we  have \eqref{eqn:triplestensorformalpha}.\eprf\\

Note the nearly identical forms of the graph moments for the stochastic block model in  \eqref{eqn:GXA},  \eqref{eqn:triplestensorform}
 and for the general mixed membership model in \eqref{eqn:GXAalpha},  \eqref{eqn:triplestensorformalpha}. In other words, the modified moments $G^{\alpha_0}_{X,A}$ and $\Triples^{\alpha_0}$ have similar relationships to underlying parameters as the raw moments in the case of the stochastic block model.  This enables us to use a unified learning approach for the two models, outlined in the next section.

\section{Algorithm for Learning Mixed Membership Models}\label{sec:algorithm}

The simple form of the graph moments  derived in the previous section is now utilized to recover the community vectors $\Pi$ and  model parameters $P, \halpha$ of the mixed membership model. The method is based on the so-called tensor power method, used to obtain a tensor decomposition. We first outline the basic tensor decomposition method below and then demonstrate how the method can be adapted to learning using the graph moments at hand. We first analyze the simpler case when exact moments are available in Section~\ref{sec:algoexact} and then extend the method to handle empirical moments computed from the network observations in Section~\ref{sec:algosamples}.

\subsection{Overview of Tensor Decomposition Through Power Iterations}

In this section, we review the basic method for tensor decomposition  based on power iterations  for a special class of tensors, \viz symmetric orthogonal tensors. Subsequently, in Section~\ref{sec:algoexact} and \ref{sec:algosamples}, we modify  this method  to learn the mixed membership model from graph moments, described in the previous section. For details on the tensor power method, refer to~\cite{AFHKL12,SIMAX-080148-Tensor-Eigenvalues}.

Recall that a third-order tensor $T$ is a three-dimensional array and
 we use $T_{p,q,r}$ to denote the $(p,q,r)$-th entry of the tensor $T$. The standard symbol $\otimes$ is used
to denote the Kronecker product, and $ (u\otimes v\otimes w)$  is a rank one tensor.  The decomposition of a tensor into its rank one components is called the CP decomposition.

\paragraph{Multi-linear maps: }We can view a tensor $T\in \R^{d \times d \times d}$ as a multilinear map in the following sense: for a
set of matrices $\{ V_i \in \R^{d \times m_i} : i \in [3] \}$, the
$(i_1,i_2, i_3)$-th entry in the three-way array representation of
$T(V_1,V_2,V_3) \in \R^{m_1 \times m_2 \times  m_3}$ is
\begin{equation*} \label{eqn:multilinear}
[ T(V_1, V_2,V_3) ]_{i_1,i_2,i_3}
\ := \
\sum_{j_1,j_2,j_3 \in [d]}
T_{j_1,j_2,j_3} \
[V_1]_{j_1, i_1} \ [V_2]_{j_2, i_2} \    [V_3]_{j_3, i_3}
.
\end{equation*}The term multilinear map arises from the fact that the above map is linear in each of the coordinates, e.g. if we replace $V_1$ by $a V_1+ b W_1$ in the above equation, where $W_1$ is a matrix of appropriate dimensions, and $a,b$ are any scalars, the output is a linear combination of the outputs under $V_1$ and $W_1$ respectively.
We will use the above notion of multi-linear transforms to describe various tensor operations. For instance, $T(I,I,v)$ yields a matrix, $T(I,v,v)$, a vector, and $T(v,v,v)$, a scalar.

\paragraph{Symmetric tensors and orthogonal decomposition: }
A special class of tensors are the symmetric tensors $T\in \R^{d\times d \times d}$ which are invariant to permutation of the array indices. Symmetric tensors have CP decomposition of the form \beq\label{eqn:decompsymm} T= \sum_{i\in [r]} \lambda_i v_i\otimes v_i\otimes v_i = \sum_{i\in [r]}\lambda_i v_i^{\otimes 3},\eeq where $r$ denotes the tensor CP rank and we use the notation $v_i^{\otimes 3}:= v_i\otimes v_i\otimes v_i$. It is convenient to first analyze methods for   decomposition of symmetric tensors and we then extend them to the general case of asymmetric tensors.

Further, a   sub-class of symmetric tensors are those which possess a decomposition into orthogonal components, i.e.  the vectors $v_i\in \R^d$ are orthogonal to one another in the above decomposition in \eqref{eqn:decompsymm} (without loss of generality, we assume that vectors $\{v_i\}$ are orthonormal in this case). An orthogonal decomposition implies that the tensor rank $r\leq d$ and there are tractable methods for recovering the rank-one components in this setting. We limit ourselves to this setting in this paper.

\paragraph{Tensor eigen analysis: }For symmetric tensors $T$ possessing an orthogonal decomposition of the form in \eqref{eqn:decompsymm}, each pair $(\lambda_i, v_i)$, for $i \in [r]$, can be interpreted as an eigen-pair for the tensor $T$, since
\[ T(I,v_i, v_i) = \sum_{j \in [r]}\lambda_j\inner{v_i,v_j}^2v_j= \lambda_i v_i, \quad \forall i\in[r],\]due to the fact that $\inner{v_i, v_j}=\delta_{i,j}$. Thus, the vectors $\{ v_i\}_{i\in [r]}$ can be interpreted as  fixed points of the map
\beq\label{eqn:basicpower} v\mapsto \frac{T(I,v,v)}{\|T(I,v,v)\|},\eeq where $\|\cdot\|$ denotes the  spectral norm (and $\|T(I,v,v)\|$ is a vector norm), and is used to normalize the vector $v$ in \eqref{eqn:basicpower}.

\paragraph{Basic tensor power iteration method: }A straightforward approach to computing the orthogonal decomposition of a symmetric tensor is to iterate according to the fixed-point map in \eqref{eqn:basicpower} with an  arbitrary initialization vector. This is referred to as the tensor power iteration method. Additionally, it is  known  that the vectors $\{v_i\}_{i\in [r]}$ are the only {\em stable} fixed points of the map in \eqref{eqn:basicpower}. In other words, the set of initialization vectors which converge to vectors other than $\{v_i\}_{i\in [r]}$ are of measure zero. This ensures that we obtain the correct set of vectors through power iterations and that no spurious    answers are obtained. See~\cite[Thm. 4.1]{AGHKT12} for details. Moreover, after an approximately fixed point is obtained (after many power iterations), the estimated eigen-pair can be subtracted out (i.e., {\em deflated}) and subsequent vectors can be similarly obtained through power iterations. Thus, we can obtain all the stable eigen-pairs $\{\lambda_i, v_i\}_{i\in [r]}$ which are the components of the orthogonal tensor decomposition. The method needs to be suitably modified when the tensor $T$ is perturbed (e.g. as in the case when empirical moments are used) and we discuss it in Section~\ref{sec:algosamples}.

\subsection{Learning Mixed Membership Models Under Exact Moments}\label{sec:algoexact}
We first describe the learning approach when exact moments are available. In Section~\ref{sec:algosamples}, we   suitably modify the approach to handle perturbations, which are introduced when  only empirical moments are available.

We now employ the tensor power method described above to obtain a CP decomposition of the graph moment tensor $\Triples^{\alpha_0}$ in \eqref{eqn:triplesalpha}.
We first describe a ``symmetrization'' procedure to convert the graph moment tensor $\Triples^{\alpha_0}$  to a symmetric orthogonal tensor through a multi-linear transformation of   $\Triples^{\alpha_0}$.  We then employ the power method to obtain a symmetric orthogonal decomposition. Finally, the original CP decomposition is obtained  by reversing the multi-linear transform of the symmetrization procedure. This yields a guaranteed method for obtaining the decomposition of graph moment tensor $\Triples^{\alpha_0}$ under exact moments. We note that this symmetrization approach has been earlier employed in other contexts, e.g. for learning hidden Markov models~\cite[Sec. 3.3]{AGHKT12}.

\paragraph{Reduction of the graph-moment tensor  to symmetric  orthogonal form (Whitening): }
Recall from Proposition~\ref{prop:dirichletmoments} that   the modified $3$-star count tensor $\Triples^{\alpha_0}$ has a CP decomposition as
\[ \Ebb[\Triples^{\alpha_0} |\Pi_A, \Pi_B, \Pi_C]=\sum_{i=1}^k \halpha_i (F_A)_i \otimes (F_B)_i \otimes (F_C)_i.\]We now describe a symmetrization procedure to convert $\Triples^{\alpha_0}$ to a symmetric orthogonal tensor through a multi-linear transformation
  using the  modified adjacency matrix $G^{\alpha_0}$, defined in \eqref{eqn:GXAalphadef}.
Consider the singular value decomposition (SVD) of the modified adjacency matrix $G^{\alpha_0}$ under exact moments: \[|X|^{-1/2}\Ebb[(G^{\alpha_0}_{X, A})^\top|\Pi] = U_A D_A V_A^\top.\] Define $ W_A := U_A D_A^{-1},$ and similarly  define $W_B$ and $W_C$ using the corresponding matrices $G_{X,B}^{\alpha_0}$ and $G_{X,C}^{\alpha_0}$ respectively. Now define \beq\label{eqn:tildeW-exact} R_{A,B}:=\frac{1}{|X|} W_B^\top \Ebb[(G^{\alpha_0}_{X, B})^\top|\Pi]\cdot\Ebb[(G^{\alpha_0}_{X, A})|\Pi] W_A, \quad  \tl{W}_B:=W_B R_{A,B},\eeq and similarly define $\tl{W}_C$.
We establish that a multilinear transformation (as defined in \eqref{eqn:multilinear}) of the graph-moment tensor $\Triples^{\alpha_0}$ using matrices $W_A, \tl{W}_B,$ and $\tl{W}_C$ results in a symmetric orthogonal form.
\begin{lemma}[Orthogonal Symmetric Tensor] \label{lemma:orthogonal}
Assume that the matrices $F_A, F_B ,F_C$ and $\Pi_X$ have rank $k$, where $k$ is the number of communities. We have an orthogonal symmetric tensor form for the modified $3$-star count tensor $\Triples^{\alpha_0}$ in \eqref{eqn:triplesalpha} under a multilinear transformation using matrices $W_A, \tl{W}_B,$ and $\tl{W}_C$:
\beq \label{eqn:reduceortho} \Ebb[\Triples^{\alpha_0}(W_A, \tl{W}_B,\tl{W}_C)|\Pi_A, \Pi_B, \Pi_C] = \sum_{i\in [k]}  \lambda_i(\Phi)_i^{\otimes 3}\in \R^{k\times k \times k},\eeq where $\lambda_i := \halpha_i^{-0.5}$ and $\Phi\in \R^{k\times k}$ is an orthogonal  matrix,  given by \beq \label{eqn:tilF}\Phi := W_A^\top F_A \Diag(\halpha^{0.5}).\eeq
\end{lemma}

\paragraph{Remark 1: } Note that the matrix $W_A$   orthogonalizes  $F_A$ under exact moments, and  is   referred to as a {\em whitening matrix}. Similarly, the matrices $\tl{W}_B=R_{A,B} W_B$ and $\tl{W}_C=R_{A,C} W_C$ consist of whitening matrices $W_B$ and $W_C$, and in addition, the matrices $R_{A,B}$ and $R_{A,C}$  serve to symmetrize the tensor. We can interpret $\{\lambda_i, (\Phi)_i\}_{i\in [k]}$ as the stable eigen-pairs of the transformed tensor (henceforth, referred to as the {\em whitened and symmetrized tensor}).

\paragraph{Remark 2: }
The full   rank assumption on matrix $F_A=\Pi_A^\top P^\top\in \R^{|A| \times  k}$ implies that $|A|\geq k$, and similarly  $|B|,|C|,|X|\geq k$.  Moreover, we require the community connectivity matrix $P \in \R^{k \times k}$ to be of full rank\footnote{In the work of~\cite{McSherry01}, where spectral clustering for stochastic block models is analyzed, rank deficient $P$ is allowed as long as the neighborhood vectors generated by any pair of communities are sufficiently different. On the other hand, our method requires $P$ to be full rank. We argue that this is a mild restriction since we allow for mixed memberships while~\cite{McSherry01} limit to the stochastic block model.} (which is a natural non-degeneracy condition). In this case, we can reduce the graph-moment tensor $\Triples^{\alpha_0}$ to a $k$-rank   orthogonal symmetric tensor, which has a unique   decomposition.
This implies that the mixed membership model    is identifiable using $3$-star and edge count moments,  when the network size $n=|A|+|B|+|C|+|X| \geq 4k$, matrix $P$ is full rank  and the community membership matrices $\Pi_A, \Pi_B, \Pi_C, \Pi_X$ each have rank $k$. On the other hand, when only empirical moments are available, roughly, we require the network size $n =\Omega(k^2(\alpha_0+1)^2)$ (where $\alpha_0:=\sum_i \alpha_i$ is related to the extent of overlap between the communities) to provide guaranteed learning of the community membership and model parameters. See Section~\ref{sec:sample} for a detailed sample analysis.\\

\bprf
Recall that the modified adjacency matrix $G^{\alpha_0}$ satisfies \begin{align*}&\Ebb[(G^{\alpha_0}_{X,A})^\top|\Pi_A, \Pi_X]= F_A \Diag(\halpha^{1/2})\Psi_X.\\ &\Psi_X:=\Diag(\halpha^{-1/2})
\left(\sqrt{\alpha_0+1}\Pi_X - (\sqrt{\alpha_0+1}-1) \left(\frac{1}{|X|}\sum_{i\in X }\pi_i\right) \vec{1}^\top\right).\end{align*} From the definition of $\Psi_X$ above, we see that it has rank $k$ when $\Pi_X$ has rank $k$.  Using the Sylvester's rank inequality, we have that the rank of $F_A\Diag(\halpha^{1/2})\Psi_X$ is at least $2k-k=k$. This implies that the whitening matrix $W_A$ also has rank $k$.
Notice that
\begin{align*}|X|^{-1} W_A^\top  \Ebb[(G^{\alpha_0}_{X, A})^\top|\Pi]\cdot\Ebb[(G^{\alpha_0}_{X, A})|\Pi]  W_A = D_A^{-1} U_A^\top U_A D_A^2 U_A^\top U_A D_A^{-1} = I\in \R^{k\times k},\end{align*} or in other words, $|X|^{-1} M M^\top =I$, where $M:= W_A^\top F_A \Diag(\halpha^{1/2}) \Psi_X$. We now have that
\begin{align*}I=|X|^{-1}\Ebb_{\Pi_X}\left[M M^\top\right] &= |X|^{-1}W_A^\top F_A \Diag(\halpha^{1/2}) \Ebb[\Psi_X\Psi_X^\top]\Diag(\halpha^{1/2}) F_A^\top W_A \\ &= W_A^\top F_A \Diag(\halpha) F_A^\top W_A ,
\end{align*}since $|X|^{-1} \Ebb_{\Pi_X}[\Psi_X \Psi_X^\top]=I$ from \eqref{eqn:nocorrelation}, and we use the fact that the sets $A$ and $X$ do not overlap. Thus, $W_A$ whitens $F_A\Diag(\halpha^{1/2})$ under exact moments (up on taking expectation over $\Pi_X$) and the columns of $W_A^\top F_A \Diag(\halpha^{1/2})$ are orthonormal.
Now note from the definition of $\tl{W}_B$ that\[ \tl{W}_B^\top \Ebb[(G^{\alpha_0}_{X,B})^\top|\Pi] = W_A^\top \Ebb[(G^{\alpha_0}_{X,A})^\top|\Pi],\] since $W_B$ satisfies
\[|X|^{-1} W_B^\top  \Ebb[(G^{\alpha_0}_{X, B})^\top|\Pi]\cdot\Ebb[(G^{\alpha_0}_{X, B})|\Pi]  W_B = I,\]
and similar result holds for $\tl{W}_C$. The final result  in \eqref{eqn:reduceortho} follows by taking expectation of tensor $\Triples^{\alpha_0}$ over $\Pi_X$.
\eprf\\

\paragraph{Overview of the learning approach under exact moments: }With the above result in place, we are now ready to describe the high-level approach for learning the mixed membership model under exact moments. First, symmetrize the graph-moment tensor $\Triples^{\alpha_0}$ as described above and then apply the tensor power method described in the previous section. This enables us to obtain the vector of eigenvalues $\lambda:=\halpha^{-1/2}$  and the  matrix of eigenvectors $\Phi=W_A^\top F_A \Diag(\halpha^{0.5})$  using  tensor power iterations. We can then recover the community membership vectors of set $A^c$ (i.e., nodes not in set $A$) under exact moments as
\[ \Pi_{A^c}  \leftarrow \Diag(\lambda)^{-1}\Phi^\top W_A^\top \Ebb[G_{A^c,A}^\top|\Pi],\] since $\Ebb[G_{A^c,A}^\top|\Pi]= F_A \Pi_{A^c}$ (since $A$ and $A^c$ do not overlap) and $ \Diag(\lambda)^{-1}\Phi^\top W_A^\top = \Diag(\halpha) F_A^\top W_A W_A^\top$ under exact moments. In order to recover the community membership vectors of set $A$, \viz $\Pi_A$, we can reverse the direction of the $3$-star counts, i.e., consider the $3$-stars from set $A$ to $X,B,C$ and obtain  $\Pi_A$ in a similar manner. Once all the community membership vectors $\Pi$ are obtained, we can obtain the community connectivity matrix $P$, using the relationship: $\Pi^\top P \Pi = \Ebb[G|\Pi]$ and noting that we assume $\Pi$ to be of rank $k$.
Thus, we are able to learn the community membership vectors $\Pi$ and the model parameters $\halpha$ and $P$ of the mixed membership model using edge counts and the $3$-star count tensor. We now describe modifications to this approach to handle empirical moments.


\subsection{Learning Algorithm Under Empirical Moments}\label{sec:algosamples}

In the previous section, we explored a tensor-based approach for learning mixed membership model under exact moments. However, in practice, we only have samples (i.e. the observed network), and the method needs to be robust to perturbations when empirical moments are employed.


\floatname{algorithm}{Algorithm}
\setcounter{algorithm}{0}
\begin{algorithm}
\caption{$\{\h{\Pi}, \h{P}, \halpha\} \leftarrow $ LearnMixedMembership$(G, k,\alpha_0, N,\tau)$}\label{algo:main}
\begin{algorithmic}
\renewcommand{\algorithmicrequire}{\textbf{Input: }}
\renewcommand{\algorithmicensure}{\textbf{Output: }}
\REQUIRE Adjacency matrix $G\in \R^{n\times n}$, $k$ is the number of communities, $\alpha_0:=\sum_i \alpha_i$, where $\alpha$ is the Dirichlet parameter vector, $N$ is the number of iterations for the tensor power method, and $\tau$ is used for thresholding the estimated community membership vectors, specified in \eqref{eqn:thresholdtau} in assumption A5. Let $A^c:=[n]\setminus A$ denote the set of nodes not in $A$.
\ENSURE Estimates of the community membership vectors $\Pi\in \R^{n\times k}$,  community connectivity matrix $P\in [0,1]^{k\times k}$, and the normalized Dirichlet parameter vector $\halpha$.
\STATE Partition the vertex set $[n]$  into 5 parts $X$, $Y$,  $A$, $B$, $C$.
\STATE Compute moments $G^{\alpha_0}_{X, A}$, $G^{\alpha_0}_{X, B }$, $G^{\alpha_0}_{X,   C}$, $\Triples_{Y\rightarrow \{A,B,C\}}^{\alpha_0}$ using \eqref{eqn:GXAalphadef} and \eqref{eqn:triplesalpha}.
\STATE $\{\h{\Pi}_{A^c}, \halpha\}\leftarrow$ LearnPartitionCommunity$(G^{\alpha_0}_{X, A}$, $G^{\alpha_0}_{X, B }$, $G^{\alpha_0}_{X,   C}$,\, $\Triples_{Y\rightarrow \{A,B,C\}}^{\alpha_0},\,G, N,\tau)$.
\STATE Interchange roles\footnote{The rows of estimates $\h{\Pi}_{A^c}$ and $\h{\Pi}_{Y^c}$ may be permuted with respect to one another, since they correspond to two different instances of the tensor power iterations. We can easily align them using the estimates corresponding to the common set $A^c \cap Y^c$.} of $Y$ and $A$ to obtain $\h{\Pi}_{Y^c}$.
\STATE Define $\h{Q}$ such that
its $i$-th row is
$\h{Q}^i:=   (\alpha_0+1)\frac{\h{\Pi}^i}{|\h{\Pi}^i|_1}
-\frac{\alpha_0}{n}\vec{1}^\top
$.
 \COMMENT{We will establish that $\h{Q}\approx (\Pi^\dagger)^\top$ under conditions A1-A5.}
\STATE Estimate $\h{P} \leftarrow   \h{Q} G \h{Q}^\top$. \COMMENT{Recall that $\Ebb[G]=\Pi^\top P \Pi$ in our model.}
\STATE Return $\h{\Pi}, \h{P}, \halpha$
\end{algorithmic}
\end{algorithm}

\subsubsection{Pre-processing steps}\label{sec:preprocess}

\floatname{algorithm}{Procedure}\setcounter{algorithm}{0}
\begin{algorithm}
\caption{$\{\h{\Pi}_{A^c}, \halpha\}\leftarrow$ LearnPartitionCommunity($G^{\alpha_0}_{X, A}$, $G^{\alpha_0}_{X, B }$, $G^{\alpha_0}_{X,   C}$, \, $\Triples_{Y\rightarrow \{A,B,C\}}^{\alpha_0}$, \, $G$,  \,$N$, $\tau$)}\label{procedure:reconstruct}
\begin{algorithmic}
\renewcommand{\algorithmicrequire}{\textbf{Input: }}
\renewcommand{\algorithmicensure}{\textbf{Output: }}
\REQUIRE Require modified adjacency submatrices $G^{\alpha_0}_{X, A}$, $G^{\alpha_0}_{X, B}$, $G^{\alpha_0}_{X, C}$,  $3$-star count tensor
$\Triples_{Y\rightarrow \{A,B,C\}}^{\alpha_0}$, adjacency matrix $G$,  number of iterations $N$ for the tensor power method and threshold $\tau$ for thresholding estimated community membership vectors. Let $\thres(A, \tau)$ denote the element-wise thresholding operation using threshold $\tau$, i.e., $\thres(A, \tau)_{i,j} = A_{i,j} $ if $A_{i,j}\geq \tau$ and $0$ otherwise. Let $e_i$ denote basis vector along coordinate $i$.
\ENSURE Estimates of  $\Pi_{A^c}$ and $\halpha$.
\STATE Compute rank-$k$ SVD: $(|X|^{-1/2} G_{X,A}^{\alpha_0})^\top_{k-svd}= U_A D_A V_A^\top$ and compute whitening matrices $\h{W}_A:= U_A D_A^{-1}$. Similarly, compute $\h{W}_B, \h{W}_C$ and $\h{R}_{AB}, \h{R}_{AC}$   using \eqref{eqn:tildeW}.
\STATE Compute whitened and symmetrized tensor $T\leftarrow\Triples_{Y\rightarrow \{A,B,C\}}^{\alpha_0}(\h{W}_A, \h{W}_B \h{R}_{AB},\h{W}_C \h{R}_{AC})$.
\STATE $\{\h{\lambda}, \h{\Phi}\}\leftarrow $TensorEigen$(T, \{\h{W}^\top_A G^\top_{i, A}\}_{i \notin A}, N)$.
\COMMENT{$\h{\Phi}$ is a $k\times k$ matrix with each columns being an estimated eigenvector and $\h{\lambda}$ is the vector of estimated eigenvalues.}
\STATE     $\h{\Pi}_{A^c}  \leftarrow\thres( \Diag(\h{\lambda})^{-1}\h{\Phi}^\top \h{W}_A^\top G_{A^c,A}^\top\,,\,\, \tau)$ and $\hat{\alpha}_i \leftarrow \h{\lambda}_i^{-2}$, for $i \in [k]$.
\STATE Return $\h{\Pi}_{A^c}$ and $\hat{\alpha}$.
\end{algorithmic}
\end{algorithm}

\paragraph{Partitioning: }In the previous section, we partitioned the nodes into four sets $A, B,C, X$ for learning under exact moments. However, we require more partitions under empirical moments to avoid statistical dependency issues and obtain stronger reconstruction guarantees. We now divide the network into five non-overlapping sets $A, B, C, X, Y$. The set $X$ is employed to compute whitening matrices $\h{W}_A$, $\h{W}_B$ and $\h{W}_C$, described in detail subsequently,   the set $Y$ is employed to compute the $3$-star count tensor $\Triples^{\alpha_0}$ and sets $A,B,C$ contain the leaves of the $3$-stars under consideration. The roles of the sets can be interchanged to obtain the community membership vectors of all the sets.



\paragraph{Whitening: }The whitening procedure is along the same lines as described in the previous section, except that now empirical moments are used. Specifically,
consider the $k$-rank singular value decomposition (SVD) of the modified adjacency matrix $G^{\alpha_0}$ defined in \eqref{eqn:GXAalphadef},  \[(|X|^{-1/2} G^{\alpha_0}_{X, A})^\top_{k-svd} = U_A D_A V_A^\top.\] Define $ \h{W}_A := U_A D_A^{-1},$ and similarly  define $\h{W}_B$ and $\h{W}_C$ using the corresponding matrices $G_{X,B}^{\alpha_0}$ and $G_{X,C}^{\alpha_0}$ respectively. Now define \beq\label{eqn:tildeW} \h{R}_{A,B}:=\frac{1}{|X|} \h{W}_B^\top (G^{\alpha_0}_{X, B})^\top_{k-svd}\cdot (G^{\alpha_0}_{X, A})_{k-svd} \h{W}_A, \eeq and similarly define $\h{R}_{AC}$. The whitened and symmetrized graph-moment tensor is now computed as
\[ \Triples^{\alpha_0}_{Y\rightarrow \{A,B,C\}}(\h{W}_A, \h{W}_B \h{R}_{AB}, \h{W}_C\h{R}_{AC}),\] where $\Triples^{\alpha_0}$ is given by \eqref{eqn:triplesalpha} and the multi-linear transformation of a tensor is defined in \eqref{eqn:multilinear}.

\subsubsection{Modifications to the  tensor  power method}\label{sec:powermodify}

Recall that under exact moments, the stable eigen-pairs of a symmetric orthogonal tensor can be computed in a straightforward manner through the basic power iteration method in \eqref{eqn:basicpower}, along with the deflation procedure. However, this is not sufficient to get good reconstruction guarantees under empirical moments. We now propose a robust tensor method, detailed  in  Procedure~\ref{alg:robustpower}. The main modifications involve: {\em(i)} efficient initialization and {\em(ii)} adaptive deflation, which are detailed below. Employing these modifications allows us to tolerate a far greater perturbation of the third order moment tensor, than the basic tensor power procedure employed in~\cite{AGHKT12}. See remarks following Theorem~\ref{thm:robustpower} in Appendix~\ref{app:tensorpower} for the precise comparison.

\paragraph{Efficient Initialization: }Recall that the basic tensor power method incorporates generic initialization vectors and this procedure recovers all the stable eigenvectors   correctly (except for initialization vectors over a set of measure zero). However, under empirical moments, we have a perturbed tensor, and here, it is advantageous to instead employ specific initialization vectors. For instance, to obtain one of the eigenvectors $(\Phi)_i$, it is advantageous to initialize with a vector in the neighborhood of $(\Phi)_i$. This not only reduces the number of power iterations required to converge (approximately), but more importantly, this makes the power method more robust to perturbations. See Theorem~\ref{thm:robustpower} in Appendix~\ref{sec:goodinit} for a detailed analysis quantifying the   relationship between initialization vectors, tensor perturbation and the resulting guarantees for recovery of the tensor eigenvectors.


For a mixed membership model in the sparse regime, recall that the community membership vectors $\Pi$ are sparse (with high probability). Under this regime of the model,  we note that the whitened neighborhood vectors contain  good initializers for the power iterations. Specifically, in    Procedure~\ref{alg:robustpower}, we
initialize with the whitened neighborhood vectors  $\h{W}_A^\top G_{i, A}^\top$, for $i \notin A$. The intuition behind this is as follows: for a suitable choice of parameters (such as the scaling of network size $n$ with respect to the number of communities $k$), we expect
neighborhood vectors $G_{i, A}^\top$ to concentrate  around their  mean values, \viz, $F_A \pi_i$. Since $\pi_i$ is sparse (w.h.p) for the model regime under consideration, this implies that there exist vectors $\h{W}_A^\top F_A\pi_i$, for $i \in A^c$, which concentrate (w.h.p) on only along a few eigen-directions  of the whitened tensor, and hence, serve as an effective initializer.

\paragraph{Adaptive Deflation: }Recall that in the basic power iteration procedure, we can obtain the eigen-pairs one after another through simple deflation:  subtracting the  estimates of the current eigen-pairs and running the power iterations again to obtain new eigenvectors. However, it turns out that we can establish  better theoretical guarantees (in terms of greater  robustness)  when we adaptively deflate the tensor in each power iteration. In    Procedure~\ref{alg:robustpower}, among the estimated eigen-pairs,  we only deflate those which ``compete'' with the current estimate of the power iteration. In other words, if the vector in the current iteration $\theta_t^{(\tau)}$  has a significant projection along the direction of an estimated eigen-pair $\phi_j$, i.e. \[|\lambda_j \inner{\theta_t^{(\tau)}, \phi_j}| > \xi,\] for some threshold $\xi$, then the eigen-pair is deflated; otherwise the eigenvector $\phi_j$ is  not deflated. This allows us to carefully control the error build-up  for each estimated eigenpair in our analysis.  Intuitively, if an eigenvector does not have a good correlation with the current estimate, then it does not interfere with the update of the current vector, while if the eigenvector has a good correlation, then it is pertinent that it be deflated so as to discourage convergence in the direction of the already estimated eigenvector.
See Theorem~\ref{thm:robustpower} in Appendix~\ref{sec:goodinit} for details.\\

Finally, we note that stabilization, as proposed by~\cite{SIMAX-080148-Tensor-Eigenvalues} for general tensor eigen-decomposition  (as opposed to orthogonal decomposition in this paper), can be effective in improving convergence, especially on real data, and we defer its detailed analysis to future work.

\begin{algorithm}[h]
\caption{$\{\lambda, \Phi\}\leftarrow $TensorEigen$(T,\, \{v_i\}_{i\in [L]}, N)$}\label{alg:robustpower}
\begin{algorithmic}
\renewcommand{\algorithmicrequire}{\textbf{Input: }}
\renewcommand{\algorithmicensure}{\textbf{Output: }}
\REQUIRE Tensor $T\in \R^{k \times k \times k}$, $L$ initialization vectors $\{v_i\}_{i\in L}$, number of
iterations  $N$.
\ENSURE the estimated eigenvalue/eigenvector pairs $\{\lambda, \Phi\}$, where $\lambda$ is the vector of eigenvalues and $\Phi$ is the matrix of eigenvectors.

\FOR{$i =1$ to $k$}
\FOR{$\tau = 1$ to $L$}
\STATE $\th{0}\leftarrow v_\tau$.
\FOR{$t = 1$ to $N$}
\STATE $\tilde{T}\leftarrow T$.
\FOR{$j=1$ to $i-1$ (when $i>1$)}
\IF{$|\lambda_j \inner{\th{t}^{(\tau)}, \phi_j}|>\xi$}
\STATE $\tilde{T}\leftarrow \tilde{T}- \lambda_j \phi_j^{\otimes 3}$.
\ENDIF
\ENDFOR

\STATE Compute power iteration update
$
\th{t}^{(\tau)}  :=
\frac{\tilde{T}(I, \th{t-1}^{(\tau)}, \th{t-1}^{(\tau)})}
{\|\tilde{T}(I, \th{t-1}^{(\tau)}, \th{t-1}^{(\tau)})\|}
$\ENDFOR
\ENDFOR

\STATE Let $\tau^* := \arg\max_{\tau \in L} \{ \tilde{T}(\th{N}^{(\tau)},
\th{N}^{(\tau)}, \th{N}^{(\tau)}) \}$.

\STATE Do $N$ power iteration updates starting from
$\th{N}^{(\tau^*)}$ to obtain eigenvector estimate $\phi_i$, and set $\lambda_i :=
\tilde{T}(\phi_i, \phi_i, \phi_i)$.

\ENDFOR
\RETURN the estimated eigenvalue/eigenvectors
$(\lambda, \Phi)$.

\end{algorithmic}
\end{algorithm}

\subsubsection{Reconstruction after tensor power method}

Recall that previously in  Section~\ref{sec:algoexact}, when exact moments are available, estimating the community membership vectors $\Pi$ is straightforward, once we recover all the stable tensor eigen-pairs. However, in case of empirical moments, we can obtain better guarantees with the following modification: the estimated community membership vectors $\Pi$   are further subject to thresholding so that the weak values are set to zero. Since we are limiting ourselves to  the regime of the mixed membership model, where the community vectors $\Pi$ are sparse (w.h.p), this modification strengthens our reconstruction guarantees. This thresholding step is incorporated in Algorithm~\ref{algo:main}.

Moreover, recall that under exact moments, estimating the
community connectivity matrix $P$ is straightforward, once we recover the community membership vectors since $P\leftarrow
(\Pi^\top)^\dagger\Ebb[G|\Pi] \Pi^\dagger$. However, when empirical moments are available, we are able to establish better reconstruction guarantees through a different method, outlined in Algorithm~\ref{algo:main}. We define $\h{Q}$ such that
its $i$-th row is
\[\h{Q}^i:=   (\alpha_0+1)\frac{\h{\Pi}^i}{|\h{\Pi}^i|_1}
-\frac{\alpha_0}{n}\vec{1}^\top
,\]
based on estimates $\h{\Pi}$, and the matrix $\h{P}$ is obtained as $\h{P}\leftarrow  \h{Q} G \h{Q}^\top$. We subsequently establish that $\h{Q} \h{\Pi}^\top \approx I$, under a set of sufficient conditions outlined in the next section.

\paragraph{Improved support recovery estimates in homophilic models: }A sub-class of community model are those satisfying {\em homophily}. As discussed in Section~\ref{sec:intro},  homophily or the tendency to form edges within the members of the same community has been posited as an important factor in community formation, especially in social settings. Many of the existing learning algorithms, e.g.~\cite{ChenSanghaviXu} require this assumption to provide guarantees in the stochastic block model setting. Moreover, our procedure described below can be easily modified to work in situations where the order of intra-connectivity and inter-connectivity among communities is reversed, i.e. in the community connectivity matrix $P\in [0,1]^{k\times k}$,   $P(i,i)\equiv p < P(i,j)\equiv q$, for all $i\neq j$. For instance, in the $k$-coloring model~\citep{McSherry01}, $p=0$ and $q>0$.

We describe the post-processing method in Procedure~\ref{algo:support} for models with community connectivity matrix $P$ satisfying $P(i,i)\equiv p > P(i,j) \equiv q$ for all $i \neq j$. For such models, we can obtain improved estimates by averaging. Specifically, consider nodes in set $C$ and edges going from $C$ to nodes in $B$. First, consider the special case of the stochastic block model: for each node $c\in C$, compute the number of neighbors in $B$ belonging to each community (as given by the estimate $\h{\Pi}$ from Algorithm~\ref{algo:main}), and declare the community with the maximum number of such neighbors as the community of node $c$. Intuitively, this provides a better estimate for $\Pi_C$ since we average over the edges in $B$. This method has been used before in the context of spectral clustering~\citep{McSherry01}.  

The same idea can be extended to the general mixed membership (homophilic) models: declare communities to be significant if they exceed a certain threshold, as evaluated by the average number of edges to each community.  The correctness of the procedure can be gleaned from the fact that if the true $F$ matrix is input, it satisfies 
\[ F_{j,i} = q + \Pi_{i,j} (p-q), \quad \forall\, i \in [k], j\in[n],\] and if the true $P$ matrix is input, $H=p$ and $L=q$. Thus, under a suitable threshold $\xi$,  the entries   $F_{j,i}$ provide information on whether the corresponding community weight $\Pi_{i,j}$ is significant.

In the next section, we establish that in certain regime of parameters, this support recovery procedure can lead to zero-error support recovery of significant community memberships of the nodes and also rule out communities where a node does not have a strong presence.

\begin{algorithm}
\caption{$\{\h{S}\} \leftarrow $ SupportRecoveryHomophilicModels$(G, k,\alpha_0,   \xi, \h{\Pi})$}\label{algo:support}
\begin{algorithmic}
\renewcommand{\algorithmicrequire}{\textbf{Input: }}
\renewcommand{\algorithmicensure}{\textbf{Output: }}
\REQUIRE Adjacency matrix $G\in \R^{n\times n}$, $k$ is the number of communities, $\alpha_0:=\sum_i \alpha_i$, where $\alpha$ is the Dirichlet parameter vector,
$\xi $ is the threshold for support recovery, corresponding to significant community memberships of an individual.  Get estimate  $\h{\Pi}$ from Algorithm~\ref{algo:main}. Also asume the model is homophilic:  $P(i,i)\equiv p> P(i,j)\equiv q$, for all $i\neq j$.
\ENSURE $\h{S}\in \{0,1\}^{n \times k}$ is the estimated support for significant community memberships (see Theorem~\ref{thm:support} for guarantees).
\STATE Consider partitions $A,B,C,X,Y$ as in Algorithm~\ref{algo:main}.
\STATE Define $\h{Q}$ on lines of definition in Algorithm~\ref{algo:main}, using estimates $\h{\Pi}$. Let the $i$-th row for set $B$ be
$\h{Q}^i_B:=   (\alpha_0+1)\frac{\h{\Pi}_B^i}{|\h{\Pi}_B^i|_1}
-\frac{\alpha_0}{n}\vec{1}^\top
$. Similarly define $\h{Q}^i_C$.
\STATE Estimate $\h{F}_C \leftarrow  G_{C,B} \h{Q}_B^\top$, $\h{P}\leftarrow \h{Q}_C\h{F}_C$.
\IF{$\alpha_0=0$ (stochastic block model)}
\FOR{$x\in C$}
\STATE Let $i^*\leftarrow \argmax_{i \in [k]}\h{F}_C(x,i)$ and $\h{S}(i^*,x) \leftarrow 1$ and $0$ o.w. 
\ENDFOR
\ELSE
\STATE Let $H$ be the average of diagonals of $\h{P}$, $L$ be the average of off-diagonals of $\h{P}$
\FOR{$x\in C$, $i\in[k]$}
\STATE  $\h{S}(i,x)\leftarrow 1 $  if
$\h{F}_C(x,i) \ge L + (H-L)\cdot \frac{3\xi}{4}$
and zero otherwise.\COMMENT{Identify large entries}
\ENDFOR
\ENDIF
\STATE Permute the roles of the sets   $A,B,C,  X, Y$ to get results for remaining nodes.
\end{algorithmic}
\end{algorithm}

\paragraph{Computational complexity: }We note that the computational complexity of the method, implemented naively,  is
$O(n^2k + k^{4.43}\halpha_{\min}^{-1})$ when $\alpha_0 > 1$
and $O(n^2k)$ when $\alpha_0 < 1$.
This is because the time for computing whitening matrices is dominated by SVD
of the top $k$ singular vectors of $n\times n$ matrix, which takes $O(n^2 k)$ time. We then compute the whitened tensor $T$ which requires   time $O(n^2 k + k^3 n) = O(n^2k)$, since for each $i\in Y$, we   multiply $G_{i,A}, G_{i,B}, G_{i,C}$ with the corresponding whitening matrices, and this step takes
$O(nk)$ time.  We then average this $k\times k\times k$ tensor over different nodes $i \in Y$ to the result, which takes $O(k^3)$ time in each step.

For the tensor power method, the time required for a single iteration is $O(k^3)$.  We need at most $\log n$ iterations per initial vector, and we need to consider  $O(\halpha_{\min}^{-1} k^{0.43})$ initial vectors (this could be smaller when $\alpha_0 < 1$). Hence the total running time of tensor
power method is $O(k^{4.43}\halpha_{\min}^{-1})$ (and when $\alpha_0$ is small this can be improved to $O(k^4\halpha_{\min}^{-1})$ which is dominated by $O(n^2k)$.

In the process of estimating $\Pi$ and $P$, the dominant operation is multiplying $k\times n$ matrix by $n\times n$ matrix, which takes
$O(n^2k)$ time. For support recovery, the dominant operation
is computing the ``average degree'', which again takes $O(n^2k)$ time. Thus, we have that the overall computational time  is
$O(n^2k + k^{4.43}\halpha_{\min}^{-1})$ when $\alpha_0 > 1$
and $O(n^2k)$ when $\alpha_0 < 1$.


Note that the above bound on  complexity of our method nearly matches the bound for spectral clustering method~\citep{McSherry01}, since computing the $k$-rank SVD   requires $O(n^2k)$ time. Another method for learning stochastic block models is based on convex optimization involving semi-definite programming (SDP)~\citep{ChenSanghaviXu}, and it provides the best scaling bounds (for both the network size $n$ and the separation $p-q$ for edge connectivity) known so far.   The specific convex problem can be solved via the method of {\em augmented Lagrange multipliers}~\citep{lin2010augmented}, where each step consists of an SVD operation and q-linear convergence is established by~\cite{lin2010augmented}. This implies that the method has complexity $O(n^3)$, since it involves taking SVD of a general $n \times n$ matrix, rather than a $k$-rank SVD. Thus, our method has significant advantage in terms of computational complexity, when the number of communities is much smaller than the network size $(k \ll n)$.

Further,  a subsequent work provides a more sophisticated implementation of the proposed tensor method through parallelization and the use of stochastic gradient descent for tensor decomposition~\citep{AnandkumarEtal:communityimplementation13}. Additionally, the $k$-rank SVD operations are approximated via randomized methods such as the Nystrom's method leading to more efficient implementations~\citep{gittens2013revisiting}. ~\cite{AnandkumarEtal:communityimplementation13} deploy the tensor approach for community detection and establish that it has a running time of $O(n+k^3)$  using $nk$ cores  under a parallel computation model~\citep{jaja1992introduction}. 

\section{Sample Analysis for Proposed Learning Algorithm}\label{sec:sample}

\subsection{Homogeneous Mixed Membership Models}\label{sec:special}

It is easier to first present the results for our proposed algorithm  for the special case, where all the communities have the same expected size and the entries of the community connectivity matrix $P$ are equal on diagonal and off-diagonal locations:\beq \label{eqn:special} \halpha_i \equiv \frac{1}{k}, \qquad P(i,j) = p \cdot\Ibb(i=j) + q \cdot \Ibb(i\neq j), \quad p\geq q.\eeq In other words, the probability of an edge according to $P$ only depends on whether it is between two individuals of the same community or between different communities. The above setting is also well studied for  stochastic block models $(\alpha_0=0)$, allowing us to compare our results with existing ones. The results for general mixed membership models are deferred to Section~\ref{sec:generalresults}.



\paragraph{[A1] Sparse regime of Dirichlet parameters: }The community membership vectors are drawn from the Dirichlet distribution, $\Dir(\alpha)$, under the mixed membership model. We assume that  $\alpha_i<1$ for $i\in [k]$   (see Section~\ref{sec:model} for an extended discussion on the sparse regime of the Dirichlet distribution) and that $\alpha_0$ is known.


\paragraph{[A2] Condition on the network size: }
Given the concentration parameter of the Dirichlet distribution, $\alpha_0:=\sum_i \alpha_i$, we require that \beq \label{eqn:condspecial1}n = \tl{\Omega}(k^2 (\alpha_0+1)^2),\eeq
and that the disjoint sets $A,B,C,X,Y$ are chosen randomly and   are of size $\Theta(n)$. Note that from assumption A1, $\alpha_i <1$ which implies that $\alpha_0 <k$. Thus, in the worst-case, when $\alpha_0=\Theta(k)$, we require\footnote{The notation $\tl{\Omega}(\cdot), \tl{O}(\cdot)$ denotes $\Omega(\cdot), O(\cdot)$ up to poly-log factors.} $n = \tl{\Omega}(k^4)$,  and in the best case, when $\alpha_0=\Theta(1)$, we require $n = \tl{\Omega}(k^2)$. The latter case includes the stochastic block model $(\alpha_0=0)$, and thus, our results match the state-of-art bounds for learning stochastic block models.


\paragraph{[A3] Condition on edge connectivity: }
Recall that $p$ is the probability of intra-community connectivity and $q$ is the probability of inter-community connectivity. We require that \beq \label{eqn:condspecial2} \frac{p-q}{\sqrt{p}}=
\Omega\left(\frac{(\alpha_0+1)k}{n^{1/2}}\right)\eeq
The above condition is on the  standardized separation   between intra-community and inter-community connectivity (note that $\sqrt{p}$ is the standard deviation of a Bernoulli random variable).
 The above condition is required to control the perturbation in the whitened tensor (computed using observed network samples), thereby, providing guarantees on the estimated eigen-pairs through the tensor power method.


\paragraph{[A4] Condition on number of iterations of the power method: }We assume that the number of iterations $N$ of the tensor power method in Procedure~\ref{alg:robustpower} satisfies
\beq N \geq C_2 \cdot \left( \log(k) + \log\log\left(
\frac{p-q}{p}\right) \right),\eeq
for some constant $C_2$.

\paragraph{[A5] Choice of $\tau$ for thresholding community vector estimates: }The threshold $\tau$ for obtaining estimates $\h{\Pi}$ of community membership vectors in Algorithm~\ref{algo:main} is chosen as
\bcase{\tau =}\label{eqn:thresholdtau}\Theta\left(\frac{ k\sqrt{\alpha_0}}{\sqrt{n}}\cdot \frac{\sqrt{p}}{p-q} \right), & $\alpha_0\neq0$,\\ 0.5, & $\alpha_0=0$,
\ecase  For the stochastic block model $(\alpha_0=0)$,  since $\pi_i$ is a basis vector,  we can use a large threshold. For general models $(\alpha_0\neq 0)$, $\tau$ can be viewed as a regularization parameter and decays as $n^{-1/2}$ when other parameters are held fixed.
We are now ready to state the error bounds on the estimates of community membership vectors $\Pi$ and the block connectivity matrix $P$. $\h{\Pi}$ and $\h{P}$ are the estimates computed in Algorithm~\ref{algo:main}.

Recall that for a matrix $M$, $(M)^i$ and $(M)_i$ denote the $i^{\tha}$ row and column respectively. We say that an event holds with high probability, if it occurs with probability $1-n^{-c}$ for some constant $c>0$.

\begin{theorem}[Guarantees on Estimating $P$, $\Pi$]\label{cor:special}Under assumptions A1-A5, we have with high probability
\begin{align}
\veps_{\pi,\ell_1}:= \max_{i\in [n]} \| \h{\Pi}^i - \Pi^i\|_1&=   \tl{O}\left(\frac{(\alpha_0+1)^{3/2}\sqrt{n p}}{(p-q)}\right)\\
\veps_P :=\max_{i,j\in [k]}|\h{P}_{i,j} - P_{i,j}|
&=  \tl{O}\left(\frac{(\alpha_0+1)^{3/2}k\sqrt{p}}{\sqrt{n}}\right)
.\end{align}
\end{theorem}

The proofs are given in the Appendix and a proof outline is provided in Section~\ref{sec:outline}.

The main ingredient in establishing the above result is the tensor concentration bound and additionally, recovery guarantees under the tensor power method in Procedure~\ref{alg:robustpower}. We now provide these results below.


Recall that $F_A:=\Pi_A^\top P^\top$ and  $\Phi=W_A^\top F_A \Diag(\halpha^{1/2})$ denotes the set of tensor eigenvectors under exact moments in \eqref{eqn:tilF}, and $\h{\Phi}$ is the set of estimated eigenvectors under empirical moments, obtained using Procedure~\ref{procedure:reconstruct}.
We establish the following guarantees.

\begin{lemma}[Perturbation bound for estimated eigen-pairs]\label{lemma:tensoroutputspecial}
Under the assumptions A1-A4, the recovered eigenvector-eigenvalue pairs $(\h{\Phi}_i, \h{\lambda}_i)$ from the tensor power method in Procedure~\ref{alg:robustpower} satisfies with high probability, for a permutation $\theta$, such that
\beq\label{eqn:phi} \max_{i\in [k]} \|\h{\Phi}_i- \Phi_{\theta(i)}\| \leq 8 k^{-1/2} \veps_T , \qquad \max_{i\in[k]} | \lambda_i - \halpha_{\theta(i)}^{-1/2}| \leq 5 \veps_T,\eeq   The tensor perturbation bound $\veps_T$ is given by
\begin{align}\nn\veps_T&:=\left\|\Triples^{\alpha_0}_{Y\rightarrow \{A,B,C\}}(\hat{W}_A, \hat{W}_B \h{R}_{AB}, \hat{W}_C \h{R}_{AC})-\Ebb[\Triples^{\alpha_0}_{Y\rightarrow \{A,B,C\}}(W_A, W_B R_{AB}, W_C R_{AC})|\Pi_{A\cup B\cup C}]\right\|\\ &=   \tl{O}\left(\frac{(\alpha_0+1)k^{3/2}\sqrt{p} }{(p-q)\sqrt{n}}\right),\label{eqn:vepsTlemma}\end{align} where $\|T\|$ for a tensor $T$ refers to its spectral norm.
\end{lemma}


\paragraph{Stochastic block models  $(\alpha_0=0)$: }For stochastic block models,  assumptions A2 and A3 reduce to
\beq \label{eqn:condspecialblock}n = \tl{\Omega}(k^2), \qquad \zeta=\Theta\left(\frac{\sqrt{p}}{p-q}\right)=
O\left(\frac{n^{1/2}}{k}\right).\eeq This matches with the best known scaling (up to poly-log factors), and was previously achieved via convex optimization by~\cite{ChenSanghaviXu} for stochastic block models. However, our results in Theorem~\ref{cor:special} do not provide zero error guarantees as in~\cite{ChenSanghaviXu}. We strengthen our results to provide zero-error guarantees in Section~\ref{sec:homophilicguarantees} below and thus, match the scaling of~\cite{ChenSanghaviXu} for stochastic block models. Moreover, we also provide zero-error support recovery guarantees for recovering significant memberships of nodes in mixed membership models  in  Section~\ref{sec:homophilicguarantees}.

\paragraph{Dependence on $\alpha_0$: }The guarantees degrade as $\alpha_0$ increases, which is intuitive since the extent of community overlap increases. The requirement for scaling of $n$ also grows as $\alpha_0$ increases. Note that the guarantees on $\veps_\pi$ and $\veps_P$ can be improved by assuming a more stringent scaling of $n$ with respect to $\alpha_0$, rather  than the one  specified by A2.

%
%

\subsubsection{Zero-error guarantees for support recovery}\label{sec:homophilicguarantees}

Recall that we proposed Procedure~\ref{algo:support} as a post-processing step to provide improved support recovery estimates. We now provide guarantees for this method.


We now specify the threshold $\xi$ for support recovery in Procedure~\ref{algo:support}.

\paragraph{[A6] Choice of $\xi$ for support recovery: }We assume that the threshold $\xi$ in Procedure~\ref{algo:support} satisfies
\[ \xi= \Omega(\veps_P),\] where $\veps_P$ is specified in Theorem~\ref{cor:special}.
We now state the guarantees for support recovery.

\begin{theorem}[Support recovery  guarantees]\label{thm:support}Assuming A1-A6  and \eqref{eqn:special}  hold,   the support recovery method in Procedure~\ref{algo:support} has the following guarantees on the estimated support set $\h{S}$: with high probability,
\beq \Pi(i,j)\geq \xi \Rightarrow \h{S}(i,j)=1 \quad\mbox{ and }\quad \Pi(i,j) \leq \frac{\xi}{2} \Rightarrow \h{S}(i,j)=0, \quad \forall i\in [k],j\in [n],\eeq where $\Pi$ is the true community membership matrix. \end{theorem}

Thus, the above result guarantees that the Procedure~\ref{algo:support} correctly recovers all the ``large'' entries of $\Pi$ and also correctly rules out all the ``small'' entries in $\Pi$. In other words, we can correctly infer all the significant memberships of each node and also rule out the set of communities where a node does not have a strong presence.

The only shortcoming of the above result is that there is a gap between the ``large'' and ``small'' values, and for an intermediate set of values (in $[\xi/2,\xi]$), we cannot guarantee correct inferences about the community memberships.  Note this gap depends on $\veps_P$, the error in estimating the $P$ matrix. This is intuitive, since as the error $\veps_P$ decreases, we can infer the community memberships over a large range of values.

For the special case of stochastic block models (i.e.
$\lim \alpha_0\rightarrow 0$), we can improve the above result and give a zero error guarantee at all nodes (w.h.p). Note that we no longer require a threshold $\xi$ in this case, and only infer one community for each node.

\begin{corollary}[Zero error guarantee for block models]\label{cor:zeroerrorblock}Assuming A1-A5  and \eqref{eqn:special}  hold,   the support recovery method in Procedure~\ref{algo:support} correctly
identifies the  community memberships for
all nodes with high probability in case of stochastic block models $(\alpha_0\to 0)$.
\end{corollary}

Thus, with the above result, we match the state-of-art results of~\cite{ChenSanghaviXu} for stochastic block models in terms of scaling requirements and recovery guarantees.   
%

\subsection{General (Non-Homogeneous) Mixed Membership Models}\label{sec:generalresults}


In the previous sections, we provided learning guarantees for learning homogeneous mixed membership models. Here, we extend the results to   learning general non-homogeneous mixed membership models  under a sufficient set of conditions, involving scaling of various parameters such as network size $n$, number of communities $k$, concentration parameter $\alpha_0$ of the Dirichlet distribution (which is a measure of overlap of the communities) and so on.


\paragraph{[B1] Sparse regime of Dirichlet parameters: }The community membership vectors are drawn from the Dirichlet distribution, $\Dir(\alpha)$, under the mixed membership model. We assume that\footnote{The assumption B1 that the Dirichlet distribution be in the sparse regime is not strictly needed. Our results can be extended to general Dirichlet distributions, but with worse scaling requirements on $n$. The dependence of $n$ is still    polynomial in $\alpha_0$, i.e. we require $n = \tl{\Omega}((\alpha_0+1)^c\halpha_{\min}^{-2})$, where $c\geq 2$ is some constant. }
 $\alpha_i<1$ for $i\in [k]$ $\alpha_i <1$  (see Section~\ref{sec:model} for an extended discussion on the sparse regime of the Dirichlet distribution).


\paragraph{[B2] Condition on the network size: }
Given the concentration parameter of the Dirichlet distribution,  $\alpha_0:=\sum_i\alpha_i$, and $\halpha_{\min}:=\alpha_{\min}/\alpha_0$, the expected size of the smallest community, define \beq\label{eqn:rho}\rho:=  \frac{\alpha_0+1}{ \halpha_{\min}}. \eeq  We require that the network size scale as \beq n = \Omega\left(\rho^{2} \log^2 k\right), \label{eqn:dimcondition}\eeq  and that the sets  $A,B,C,X,Y$ are $\Theta(n)$. Note that from assumption B1, $\alpha_i <1$ which implies that $\alpha_0 <k$. Thus, in the worst-case, when $\alpha_0=\Theta(k)$, we require\footnote{The notation $\tl{\Omega}(\cdot), \tl{O}(\cdot)$ denotes $\Omega(\cdot), O(\cdot)$ up to log factors.} $n = \tl{\Omega}(k^4)$, assuming equal sizes: $\halpha_i=1/k$, and in the best case, when $\alpha_0=\Theta(1)$, we require $n = \tl{\Omega}(k^2)$. The latter case includes the stochastic block model $(\alpha_0=0)$, and thus, our results match the state-of-art bounds for learning stochastic block models. See Section~\ref{sec:special} for an extended discussion.


\paragraph{[B3] Condition on relative community sizes and block connectivity matrix: }Recall that $P\in [0,1]^{k\times k}$ denotes the block connectivity matrix. Define
\beq\label{eqn:relsizes}\zeta:=\left(\frac{\halpha_{\max}}{\halpha_{\min}}\right)^{1/2}\frac{
\sqrt{(\max_i (P\halpha)_i)}}{ \sigma_{\min}(P)},\eeq where $\sigma_{\min}(P)$ is the minimum singular value of $P$.
 We require that \bcase{\zeta=}O\left(\frac{n^{1/2}}{\rho}\right),&$\alpha_0<1$\\
O\left(\frac{n^{1/2}}{\rho k \halpha_{\max}}\right)& $\alpha_0\geq 1$. \ecase
Intuitively, the above condition requires the ratio of maximum and minimum expected  community sizes to be not too large  and for the matrix $P$ to be well conditioned.
 The above condition is required to control the perturbation in the whitened tensor (computed using observed network samples), thereby, providing guarantees on the estimated eigen-pairs through the tensor power method. The above condition can be interpreted as  a separation requirement between intra-community and inter-community connectivity in the special case considered in Section~\ref{sec:special}.
Specifically, for the special case of homogeneous mixed membership model,  we have \[ \sigma_{\min}(P)=\Theta(p-q),\quad \max_i (P\halpha)_i = \frac{p}{k}+(k-1)\frac{q}{k} \le p.\]
Thus, the assumptions A2 and A3 in Section~\ref{sec:special} given by \[ n = \tl{\Omega}(k^2 (\alpha_0+1)^2), \qquad \zeta=\Theta\left(\frac{\sqrt{p}}{p-q}\right)=
O\left(\frac{n^{1/2}}{(\alpha_0+1)k}\right)\]are special cases of the assumptions B2 and B3 above.


\paragraph{[B4] Condition on number of iterations of the power method: }We assume that the number of iterations $N$ of the tensor power method in Procedure~\ref{alg:robustpower} satisfies
\beq N \geq C_2 \cdot \left( \log(k) + \log\log\left(
\frac{ \sigma_{\min}(P)}{
(\max_i (P\halpha)_i)} \right) \right),\eeq
for some constant $C_2$.

\paragraph{[B5] Choice of $\tau$ for thresholding community vector estimates: }The threshold $\tau$ for obtaining estimates $\h{\Pi}$ of community membership vectors in Algorithm~\ref{algo:main} is chosen as
\bcase{\tau =}\label{eqn:thresholdtaugen}\Theta\left(\frac{ \rho^{1/2}\cdot\zeta \cdot \halpha_{\max}^{1/2}}{n^{1/2} \cdot\halpha_{\min}} \right), & $\alpha_0\neq0$,\\ 0.5, & $\alpha_0=0$,
\ecase  For the stochastic block model $(\alpha_0=0)$,  since $\pi_i$ is a basis vector,  we can use a large threshold. For general models $(\alpha_0\neq 0)$, $\tau$ can be viewed as a regularization parameter and decays as $n^{-1/2}$ when other parameters are held fixed. Moreover, when $n =\tl{\Theta}(\rho^2)$, we have that $\tau \sim \rho^{-1/2}$ when other terms are held fixed. Recall that $\rho \propto (\alpha_0+1)$ when the expected community sizes $\halpha_i $ are held fixed. In this case, $\tau \sim \rho^{-1/2}$ allows for smaller values to be picked up after thresholding as $\alpha_0$ is increased. This is intuitive since   as $\alpha_0$ increases, the community vectors $\pi$ are  more ``spread out'' across different communities and have smaller values.

We are now ready to state the error bounds on the estimates of community membership vectors $\Pi$ and the block connectivity matrix $P$. $\h{\Pi}$ and $\h{P}$ are the estimates computed in Algorithm~\ref{algo:main}.

Recall that for a matrix $M$, $(M)^i$ and $(M)_i$ denote the $i^{\tha}$ row and column respectively. We say that an event holds with high probability, if it occurs with probability $1-n^{-c}$ for some constant $c>0$.

\begin{theorem}[Guarantees on estimating $P$, $\Pi$]\label{thm:mainsample}Under assumptions B1-B5, The estimates $\h{P}$ and $\h{\Pi}$ obtained from Algorithm~\ref{algo:main}  satisfy with high  probability,
\begin{align}\label{eqn:vepspil1}\veps_{\pi,\ell_1}:= \max_{i\in [k]}|(\h{\Pi})^i - (\Pi)^i|_1 &=\tl{O}\left(n^{1/2}\cdot\rho^{3/2}\cdot \zeta\cdot \halpha_{\max} \right)\\  \label{eqn:vepsP}\veps_P :=\max_{i,j\in [n]}|\h{P}_{i,j} - P_{i,j}|
&=  \tl{O}\left( n^{-1/2}\cdot \rho^{5/2} \cdot \zeta\cdot \halpha_{\max}^{3/2} \cdot (P_{\max}-P_{\min})\right)\end{align}
\end{theorem}

The proofs are   in   Appendix~\ref{sec:mainproof} and a proof outline is provided in Section~\ref{sec:outline}.

The main ingredient in establishing the above result is the tensor concentration bound and additionally, recovery guarantees under the tensor power method in Procedure~\ref{alg:robustpower}. We now provide these results below.


Recall that $F_A:=\Pi_A^\top P^\top$ and  $\Phi=W_A^\top F_A \Diag(\halpha^{1/2})$ denotes the set of tensor eigenvectors under exact moments in \eqref{eqn:tilF}, and $\h{\Phi}$ is the set of estimated eigenvectors under empirical moments, obtained using Procedure~\ref{procedure:reconstruct}.
We establish the following guarantees.

\begin{lemma}[Perturbation bound for estimated eigen-pairs]\label{lemma:tensoroutput}
Under the assumptions B1-B4, the recovered eigenvector-eigenvalue pairs $(\h{\Phi}_i, \h{\lambda}_i)$ from the tensor power method in Procedure~\ref{alg:robustpower} satisfies with high probability, for a permutation $\theta$, such that
\beq\label{eqn:phi-gen} \max_{i\in [k]} \|\h{\Phi}_i- \Phi_{\theta(i)}\| \leq 8 \halpha_{\max}^{1/2} \veps_T , \qquad \max_i | \lambda_i - \halpha_{\theta(i)}^{-1/2}| \leq 5 \veps_T,\eeq   The tensor perturbation bound $\veps_T$ is given by
\begin{align}\nn\veps_T&:=\left\|\Triples^{\alpha_0}_{Y\rightarrow \{A,B,C\}}(\hat{W}_A, \hat{W}_B \h{R}_{AB}, \hat{W}_C \h{R}_{AC})-\Ebb[\Triples^{\alpha_0}_{Y\rightarrow \{A,B,C\}}(W_A, W_B R_{AB}, W_C R_{AC})|\Pi_{A\cup B\cup C}]\right\|\\ &=\tl{O}\left( \frac{\rho}{\sqrt{n}} \cdot \frac{\zeta}{\halpha_{\max}^{1/2}}\right),\label{eqn:vepsTlemma-gen}\end{align} where $\|T\|$ for a tensor $T$ refers to its spectral norm, $\rho$ is defined in  \eqref{eqn:rho} and $\zeta$ in \eqref{eqn:relsizes}.
\end{lemma}

\subsubsection{Application to Planted Clique Problem}\label{sec:clique}

The planted clique problem is a special case of the stochastic block model~\cite{condon1999algorithms}, and is arguably the simplest setting for the community problem. Here, a clique of size $s$ is uniformly planted (or placed) in an Erd\H{o}s-R\'{e}nyi graph with edge probability $0.5$. This can be viewed as a stochastic block model with $k=2$ communities, where $\halpha_{\min}=s/n$ is the probability of a node being in a clique and $\halpha_{\max}=1-s/n$.  The connectivity matrix is $P=[1,q;q,q]$ with $q=0.5$, since the probability of connectivity within the clique is $1$ and the probability of connectivity for any other node pair is $0.5$. 

Since the planted clique setting has unequal sized communities,  the general result in Section~\ref{thm:mainsample} is applicable, and we demonstrate how the assumptions $(B1)$-$(B5)$ simplify for the planted clique setting. We have that $\alpha_0=0$, since the communities are non-overlapping. For assumption $B2$, we have that
\beq\label{eqn:b2clique}\rho= \frac{\alpha_0+1}{\halpha_{\min}} = \frac{n}{s}, \quad n =\tl{\Omega}(\rho^2) \Rightarrow s = \tl{\Omega}(\sqrt{n}).\eeq For assumption $B3$, we have that $\sigma_{\min}(P)=\Theta(1)$ and that $\max_i (P\halpha)_i \leq s/n + q \leq 2$, and thus the assumption $B3$ simplifies as
\beq\label{eqn:b3clique} \zeta:=\left(\frac{\halpha_{\max}}{\halpha_{\min}}\right)^{1/2}\frac{
\sqrt{(\max_i (P\halpha)_i)}}{ \sigma_{\min}(P)} = \tl{O}\left(\frac{\sqrt{n}}{\rho}\right) \,\,\Rightarrow\,\, s = \tl{\Omega}\left(n^{2/3} \right).\eeq
The condition in \eqref{eqn:b2clique} that $s=\tl{\Omega}(n^{1/2})$ matches the computational lower bounds for recovering the clique~\citep{FGRVX12}. Unfortunately, the condition in \eqref{eqn:b3clique} that $s = \tl{\Omega}\left(n^{2/3}\right)$ is worse. This is required for assumption $(B3)$ to hold, which is needed to ensure the success of the tensor power method. The whitening step is particularly sensitive to the condition number of the matrix to be whitened (i.e., matrices $F_A, F_B, F_C$ in our case and the condition numbers for these matrices  depend on the ratio of the community sizes), which results in a weaker guarantee.  Thus, our method does not perform very well when the community sizes are drastically different. It remains an open question if our method can be improved in this setting. We conjecture that using ``peeling'' ideas similar to~\cite{ailon2013breaking}, where the communities are recovered one by one can improve our dependence on the ratio of community sizes.

\subsection{Proof Outline}\label{sec:outline}
We now summarize the main techniques involved in proving  Theorem~\ref{thm:mainsample}. The details are in the Appendix. The main ingredient is the concentration of the adjacency matrix: since the edges are drawn independently conditioned on the community memberships, we establish that the adjacency matrix concentrates around its mean under the stated assumptions. See Appendix~\ref{app:adjperturb} for details. With this in hand, we can then establish concentration of various quantities used by our learning algorithm.

\paragraph{Step 1: Whitening matrices.} We first establish
concentration bounds on the whitening matrices  $\h{W}_A$, $\h{W}_B$, $\h{W}_C$ computed using empirical moments, described in Section~\ref{sec:preprocess}. With this in hand, we can approximately recover the span of matrix $F_A$ since $\h{W}_A^\top F \Diag(\halpha_i)^{1/2}$ is a rotation matrix. The main technique employed is the Matrix Bernstein's inequality~\citep[thm. 1.4]{tropp2012user}. See Appendix~\ref{app:whitenperturb} for details.

\paragraph{Step 2: Tensor concentration bounds} Recall that we use the whitening matrices to obtain a symmetric orthogonal tensor. We establish that the whitened and symmetrized tensor concentrates around its mean. (Note that the empirical third order tensor $T_{X\rightarrow A,B,C}$ tends to its expectation conditioned on $\Pi_A, \Pi_B, \Pi_C $ when $|X|\to \infty$). This is done in several stages and we carefully control the tensor perturbation bounds.  See Appendix~\ref{app:tensorperturb} for details.

\paragraph{Step 3: Tensor power method analysis.} We analyze the performance of Procedure~\ref{alg:robustpower} under empirical moments. We employ the various improvements, detailed in Section~\ref{sec:powermodify} to establish guarantees on the recovered eigen-pairs. This includes coming up with a condition on the tensor perturbation bound, for the tensor power method to succeed. It also involves establishing that there exist good initializers for the power method among (whitened) neighborhood vectors. This allows us to obtain stronger guarantees for the tensor power method, compared to earlier analysis  by \cite{AGHKT12}. This analysis is crucial for us to obtain state-of-art scaling bounds for guaranteed recovery (for the special case of stochastic block model). See Appendix~\ref{app:tensorpower} for details.

\paragraph{Step 4: Thresholding of estimated community vectors} In Step 3, we provide guarantees for recovery of each eigenvector in $\ell_2$ norm. Direct application of this result only allows us to obtain $\ell_2$ norm bounds for row-wise recovery of the community matrix $\Pi$. In order to strengthen the result to an $\ell_1$ norm bound, we threshold the estimated $\Pi$ vectors. Here, we exploit the sparsity in Dirichlet draws and carefully control the contribution of weak entries in the vector. Finally, we establish perturbation bounds on $P$ through rather straightforward concentration bound arguments.
 See Appendix~\ref{app:recon} for details.

\paragraph{Step 5: Support recovery guarantees.} To simplify the argument, consider the stochastic block model. Recall that Procedure~\ref{algo:support} readjusts the community membership estimates based on degree averaging.   For each vertex, if we count the average degree towards these
``approximate communities'', for the correct community the result is concentrated around value $p$ and for the wrong community the result is around value $q$. Therefore, we can correctly identify the   community memberships of all the nodes, when $p-q$ is sufficiently large, as specified by A3. The argument can be easily extended to general mixed membership models.
See Appendix~\ref{app:support} for details. 

\subsection{Comparison with Previous Results}\label{sec:compare}

We now compare the results of this paper to   our previous work~\citep{AGHKT12} on the use of tensor-based approaches for learning various latent variable models such as topic models, hidden Markov models (HMM) and Gaussian mixtures. At a high level, the tensor approach is exploited in a similar manner in all these models (including the community model in this paper), \viz that the conditional-independence relationships of the model result in a low rank tensor, constructed from low order moments under the given model. However, there are several important differences between the community model and the other latent variable models considered by~\citet{AGHKT12} and we list them below. We also precisely list the various algorithmic improvements proposed in this paper with respect to the tensor power method, and how they can be applicable to other latent variable models.

\subsubsection{Topic model vs. community model}

Among the latent variable models studied by~\citet{AGHKT12}, the topic model, \viz latent Dirichlet allocation (LDA), bears the closest resemblance to MMSB. In fact, the MMSB model was originally inspired by the LDA model. The analogy between the MMSB model and the LDA is direct under our framework and we describe it below.

Recall that  for learning MMSBs,  we consider a partition of the nodes $\{X,A,B,C\}$ and we consider the set of $3$-stars from set $X$ to $A,B,C$. We can construct an equivalent topic model as follows: the nodes in $X$ form the ``documents'' and for each  document $x\in X$, the neighborhood vectors $G_{xA}^\top, G_{xB}^\top, G_{xC}^\top$ form the three ``words'' or ``views'' for that document. 
In each document $x\in X$, the community vector $\pi_x$ corresponds to the ``topic vector'' and the matrices $F_A$, $F_B$ and $F_C$ correspond to the topic-word matrices. Note that the three views $G_{xA}^\top, G_{xB}^\top, G_{xC}^\top$ are conditionally independent given the topic vector $\pi_x$. Thus, the community model can be cast as a topic model or a multi-view model. See Figure~\ref{fig:comm-topic}.

\begin{figure}[t]
\subfloat[a][Community model as a topic model]
{\begin{minipage}{3.1in}
\bc\bp\psfrag{documents}[l]{Documents}
\psfrag{view 1}[l]{View 1}
\psfrag{view 2}[l]{View 2}
\psfrag{view 3}[l]{View 3}
\psfrag{x}[l]{$x$}
\psfrag{X}[l]{$X$}
\psfrag{A}[l]{$A$}
\psfrag{B}[l]{$B$}
\psfrag{C}[l]{$C$}
\includegraphics[width=2.5in,height=1.5in]{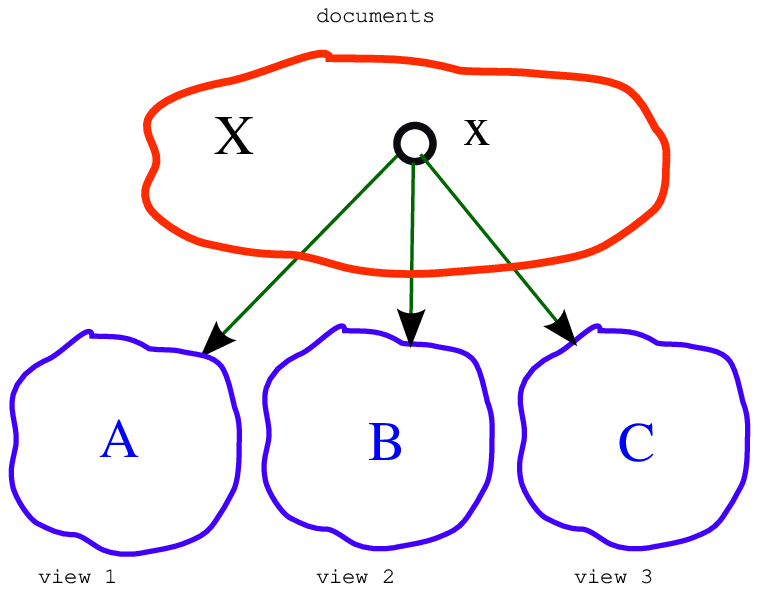}
\ep

\tcw{\Large{.}}\ec\end{minipage}}\hfil
\subfloat[b][Graphical model representation]
{\begin{minipage}{3.1in}
\bc\bp\psfrag{x}[l]{$\pi_x$}
\psfrag{A}[l]{$G^\top_{x,A}$}
\psfrag{B}[l]{$G^\top_{x,B}$}
\psfrag{C}[l]{$G^\top_{x,C}$}
\includegraphics[width=2.5in,height=1.5in]{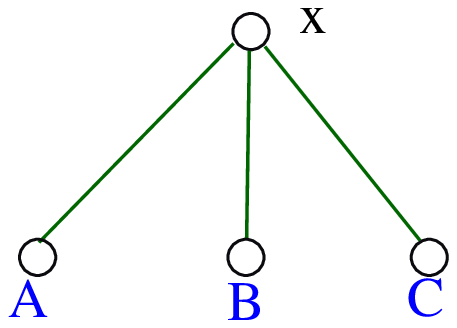}
\ep

\tcw{\Large{.}}\ec\end{minipage}}\caption{Casting the community model as a topic model, we obtain conditional independence of the three views.}\label{fig:comm-topic}
\end{figure}

Although the community model can be viewed as a topic model, it has some important special properties which allows us to provide better guarantees. The topic-word matrices $F_A, F_B, F_C$ are not arbitrary matrices. Recall that $F_A:= \Pi_A^\top P^\top$ and similarly $F_B, F_C$ are random matrices and we can provide strong concentration bounds for these matrices by appealing to random matrix theory. Moreover, each of the views in the community model has additional structure, \viz the vector $G^\top_{x,A}$ has independent Bernoulli entries conditioned on the community vector $\pi_x$, while in a general multi-view model, we only specify the conditional distribution of each view given the hidden topic vector. This further allows us to provide specialized concentration bounds for the community model. Importantly, we can recover the community memberships (or topic vectors) accurately while for a general multi-view model  this cannot be guaranteed and we can only hope to recover the model parameters.

\subsubsection{Improvements to tensor recovery guarantees in this paper}

In this paper, we make modifications to the tensor power method of~\citet{AGHKT12} and obtain better guarantees for the community setting. Recall that the two modifications are adaptive deflation and initialization using whitened neighborhood vectors. The adaptive deflation leads to a weaker gap condition for an initialization vector to succeed in estimating a tensor eigenvector efficiently. Initialization using whitened neighborhood vectors allows us to tolerate more noise in the estimated $3$-star tensor, thereby improving our sample complexity result. We make this improvement precise below.

If we directly apply the tensor power method of~\citet{AGHKT12}, without considering the modifications, we require a stronger condition on the sample complexity and edge connectivity. For simplicity, consider the homogeneous setting of Section~\ref{sec:special}. The conditions $(A2)$ and $(A3)$ now need to be replaced with  stronger conditions:
\paragraph{[A2'] Sample complexity: }The number of samples satisfies
\[ n = \tl{\Omega}(k^4 (\alpha_0+1)^2). \] 
\paragraph{[A3'] Edge connectivity: }The edge connectivity parameters $p,q$ satisfy
\[ \frac{p-q}{\sqrt{p}} = \Omega\left(\frac{(\alpha_0+1)k^2}{\sqrt{n}}\right).\]  Thus, we obtain significant improvements in recovery guarantees via algorithmic modifications and careful analysis of concentration bounds.
 
The guarantees derived in this paper are specific to the community setting, and we outlined previously the special properties of the  community model when compared to a general multi-view model. However, when the documents of the topic model are sufficiently long, the word frequency vector within a document has good concentration, and our modified tensor method has better recovery guarantees in this setting as well. Thus, the improved tensor recovery guarantees derived in this paper are applicable in scenarios where we have access to better initialization vectors rather than simple random initialization.  
\section{Conclusion}

In this paper, we presented a novel approach for learning overlapping communities based on a tensor decomposition approach. We established that our method is guaranteed to recover the underlying community memberships correctly, when the communities are  drawn from a mixed membership stochastic block model (MMSB). Our method is also computationally efficient and requires simple linear algebraic operations and tensor iterations.  Moreover, our method is tight for the special case of the stochastic block model (up to poly-log factors), both in terms of sample complexity and the separation between edge connectivity  within a community and across different communities.

We now note a number of interesting open problems and extensions. While we obtained tight guarantees for MMSB models with uniform sized communities, our guarantees are weak when the community sizes are drastically different, such as in the planted clique setting where we do not match the computational lower bound~\citep{FGRVX12}. The whitening step in the tensor decomposition method is particularly sensitive to the ratio of community sizes and it is interesting to see if modifications can be made to our algorithm to provide tight guarantees under unequal community sizes.
While this paper mostly dealt with the theoretical analysis of the tensor method for community detection, we note recent experimental results where the   tensor method is deployed  on graphs with millions of nodes with very good accuracy and running times~\citep{AnandkumarEtal:communityimplementation13}. In fact, the running times are more than an order of magnitude better than the state-of-art variational approach for learning MMSB models.  The work of~\citep{AnandkumarEtal:communityimplementation13} makes an important modification to make the method scalable, \viz that the tensor decomposition is carried out through stochastic updates in parallel unlike the serial batch updates considered here. Establishing theoretical guarantees for stochastic tensor decomposition is an important problem. Moreover, we have limited ourselves to the MMSB models, which assumes a linear model for edge formation, which is not applicable universally. For instance, exclusionary relationships, where two nodes cannot be connected because of their memberships in certain communities cannot be imposed in the MMSB model. Are there other classes of mixed membership models which do not suffer from this restriction, and yet  are identifiable and are amenable for learning?   Moreover, the Dirichlet distribution in the MMSB model imposes constraints on the memberships across different communities. Can we incorporate mixed memberships with arbitrary correlations?
The answers to these questions will further push the boundaries of tractable learning of mixed membership communities models.

\subsubsection*{Acknowledgements} We thank the JMLR Action Editor  Nathan Srebro and the anonymous reviewers for comments which significantly improved this manuscript. We thank Jure Leskovec  for helpful discussions regarding various community models.
Part of this work was done when AA and RG were visiting MSR New England. AA is supported in part by the Microsoft faculty fellowship, NSF Career award CCF-1254106, NSF Award CCF-1219234 and the ARO YIP Award W911NF-13-1-0084.

\bibliography{community} \bibliographystyle{plainnat}

\newpage
\appendix

\section{Tensor Power Method Analysis}\label{app:tensorpower}

In this section, we leverage on the perturbation analysis for tensor power method in~\cite{AGHKT12}. As discussed in Section~\ref{sec:powermodify}, we propose the following modifications to the tensor power method and obtain guarantees below for the modified method. The   two main modifications are: (1) we modify the tensor deflation process in the robust power method in Procedure~\ref{alg:robustpower}. Rather than a fixed deflation step after obtaining an estimate of the eigenvalue-eigenvector pair, in this paper, we deflate adaptively depending on the current estimate, and (2)rather than selecting random initialization vectors, as in~\cite{AGHKT12}, we initialize with  vectors obtained from adjacency matrix.    

Below in Section~\ref{sec:goodinit}, we establish success of the modified tensor method under ``good'' initialization vectors, as defined below. This involves improved error bounds for the modified deflation procedure provided in Section~\ref{sec:deflation}. In Section~\ref{sec:dirichlet}, we subsequently establish that under the Dirichlet distribution (for small $\alpha_0$), we obtain ``good'' initialization vectors.

\subsection{Analysis under good initialization vectors}\label{sec:goodinit}

We now show that when ``good'' initialization vectors are input to tensor power method in Procedure~\ref{alg:robustpower}, we obtain good estimates of eigen-pairs under appropriate choice of number of iterations $N$ and spectral norm $\epsilon$ of tensor perturbation.

Let $T=\sum_{i\in [k]}\lambda_i v_i$, where $v_i$ are orthonormal vectors and $\lambda_1\geq \lambda_2\geq\ldots \lambda_k$. Let $\tT=T+E$ be the perturbed tensor with $\|E\|\leq \epsilon$. Recall that $N$ denotes the number of iterations of the tensor power method.


We call an initialization vector $u$ to be $(\gamma, R_0)$-good  if there exists $v_i$ such that $\inner{u, v_i}> R_0$
  and \beq\label{eqn:gamma}|\inner{u, v_i}| -\max_{j<i} |\inner{u,v_j}| > \gamma  |\inner{u,v_i}|.\eeq   Choose $\gamma=1/100$.


\begin{theorem}
\label{thm:robustpower}
There exists universal constants $C_1, C_2 > 0$  such that the
following holds.
\beq\label{eqn:robustpowerconditions}
\epsilon \leq C_1 \cdot \lambda_{\min} R_0^2,
\qquad
N \geq C_2 \cdot \left( \log(k) + \log\log\left(
\frac{\lambdamax}{\eps} \right) \right)
,
\eeq Assume there is at least one good initialization vector corresponding to each $v_i$, $i\in [k]$. The parameter $\xi$ for choosing deflation vectors in each iteration of the tensor power method in Procedure~\ref{alg:robustpower}  is chosen as $\xi\geq 25 \eps$. We obtain  eigenvalue-eigenvector pairs  $(\hat\lambda_1,\hat{v}_1), (\hat\lambda_2,\hat{v}_2), \dotsc,
(\hat\lambda_k,\hat{v}_k)$ such that  there exists a permutation $\pi$ on
$[k]$ with
\[
\|v_{\pi(j)}-\hat{v}_j\| \leq 8 \epsilon/\lambda_{\pi(j)}
, \qquad
|\lambda_{\pi(j)}-\hat\lambda_j| \leq 5\epsilon , \quad \forall j \in [k]
,
\]
and
\[
\left\|
T - \sum_{j=1}^k \hat\lambda_j \hat{v}_j^{\otimes 3}
\right\| \leq 55\eps .
\]
\end{theorem}

\paragraph{Remark 1 (need for adaptive deflation): }We now compare the above result with the result in~\cite[Thm. 5.1]{AGHKT12}, where similar guarantees are obtained for a simpler version of the tensor power method without any adaptive deflation and using random initialization. The main difference is in our requirement of the gap $\gamma$ in \eqref{eqn:gamma} for an initialization vector is weaker than the gap requirement in~\cite[Thm. 5.1]{AGHKT12}. This is due to the use of adaptive deflation in this paper.

\paragraph{Remark 2 (need for non-random initialization): }In this paper, we employ whitened neighborhood vectors generated under the MMSB model for initialization, while~\cite[Thm. 5.1]{AGHKT12} assumes a random initialization. Under random initialization, we obtain $R_0 \sim 1/\sqrt{k}$ (with poly$(k)$ trials), while for initialization using whitened neighborhood vectors, we subsequently establish that $R_0 = \Omega(1)$ is a constant, when number of samples $n$ is large enough.
We also establish that the gap requirement in \eqref{eqn:gamma} is satisfied   for the choice of $\gamma=1/100$ above. See Lemma~\ref{lemma:dirichletinit} for details. Thus, we can tolerate much larger perturbation $\epsilon$ of the third order moment tensor, when non-random initializations are employed.\\

\bprf The proof is on lines of the proof of~\cite[Thm. 5.1]{AGHKT12} but  here, we consider  the modified deflation procedure, which improves the condition on $\epsilon$ in \eqref{eqn:robustpowerconditions}. We provide the full proof below for completeness.

We prove by induction on $i$, the number of eigenpairs estimated so far by Procedure~\ref{alg:robustpower}. Assume that there exists a permutation $\pi$ on $[k]$ such that the
following assertions hold.
\begin{enumerate}
\item For all $j \leq i$, $\|v_{\pi(j)}-\hat{v}_j\| \leq 8
\epsilon/\lambda_{\pi(j)}$ and $|\lambda_{\pi(j)}-\hat\lambda_j| \leq 12
\epsilon$.

\item $D(u,i)$ is the set of deflated vectors given current estimate of the power method is $u\in S^{k-1}$: \[ D(u,i;\xi) := \{j : |\hlambda_i \hat{\theta}_i| \geq \xi\} \cap [i],\] where $\hat{\theta}_i:=\inner{u,\hv_i}$.
\item The error tensor
\begin{align*}
\tilde{E}_{i+1,u} & :=
\biggl( \hat{T} - \sum_{j \in D(u,i;\xi)} \hat\lambda_j \hat{v}_j^{\otimes 3} \biggr)
- \sum_{j \notin D(u,i;\xi)} \lambda_{\pi(j)} v_{\pi(j)}^{\otimes 3}
= E + \sum_{j \in D(u,i;\xi)}
\Bigl(
\lambda_{\pi(j)} v_{\pi(j)}^{\otimes 3}
-
\hat\lambda_j \hat{v}_j^{\otimes 3}
\Bigr)
\end{align*}
satisfies
\begin{align}
\|\tl{E}_{i+1,u}(I,u,u)\|
& \leq 56\eps , \quad \forall u \in S^{k-1} ;
\label{eq:regular-bound} \\
\|\tl{E}_{i+1,u}(I,u,u)\|
& \leq 2\eps , \quad \forall u \in S^{k-1} \ \text{s.t.}
\ \exists j \geq i+1 \centerdot (u^\t v_{\pi(j)})^2 \geq 1 -
(168\eps/\lambda_{\pi(j)})^2
.
\label{eq:smaller-bound}
\end{align}

\end{enumerate}
We   take $i = 0$ as the base case, so we can ignore the first
assertion, and just observe that for $i = 0$, $D(u,0;\xi) = \emptyset$ and thus
\[
\tilde{E}_{1,u} = \hat{T} - \sum_{j=1}^k \lambda_i v_i^{\otimes 3} = E, \quad \forall\, u \in S^{k-1}.
\]
We have $\|\tilde{E}_1\| = \|E\| = \eps$, and therefore the second
assertion holds.

Now fix some $i \in [k]$, and assume as the inductive hypothesis. The power iterations now take a subset of $j\in [i]$ for deflation, depending on the current estimate.
Set
\begin{equation}
    C_1 := \min\left\{
        (56\cdot9\cdot 102)^{-1},
        (100\cdot 168)^{-1},
        \Delta' \textup{ from Lemma~\ref{lem:deflation} with } \Delta = 1/50
    \right\}.
    \label{eq:C1:defn}
\end{equation}
For all good initialization vectors which are $\gamma$-separated relative to $\pi(j_{\max})$, we have (i) $|\th{j_{\max},0}^{(\tau)}| \geq
R_0$, and (ii) that by  \cite[Lemma B.4]{AGHKT12} (using $\teps/p :=
2\eps$, $\kappa := 1$, and $i^* := \pi(j_{\max})$, and providing $C_2$),
\begin{align*}
|\tilde{T}_i(\th{N}^{(\tau)},\th{N}^{(\tau)},\th{N}^{(\tau)})
- \lambda_{\pi(j_{\max})}|
& \leq 5\eps
\end{align*}
(notice by definition that $\gamma \geq 1/100$ implies
$\gamma_0 \geq 1 - 1/(1 + \gamma) \geq 1/101$,
thus it follows from the bounds on the other quantities that
$
\teps
= 2p\eps
\leq 56 C_1 \cdot  \lambdamin R_0^2
< \frac{\gamma_0}{2(1+8\kappa)}
\cdot \tlambdamin \cdot \th{i^*,0}^2$
as necessary).
Therefore $\th{N} := \th{N}^{(\tau^*)}$ must satisfy
\[
\tilde{T}_i(\th{N},\th{N},\th{N})
= \max_{\tau \in [L]}
\tilde{T}_i(\th{N}^{(\tau)},\th{N}^{(\tau)},\th{N}^{(\tau)})
\geq \max_{j \geq i} \lambda_{\pi(j)} - 5\eps
= \lambda_{\pi(j_{\max})} - 5\eps
.
\]
On the other hand, by the triangle inequality,
\begin{align*}
\tilde{T}_i(\th{N},\th{N},\th{N})
& \leq \sum_{j \geq i} \lambda_{\pi(j)} \th{\pi(j),N}^3
+ |\tilde{E}_i(\th{N},\th{N},\th{N})|
\\
& \leq \sum_{j \geq i}
\lambda_{\pi(j)} |\th{\pi(j),N}| \th{\pi(j),N}^2
+ 56\eps
\\
& \leq \lambda_{\pi(j^*)} |\th{\pi(j^*),N}|
+ 56\eps
\end{align*}
where $j^* := \arg\max_{j \geq i} \lambda_{\pi(j)} |\th{\pi(j),N}|$.
Therefore
\[
\lambda_{\pi(j^*)} |\th{\pi(j^*),N}|
\geq \lambda_{\pi(j_{\max})} - 5\eps - 56\eps
\geq \frac45 \lambda_{\pi(j_{\max})}
.
\]
Squaring both sides and using the fact that $\th{\pi(j^*),N}^2 +
\th{\pi(j),N}^2 \leq 1$ for any $j \neq j^*$,
\begin{align*}
\bigl( \lambda_{\pi(j^*)} \th{\pi(j^*),N} \bigr)^2
& \geq
\frac{16}{25} \bigl( \lambda_{\pi(j_{\max})} \th{\pi(j^*),N} \bigr)^2
+ \frac{16}{25} \bigl( \lambda_{\pi(j_{\max})} \th{\pi(j),N} \bigr)^2
\\
& \geq
\frac{16}{25} \bigl( \lambda_{\pi(j^*)} \th{\pi(j^*),N} \bigr)^2
+ \frac{16}{25} \bigl( \lambda_{\pi(j)} \th{\pi(j),N} \bigr)^2
\end{align*}
which in turn implies
\[
\lambda_{\pi(j)} |\th{\pi(j),N}|
\leq \frac34 \lambda_{\pi(j^*)} |\th{\pi(j^*),N}|
, \quad j \neq j^*
.
\]
This means that $\th{N}$ is $(1/4)$-separated relative to $\pi(j^*)$.
Also, observe that
\[
|\th{\pi(j^*),N}|
\geq \frac45 \cdot
\frac{\lambda_{\pi(j_{\max})}}{\lambda_{\pi(j^*)}}
\geq \frac45 ,
\quad
\frac{\lambda_{\pi(j_{\max})}}{\lambda_{\pi(j^*)}}
\leq \frac54 .
\]
Therefore by~\cite[Lemma B.4]{AGHKT12} (using $\teps/p := 2\eps$, $\gamma :=
1/4$, and $\kappa := 5/4$), executing another $N$ power iterations starting
from $\th{N}$ gives a vector $\hat\theta$ that satisfies
\[
\|\hat\theta - v_{\pi(j^*)}\|
\leq \frac{8\eps}{\lambda_{\pi(j^*)}}
, \qquad
|\hat\lambda - \lambda_{\pi(j^*)}|
\leq 5\eps
.
\]
Since $\hat{v}_i = \hat\theta$ and $\hat\lambda_i = \hat\lambda$, the first
assertion of the inductive hypothesis is satisfied, as we can modify the
permutation $\pi$ by swapping $\pi(i)$ and $\pi(j^*)$ without affecting the
values of $\{ \pi(j) : j \leq i-1 \}$ (recall $j^* \geq i$).

We now argue that  $\tilde{E}_{i+1,u}$ has the required properties to
complete the inductive step.
By Lemma~\ref{lem:deflation} (using $\teps := 5\eps$, $\xi=5 \teps = 25 \eps$ and $\Delta := 1/50$,
the latter providing one upper bound on $C_1$ as per \eqref{eq:C1:defn}),
we have for any unit vector $u \in S^{k-1}$,
\begin{equation} \label{eq:approx-bound}
\Biggl\|
\biggl(
\sum_{j \leq i}
\Bigl(
\lambda_{\pi(j)} v_{\pi(j)}^{\otimes 3}
-
\hat\lambda_j \hat{v}_j^{\otimes 3}
\Bigr)
\biggr)(I,u,u)
\Biggr\|
\leq \biggl( 1/50 + 100 \sum_{j=1}^i (u^\t v_{\pi(j)})^2 \biggr)^{1/2}
5\eps
\leq 55\eps
.
\end{equation}
Therefore by the triangle inequality,
\[
\|\tilde{E}_{i+1}(I,u,u)\|
\leq \|E(I,u,u)\|
+ \Biggl\|
\biggl(
\sum_{j \leq i}
\Bigl(
\lambda_{\pi(j)} v_{\pi(j)}^{\otimes 3}
-
\hat\lambda_j \hat{v}_j^{\otimes 3}
\Bigr)
\biggr)(I,u,u)
\Biggr\|
\leq 56\eps
.
\]
Thus the bound~\eqref{eq:regular-bound} holds.

To prove that~\eqref{eq:smaller-bound} holds, for any unit vector $u \in
S^{k-1}$ such that there exists $j' \geq i+1$ with $(u^\t
v_{\pi(j')})^2 \geq 1 - (168\eps/\lambda_{\pi(j')})^2$.
We have (via the second bound on $C_1$ in \eqref{eq:C1:defn} and the corresponding
assumed bound $\epsilon \leq C_1 \cdot \lambdamin R_0^2$)
\[
100 \sum_{j=1}^i (u^\t v_{\pi(j)})^2
\leq 100 \Bigl( 1 - (u^\t v_{\pi(j')})^2 \Bigr)
\leq 100 \biggl( \frac{168\eps}{\lambda_{\pi(j')}} \biggr)^2
\leq \frac1{50}
,
\]
and therefore
\[
\biggl( 1/50 + 100 \sum_{j=1}^i (u^\t v_{\pi(j)})^2 \biggr)^{1/2} 5\eps
\leq (1/50 + 1/50)^{1/2} 5\eps \leq \eps .
\]
By the triangle inequality, we have $\|\tilde{E}_{i+1}(I,u,u)\| \leq
2\eps$.
Therefore~\eqref{eq:smaller-bound} holds, so the second assertion of the
inductive hypothesis holds.
 We conclude that by the induction principle, there exists a permutation
$\pi$ such that two assertions hold for $i = k$.
From the last induction step ($i = k$), it is also clear
from~\eqref{eq:approx-bound} that $\|T - \sum_{j=1}^k \hat\lambda_j
\hat{v}_j^{\otimes 3}\| \leq 55\eps$.
This completes the proof of the theorem.  \eprf\\

\subsection{Deflation Analysis}\label{sec:deflation}

\begin{lemma}[Deflation analysis]\label{lem:deflation}Let $\teps>0$ and let $\{v_1, \ldots, v_k\}$ be an orthonormal basis for $\R^k$ and $\lambda_i \geq 0$ for $i\in [k]$. Let $\{\hv_1, \ldots, \hv_k\}\in \R^k$ be a set of unit vectors and $\hlambda_i\geq 0$. Define third order tensor $\deflate_i$ such that
\[\deflate_i:= \lambda_i v_i^{\otimes 3} - \hlambda_i \hv_i^{\otimes 3}, \quad \forall\,i \in k.\] For some $t\in [k]$ and a unit vector $u \in S^{k-1}$ such that $u= \sum_{i\in [k]} \theta_i v_i$ and  $\hat{\theta}_i:=\inner{u,\hv_i}$, we have   for $i\in [t]$,
\begin{align*} |\hlambda_i \hat{\theta}_i| &\geq \xi\geq 5\teps,\\
|\hlambda_i - \lambda_i| & \leq \teps , \\
\|\hv_i - v_i\| & \leq \min\{ \sqrt2, \ 2\teps / \lambda_i \},
\end{align*}
then, the following holds\begin{multline}\nn
\biggl\| \sum_{i=1}^t  \deflate_i(I,u,u) \biggr\|_2^2
\leq
\biggl(
4(5 + 11\teps/\lambdamin)^2
+ 128 ( 1 + \teps / \lambdamin )^2 (\teps / \lambdamin)^2
\biggr) \teps^2 \sum_{i=1}^t  \theta_i^2
\\
+ 64(1+\teps/\lambdamin)^2 \teps^2
+ 2048 ( 1 + \teps / \lambdamin )^2
\teps^2
.\label{eqn:deflationoriginal}
\end{multline}
In particular, for any $\Delta \in (0,1)$, there exists a constant $\Delta'
> 0$ (depending only on $\Delta$) such that $\teps \leq \Delta'
\lambdamin$ implies
\[
\biggl\| \sum_{i=1}^t  \deflate_i(I,u,u) \biggr\|_2^2 \le \biggl( \Delta +
100 \sum_{i=1}^t  \theta_i^2 \biggr) \teps^2 .
\]
  \end{lemma}

\bprf The proof is on lines of deflation analysis in~\cite[Lemma B.5]{AGHKT12}, but we improve the bounds based on additional properties of vector $u$.  From~\cite{AGHKT12}, we have that
for all $i \in [t]$, and any unit vector $u$,
\begin{align}\nn
\biggl\| \sum_{i=1}^t  \deflate_i(I,u,u) \biggr\|_2^2
\leq&
\biggl(
4(5 + 11\teps/\lambdamin)^2
+ 128 ( 1 + \teps / \lambdamin )^2 (\teps / \lambdamin)^2
\biggr) \teps^2 \sum_{i=1}^t  \theta_i^2
\\ &
+ 64(1+\teps/\lambdamin)^2 \teps^2 \sum_{i=1}^t (\teps / \lambda_i)^2
+ 2048 ( 1 + \teps / \lambdamin )^2
\teps^2 \biggl( \sum_{i=1}^t  (\teps / \lambda_i)^3 \biggr)^2
.\label{eqn:deflationoriginal}
\end{align}
Let $\hlambda_i = \lambda_i + \delta_i$ and $\hat{\theta}_i = \theta_i + \beta_i$. We have $\delta_i \leq \teps$ and $\beta_i \leq 2\teps/\lambda_i$, and that $|\hlambda_i \hat{\theta_i}| \geq \xi$.
\begin{align*}||\hlambda_i\hth{i}| - |\lambda_i\th{i}|| & \le |\hlambda_i\hth{i} - \lambda_i\th{i}| \\
& \le | (\lambda_i+\delta_i)(\th{i}+\beta_i) - \lambda_i\th{i}| \\
& \le | \delta_i \th{i} + \lambda_i \beta_i + \delta_i \beta_i | \\
& \le 4\teps.
 \end{align*} Thus, we have that $|\lambda_i \th{i}| \geq 5\teps - 4\teps=\teps$. Thus $\sum_{i=1}^t \teps^2 /\lambda_i^2 \leq \sum_i \th{i}^2\leq 1$. Substituting in \eqref{eqn:deflationoriginal}, we have the result.
 \eprf\\

\section{Proof of Theorem~\ref{thm:mainsample}}\label{sec:mainproof}

We now prove the main results on error bounds claimed in Theorem~\ref{thm:mainsample} for the estimated community vectors $\h{\Pi}$ and estimated block probability matrix $\h{P}$ in Algorithm~\ref{algo:main}. Below, we first show that the tensor perturbation bounds claimed in Lemma~\ref{lemma:tensoroutput} holds.

%

\paragraph{Notation: }Let $\|T\|$ denote the spectral norm for a tensor $T$ (or in special cases a matrix or a vector). Let $\|M\|_F$ denote the Frobenius norm. Let $|M_1|$ denote the operator $\ell_1$ norm, i.e., the maximum $\ell_1$ norm of its columns and $\|M\|_\infty$ denote the maximum $\ell_1$ norm of its rows. Let $\kappa(M)$ denote the condition number, i.e., $\frac{\|M \|}{\sigma_{\min}(M)}$.

\subsection{Proof of Lemma~\ref{lemma:tensoroutput}}\label{sec:tensoroutput}

%
%

From Theorem~\ref{thm:robustpower} in Appendix~\ref{app:tensorpower}, we see that the tensor power method returns  eigenvalue-vector pair $(\h{\lambda}_i, \h{\Phi}_i)$ such that there exists a permutation $\theta$ with
\beq\label{eqn:phi-proof} \max_{i\in [k]} \|\h{\Phi}_i- \Phi_{\theta(i)}\| \leq 8 \halpha_{\max}^{1/2} \veps_T ,\eeq and \beq \max_i | \lambda_i - \halpha_{\theta(i)}^{-1/2}| \leq 5 \veps_T,\eeq when  the perturbation of the tensor is small enough, according to
 \beq\label{eqn:Trequirement-app} \veps_T \leq   C_1 \halpha_{\max}^{-1/2} r_0^2 ,\eeq for some constant $C_1$,  when initialized with a $(\gamma, r_0)$ good vector. 

With the above result, two aspects need to be established: (1) the  whitened tensor perturbation $\epsilon_T$ is as claimed, (2) the condition in  \eqref{eqn:Trequirement-app} is satisfied and (3) there exist good initialization vectors when whitened neighborhood vectors are employed. The tensor perturbation bound $\epsilon_T$ is established in Theorem~\ref{theorem:tensorperturb} in Appendix~\ref{app:tensorperturb}.

Lemma~\ref{lemma:dirichletinit} establishes that when $\zeta = O(\sqrt{n}r_0^2/\rho)$, we have good initialization vectors with Recall 
$r_0^2 = \Omega(1/\halpha_{\max}k)$ when $\alpha_0 > 1$ and
$r_0^2 = \Omega(1)$ for $\alpha_0 \le 1$, and $\gamma=1/100$   with probability $1-9\delta$ under Dirichlet distribution, when \beq \label{eqn:ndirichletmainproof}n = \tl{\Omega}\left(\alpha_{\min}^{-1}  k^{0.43}\log(k/\delta)\right),\eeq  which is  satisfied since we
assume $\halpha_{\min}^{-2} < n$.  
 
We now show that the condition in \eqref{eqn:Trequirement-app} is satisfied under the assumptions B1-B4. Since $\eps_T $ is given by
\[ \veps_T =\tl{O}\left( \frac{\rho}{\sqrt{n}} \cdot \frac{\zeta}{\halpha_{\max}^{1/2}}\right),\]
 the condition in \eqref{eqn:Trequirement-app}
is equivalent to $\zeta = O(\sqrt{n}r_0^2/\rho)$.
Therefore when $\zeta = O(\sqrt{n}r_0^2/\rho)$, the assumptions
of Theorem~\ref{thm:robustpower} are satisfied.

\subsection{Reconstruction of $\Pi$ after tensor power method}\label{app:recon}

Let $(M)^i$ and $(M)_i$ denote the $i^{\tha}$ row and $i^{\tha}$ column in matrix $M$ respectively.
Let $Z \subseteq A^c$ denote any subset of nodes not in $A$, considered in Procedure LearnPartition Community. Define
\beq\label{eqn:tlPi}\tl{\Pi}_Z:=\Diag(\lambda)^{-1}\Phi^\top \h{W}_A^\top G_{Z,A}^\top.\eeq  Recall that the final estimate $\h{\Pi}_Z$ is obtained by thresholding $\tl{\Pi}_Z$ element-wise with threshold $\tau$  in Procedure~\ref{procedure:reconstruct}. We first analyze perturbation of $\tl{\Pi}_Z$.

\begin{lemma}[Reconstruction Guarantees for $\tl{\Pi}_Z$]\label{lemma:reconPi}
 Assuming Lemma~\ref{lemma:tensoroutput} holds and the tensor power method recovers eigenvectors and eigenvalues up to the guaranteed errors, we have with probability $1-122\delta$,
\begin{align*} \veps_\pi:=\max_{i\in Z}
\| (\tl{\Pi}_Z)^i - (\Pi_Z)^i\| &=O\left(\veps_T\halpha_{\max}^{1/2}\left(\frac{\halpha_{\max}}{\halpha_{\min}}\right)^{1/2} \|\Pi_Z\|  \right),\\ &=O\left(\rho\cdot \zeta\cdot \halpha_{\max}^{1/2} \left(\frac{\halpha_{\max}}{\halpha_{\min}}\right)^{1/2} \right)  \end{align*} where $\veps_T$ is given by \eqref{eqn:vepsTmainproof}.
\end{lemma}

\bprf We have $(\tl{\Pi}_Z)^i = \lambda_i^{-1} ((\Phi)_i)^\top \h{W}_A^\top G_{Z,A}^\top$.
We will now use perturbation bounds for each of the terms to get the result.

The first term is
\begin{align*}
&\| \Diag(\lambda_i)^{-1}-\Diag(\halpha^{1/2}_i)\|\cdot\|\Diag(\halpha^{1/2})
\tl{F}_A^\top\| \cdot \|\tl{F}_A\| \cdot \| \Pi_Z\|
\\ &\leq 5 \veps_T \halpha_{\max}\halpha_{\min}^{-1/2} (1+\veps_1)^2 \|\Pi_Z\|
\end{align*} from the fact that $\|\Diag(\halpha^{1/2})
\tl{F}_A^\top\| \leq 1+\veps_1$, where $\veps_1$ is given by  \eqref{eqn:veps}. The second term is
\begin{align*}
& \| \Diag(\halpha^{1/2})\|\cdot\|(\Phi)_i-\halpha^{1/2}_i
(\tl{F}_A)_i\| \cdot \|\tl{F}_A\| \cdot \| \Pi_Z\|
\\ &\leq 8\halpha_{\max}\veps_T\halpha_{\min}^{-1/2} (1+\veps_1) \|\Pi_Z\|\end{align*}
The third term is
\begin{align}\nn&\| \halpha^{1/2}_i\|\cdot \| (\h{W}_A^\top  - W_A^\top) F_A \Pi_Z\|
\\ &\leq \halpha_{\max}^{1/2}\halpha_{\min}^{-1/2} \|\Pi_Z\|\epsilon_W \\  &\leq O\left(
\left(\frac{\halpha_{\max}}{\halpha_{\min}}\right)^{1/2} \veps_T
 \halpha_{\min}^{1/2}\|\Pi_Z\|
\right),\label{eqn:thirdterm-recon}
\end{align}from Lemma~\ref{lemma:WAerror} and finally, we have
\begin{align}\nn&\| \halpha^{1/2}_i\|\cdot \| W_A\| \cdot\|G^\top_{Z,A}  -   F_A \Pi_Z\|\\ &\leq
O\left(\halpha_{\max}^{1/2}  \frac{\sqrt{\alpha_0+1}}{\halpha_{\min} \sigma_{\min}(P)}
\sqrt{(\max_i (P\halpha)_i)(1+\veps_2+\veps_3)\log \frac{k}{\delta}}\right) \\ & \leq
O\left(\left(\frac{\halpha_{\max}}{\halpha_{\min}}\right)^{1/2} \veps_T
\sqrt{ \alpha_0+1} (1+\veps_2 + \veps_3) \sqrt{\frac{\log k}{\delta}} \right)
\label{eqn:fourthterm-recon}\end{align}from   Lemma~\ref{lemma:adj} and Lemma~\ref{lemma:specnormexact}.

The third term in \eqref{eqn:thirdterm-recon} dominates the last term in \eqref{eqn:fourthterm-recon} since $(\alpha_0+1)\log k/\delta < n \halpha_{\min}$ (due to assumption B2 on scaling of $n$). \eprf\\


We now show that if we threshold the entries of $\tl{\Pi}_Z$, the the resulting
matrix $\hat{\Pi}_Z$ has rows  close to those in $\Pi_Z$ in $\ell_1$ norm.

\begin{lemma}[Guarantees after thresholding]\label{lem:thresholding}
For $\h{\Pi}_Z:=\thres(\tl{\Pi}_Z, \,\tau)$, where $\tau$ is the threshold, we have with probability $1-2\delta$, that
\begin{align*} \veps_{\pi,\ell_1}:=\max_{i\in [k]}|(\h{\Pi}_Z)^i - (\Pi_Z)^i|_1 = O&\left(\sqrt{n\eta}\, \veps_\pi \sqrt{ \log\frac{1}{ 2\tau}}\left(1-\sqrt{\frac{2\log(k/\delta)}
{n\eta\log(1/2\tau)}
}\right)\right. \\ & \left.+ n\eta \tau\, +\,
\sqrt{(n \eta +4\tau^2 ) \log\frac{k}{\delta}}
 + \frac{\veps_\pi^2}{\tau}\right) ,\end{align*} where $\eta=\halpha_{\max}$ when $\alpha_0<1$ and $\eta=\alpha_{\max}$ when $\alpha_0\in [1,k)$. \end{lemma}

\paragraph{Remark 1: }The above guarantee on $\h{\Pi}_Z$ is stronger than for $\tl{\Pi}_Z$ in Lemma~\ref{lemma:reconPi} since   this is an $\ell_1$ guarantee on the rows compared to $\ell_2$ guarantee on rows for $\tl{\Pi}_Z$.

\paragraph{Remark 2: }When $\tau$ is chosen as \[\tau= \Theta(\frac{\veps_\pi}{\sqrt{n \eta}})=
\Theta\left(\frac{ \rho^{1/2}\cdot\zeta \cdot \halpha_{\max}^{1/2}}{n^{1/2} \cdot\halpha_{\min}} \right),    \] we have that
\begin{align*}  \max_{i\in [k]}|(\h{\Pi}_Z)^i - (\Pi_Z)^i|_1
&= \tl{O}\left( \sqrt{n \eta}\cdot \veps_\pi\right)\\ &=\tl{O}\left(n^{1/2}\cdot\rho^{3/2}\cdot \zeta\cdot \halpha_{\max} \right)\end{align*}



\bprf Let $S_i:=\{j: \h{\Pi}_Z(i,j)> 2\tau\}$. For a vector $v$, let $v_{S}$ denote the sub-vector by considering entries in set $S$. We now have
\[ |(\h{\Pi}_Z)^i - (\Pi_Z)^i|_1 \leq
|(\h{\Pi}_{Z})^i_{S_i} - (\Pi_{Z})^i_{S_i}|_1
+ |(\Pi_{Z})^i_{S^c_i}|_1+ |(\h{\Pi}_{Z})^i_{S^c_i}|_1\]

\paragraph{Case $\alpha_0<1$: }
From Lemma~\ref{lem:dirichletsparse}, we  have $\Pbb[\Pi(i,j)\geq 2\tau]\leq 8\halpha_i\log(1/2\tau)$. Since $\Pi(i,j)$ are independent for $j \in Z$, we have from multiplicative Chernoff bound~\cite[Thm 9.2]{Kearns&Vazirani:book}, that with probability $1-\delta$,
\[ \max_{i\in [k]}|S_i|< 8n\halpha_{\max}\log\left(\frac{1}{2\tau}\right)
\left(1-\sqrt{\frac{2\log(k/\delta)}{n\halpha_i\log(1/2\tau)}
}\right).\]
 We have \[ |(\tl{\Pi}_{Z})_{S_i}^i - (\Pi_{Z})^i_{S_i}|_1\leq \veps_\pi |S_i|^{1/2},\]
and the $i^{\tha}$ rows of $\tl{\Pi}_Z$ and $\h{\Pi}_Z$ can differ on $S_i$, we have $|\tl{\Pi}_Z (i,j) - \h{\Pi}_Z(i,j)| \leq \tau$, for $j \in S_i$, and number of such terms is at most $\veps_\pi^2/\tau^2$. Thus,   \[ |(\tl{\Pi}_{Z})_{S_i}^i - (\h{\Pi}_{Z})^i_{S_i}|_1
  \leq \frac{\veps_\pi^2}{\tau}.\]For the other term, from Lemma~\ref{lem:dirichletsparse}, we  have
\[\Ebb[\Pi_Z(i,j) \cdot \delta(\Pi_Z(i,j) \leq 2\tau )]
\leq \halpha_i (2\tau ). \] Applying Bernstein's bound
we have with probability $1-\delta$
\[ \max_{i \in [k]}\sum_{j \in Z} \Pi_Z(i,j) \cdot \delta(\Pi_Z(i,j) \leq 2\tau )
\leq n \halpha_{\max} (2\tau ) +
\sqrt{2(n \halpha_{\max} +4\tau^2 ) \log\frac{k}{\delta}}
.\]
For $\h{\Pi}^i_{S_i^c}$, we further divide $S_i^c$
 into $T_i$ and $U_i$, where $T_i:=\{j: \tau/2 < \Pi_Z(i,j)\le 2\tau\}$
 and $U_i:=\{j: \Pi_Z(i,j)\le \tau/2\}$.

In the set $T_i$, using similar argument we know
$|(\Pi_Z)^i_{T_i} - (\tl{\Pi}_Z)^i_{T_i}|_1
 \le O(\veps_\pi \sqrt{n\halpha_{\max} \log 1/\tau})$,
therefore
$$|\h{\Pi}^i_{T_i}|_1 \le |\tl{\Pi}^i_{T_i}|_1
\le |\Pi^i_{T_i} - \tl{\Pi}^i_{T_i}|_1+|\Pi^i_{S_i^c}|_1 \le O(\veps_\pi \sqrt{n\halpha_{\max} \log 1/\tau}).$$

Finally, for index $j\in U_i$, in order for
$\h{\Pi}_Z(i,j)$ be positive, it is required that $\tl{\Pi}_Z(i,j) - \Pi_Z(i,j) \ge \tau/2$. In this
case, we have
$$|(\h{\Pi}_Z)^i_{U_i}|_1 \le \frac{4}{\tau} \norm{(\tl{\Pi}_Z)^i_{U_i}
-\Pi^i_{U_i}}^2 \le \frac{4\veps_\pi^2}{\tau}.$$
%
%

%
%
%
%
%

\paragraph{Case $\alpha_0\in [1,k)$: }From Lemma~\ref{lem:dirichletsparse}, we see that the results hold when we replace $\halpha_{\max}$ with $\alpha_{\max}$.
 \eprf

\subsection{Reconstruction of $P$ after tensor power method}\label{app:reconP}

Finally we would like to use the community vectors $\Pi$ and the adjacency matrix $G$ to estimate the $P$ matrix. Recall that in the generative model, we have $\Ebb[G]= \Pi^\top P \Pi$. Thus, a straightforward estimate is to use $(\h{\Pi}^\dagger)^\top G \h{\Pi}^\dagger$. However, our guarantees on $\h{\Pi}$ are not strong enough to control the error on $\h{\Pi}^\dagger$ (since we only have  row-wise $\ell_1$ guarantees).

We propose an alternative estimator $\h{Q}$ for $\h{\Pi}^\dagger$ and use it to find $\h{P}$ in Algorithm~\ref{algo:main}.
Recall that the $i$-th row of $\h{Q}$ is given by
\[\h{Q}^i := (\alpha_0+1)\frac{\h{\Pi}^i}{|\h{\Pi}^i|_1}
-\frac{\alpha_0}{n}\vec{1}^\top.\]
Define $Q$  using exact communities, i.e.  \[Q^i := (\alpha_0+1)\frac{\Pi^i}{|\Pi^i|_1}
-\frac{\alpha_0}{n}\vec{1}^\top.\]
%
%
  We show below that $\h{Q}$ is close to $\Pi^\dagger$, and therefore, $\h{P}:= \h{Q}^\top G \h{Q}$ is close to $P$ w.h.p.

\begin{lemma}[Reconstruction of $P$]\label{lem:P}
With  probability $1-5\delta$,
$$ \veps_P :=\max_{i,j\in [n]}|\h{P}_{i,j} - P_{i,j}|
\le  O\left(
\frac{(\alpha_0+1)^{3/2}\veps_\pi (P_{\max}-P_{\min})}{\sqrt{n}}  \halpha_{\min}^{-1}  \halpha_{\max}^{1/2}\log\frac{nk }{\delta}  \right)$$
\end{lemma}

\paragraph{Remark: }If we define a new matrix $Q'$ as $(Q')^i : = \frac{\alpha_0+1}{n\halpha_i}\Pi^i -\frac{\alpha_0}{n}\vec{1}^\top$, then $\E_\Pi[Q'\Pi^\top] = I$. Below,   we show that $Q'$ is close to  $Q$ since $\E[|\Pi^i|_1] = n\halpha_i$ and thus the above result holds. We require  $Q$ to be  normalized by $|\Pi^i|_1$ in order to ensure  that the first term of $Q$ has equal column norms,  which will be used in our proofs subsequently.

\bprf The proof goes in three
steps:
$$P \approx Q\Pi^\top P \Pi Q^\top \approx QG Q^\top\approx
\h{Q} G \h{Q}^\top.$$

Note that $\Ebb_{\Pi}[\Pi Q^\top]=I$ and by Bernstein's bound, we can claim that   $\Pi Q^\top$ is   close to $I$ and can show that the $i$-th row of $Q\Pi^\top$ satisfies
\[ \Delta_i :=|(Q\Pi^\top)^i- e_i^\top|_1=O\left(k\sqrt{ \log\left(\frac{nk}{\delta}\right)\frac{\halpha_{\max}}{\halpha_{\min}}}\frac{1}{\sqrt{n}}\right) \] with probability $1-\delta$. Moreover,
\begin{align*}
|(\Pi^\top P \Pi Q^\top)_{i,j}
-(\Pi^\top P)_{i,j}|
&\le |(\Pi^\top P)^i((Q)_j-e_j)|
= |(\Pi^\top P)^i \Delta_j|\\
&\le O\left(\frac{P_{\max} k\cdot \sqrt{\halpha_{\max}/\halpha_{\min}}}{\sqrt{n}}\sqrt{\log\frac{nk}{\delta}}\right).
\end{align*}using the fact that $(\Pi^\top P)_{i,j}\leq P_{\max}$.

Now we claim that $\h{Q}$ is close to $Q$ and it can be shown that
\beq\label{eqn:Qperturb}|Q^i - \h{Q}^i|_1 \le O\left(\frac{\veps_P}{P_{\max}-P_{\min}}\right)\eeq

Using \eqref{eqn:Qperturb}, we have
\begin{align*}
|(\Pi^\top P \Pi Q^\top)_{i,j}
-(\Pi^\top P \Pi \h{Q}^\top)_{i,j}|
& = |(\Pi^\top P \Pi)^i (Q^\top - \h{Q}^\top)_j|\\& = ( (\Pi^\top P \Pi)^i - P_{\min} \vec{1}^\top)  |(Q^\top - \h{Q}^\top)_j|_1\\
& \le O((P_{\max}-P_{\min})|(Q^\top - \h{Q}^\top)_j|_1)=O(\veps_P).
\end{align*} using the fact that $(Q^j - \h{Q}^j)\vec{1} = 0$, due to the normalization.


Finally, $|(G \h{Q}^\top)_{i,j}(\Pi^\top P \Pi \h{Q}^\top)_{i,j}|$ are small
by standard concentration bounds (and the differences are of lower order). Combining these
  $|\h{P}_{i,j}-P_{i,j}|\le O(\veps_P)$.

\eprf

\subsection{Zero-error support recovery guarantees}\label{app:support}

Recall that we proposed Procedure~\ref{algo:support} to provide improved support recovery estimates in the special case of homophilic models (where there are more edges within a community than to any community outside). We limit our analysis to the special case of uniform sized communities $(\alpha_i =1/k)$ and matrix $P$ such that $P(i,j)= p \Ibb(i=j) + q \Ibb(i\neq j)$ and $p\geq q$. In principle, the analysis can be extended to homophilic models with more general $P$ matrix (with suitably chosen thresholds for support recovery).

We first consider analysis for the stochastic block model (i.e.
$\alpha_0\rightarrow 0$) and prove the guarantees claimed in Corollary~\ref{cor:zeroerrorblock}.\\

\bprfof{Corollary~\ref{cor:zeroerrorblock}} Recall the definition of $\tl{\Pi}$ in \eqref{eqn:tlPi} and $\h{\Pi}$ is obtained by thresholding $\tl{\Pi}$ with threshold $\tau$.
Since the threshold $\tau$ for stochastic block models  is $0.5$ (assumption B5), we have
\beq\label{eqn:hPiblock} |(\h{\Pi})^i - (\Pi)^i|_1 = O(\veps_\pi^2), \eeq where $\veps_\pi$ is the row-wise $\ell_2$ error for $\tl{\Pi}$ in Lemma~\ref{lemma:reconPi}. This is because $\Pi(i,j)\in \{0,1\}$, and in order for our method to make a mistake, it takes $1/4$ in the $\ell_2^2$ error.


In Procedure~\ref{algo:support}, for the stochastic block model $(\alpha_0=0)$, for a node $x\in [n]$, we have
\[\h{F}(x, i) = \sum_{y \in [n]}G_{x, y} \frac{\h{\Pi}(i,y)}{|\h{\Pi}^i|_1}\approx \sum_{y \in [n]}G_{x, y} \frac{\h{\Pi}(i,y)}{|\Pi^i|_1}\approx\frac{k}{n}\sum_{y \in [n]}G_{x, y} \h{\Pi}(i,y) , \]using \eqref{eqn:hPiblock} and the fact that the size of each community on average is $n/k$.
In other words,
 for each vertex $x$, we compute
the average number of edges from this vertex to all the estimated
communities according to $\h{\Pi}$, and set it to belong to the one with largest average degree.
Note that the margin of error on average for each node to be assigned the correct community according to the above procedure is $(p-q)n /k$, since the size of each community is $n/k$ and the average number of intra-community edges at a node is $p n/k$ and edges to any different  community at a node is $q n/k$.
From \eqref{eqn:hPiblock}, we have that the average number of errors made is $O((p-q) \veps_\pi^2)$. Note that the degrees concentrate around their expectations according to Bernstein's bound and the fact that the edges used for averaging is independent from the edges used for estimating $\h{\Pi}$.
Thus, for our method to succeed in inferring the correct community at a node, we require, \[O((p-q) \veps_\pi^2 )\leq (p-q)\frac{n}{k},\] which implies
$$p-q \ge \tl{\Omega}\left(\frac{\sqrt{p} k}{\sqrt{n}}\right).$$
 \eprfof\\

We now prove the general result on support recovery.\\

\bprfof{Theorem~\ref{thm:support}} From Lemma~\ref{lem:P},
  \[|\h{P}_{i,j}-P_{i,j}|\le O(\veps_P)\]
  which implies bounds for
the average of diagonals $H$ and average of off-diagonals $L$:
\[ |H-p|=O(\veps_P), \quad |L-q| = O(\veps_P).\]
On similar lines as the proof of Lemma~\ref{lem:P} and from independence of edges used to define $\h{F}$ from the edges used to estimate $\h{\Pi}$, we also have
\[|\h{F}(j,i) - F(j,i)| \le O( \veps_P).\]
Note that $F_{j,i}=q+\Pi_{i,j}(p-q) $.
The threshold $\xi$ satisfies $\xi =\Omega( \veps_P)$, therefore,
all the entries in $F$ that are larger than
$q+(p-q)\xi$, the corresponding entries in $S$ are declared to be one, while
none of the entries that are smaller than $q+(p-q)\xi/2$ are set to one in $S$.
\eprfof


\section{Concentration Bounds}

\subsection{Main Result: Tensor Perturbation Bound}\label{app:tensorperturb}

We now provide the main result that the third-order whitened tensor computed from samples concentrates. Recall that $\Triples^{\alpha_0}_{Y\rightarrow \{A,B,C\}}$ denotes the third order moment computed using edges from partition $Y$ to partitions $A, B, C$ in \eqref{eqn:triplesalpha}. $\h{W}_A, \h{W}_B\h{R}_{AB}, \h{W}_C\h{R}_{AC}$ are the whitening matrices defined in \eqref{eqn:tildeW}.
The corresponding whitening matrices $W_A, W_B R_{AB}, W_C R_{AC}$ for exact moment third order tensor
$\Ebb[\Triples^{\alpha_0}_{Y\rightarrow \{A,B,C\}}|\Pi]$
 will be defined later. Recall that $\rho$ is defined in \eqref{eqn:rho} as $\rho:= \frac{\alpha_0+1}{ \halpha_{\min}}.$
Given $\delta\in (0,1)$, throughout assume that  \beq n = \Omega\left( \rho^2 \log^2 \frac{k}{\delta}\right), \label{eqn:dimcondition-app-app}\eeq as in Assumption $(B2)$.

\begin{theorem}[Perturbation of whitened tensor]\label{theorem:tensorperturb} When the partitions $A,B,C,X,Y$ satisfy \eqref{eqn:dimcondition-app-app}, we have with probability $1-100\delta$,
  \begin{align}\nn\veps_T&:=\left\|\Triples^{\alpha_0}_{Y\rightarrow \{A,B,C\}}(\hat{W}_A, \hat{W}_B \h{R}_{AB}, \hat{W}_C \h{R}_{AC})-\Ebb[\Triples^{\alpha_0}_{Y\rightarrow \{A,B,C\}}(W_A, \tl{W}_B, \tl{W}_C)|\Pi_A, \Pi_B, \Pi_C]\right\|\\ &=
  O \left(   \frac{(\alpha_0+1) \sqrt{(\max_i (P\halpha)_i)}}{ n^{1/2}  \halpha_{\min}^{3/2} \sigma_{\min}(P)}\cdot \left(1+\left(\frac{\rho^2}{n}\log^2 \frac{k}{\delta}\right)^{1/4}\right)\sqrt{\frac{\log k}{\delta}}\right)\nn\\ &=\tl{O}\left( \frac{\rho}{\sqrt{n}} \cdot \frac{\zeta}{\halpha_{\max}^{1/2}}\right). \label{eqn:vepsTmainproof}\end{align}
\end{theorem}

\paragraph{Proof Overview: }The proof of the above result follows. It consists mainly of the following steps: (1) Controlling the perturbations of the whitening matrices and (2) Establishing concentration of the third moment tensor (before whitening). Combining the two, we can then obtain perturbation of the whitened tensor. Perturbations for the whitening step is established in Appendix~\ref{app:whitenperturb}. Auxiliary concentration bounds required for the whitening step, and for the claims below are in Appendix~\ref{app:auxconc} and~\ref{app:adjperturb}.\\

\bprfof{Theorem~\ref{theorem:tensorperturb}}In tensor $T^{\alpha_0}$ in \eqref{eqn:triplesalpha}, the first term is \[(\alpha_0+1)(\alpha_0+2)\sum_{i\in Y} \left( G^\top_{i,A}\otimes G^\top_{i,B} \otimes G^\top_{i,C}\right).\]We claim that this term dominates in the perturbation analysis since the mean vector perturbation is of lower order. We now consider perturbation of the whitened tensor\[ \Lambda_0 = \frac{1}{|Y|}\sum_{i\in Y} \left( (\h{W}_A^\top G^\top_{i,A})\otimes (\h{R}_{AB}^\top\h{W}^\top_B  G^\top_{i,B}) \otimes (\h{R}_{AC}^\top \h{W}^\top_C G^\top_{i,C})\right).\]

We show that this tensor is close to the corresponding term in the expectation in three steps.

First we show it is close to

\[ \Lambda_1 = \frac{1}{|Y|}\sum_{i\in Y} \left( (\h{W}_A^\top F_A\pi_i)\otimes (\h{R}_{AB}^\top\h{W}^\top_B  F_B \pi_i) \otimes (\h{R}_{AC}^\top \h{W}^\top_C F_C \pi_i)\right).\]

Then this vector is close to the expectation over $\Pi_Y$.

\[
\Lambda_2 = \Ebb_{\pi\sim \Dir(\alpha)} \left( (\h{W}_A^\top F_A\pi)\otimes (\h{R}_{AB}^\top\h{W}^\top_B  F_B \pi) \otimes (\h{R}_{AC}^\top \h{W}^\top_C F_C \pi)\right).
\]

Finally we replace the estimated whitening matrix $\hat{W}_A$ with $W_A$, defined in \eqref{eqn:WA}, and note that $W_A$ whitens the exact moments.
\[
\Lambda_3 = \Ebb_{\pi\sim \Dir(\alpha)} \left( (W_A^\top F_A\pi)\otimes (\tl{W}_B^\top  F_B \pi) \otimes (\tl{W}_C^\top F_C \pi)\right).
\]

For $\Lambda_0 - \Lambda_1$, the dominant term in the perturbation bound (assuming partitions $A,B,C,X,Y$ are of size $n$) is (since for any rank $1$ tensor, $\|u \otimes v \otimes w\|= \|u\|\cdot \|v\|\cdot \|w\|$),
\begin{align*}& O \left( \frac{1}{|Y|}\|\tl{W}^\top_B F_B\|^2\left\|
\sum_{i\in Y}\left(\h{W}^\top_A   G^\top_{i,A} - \h{W}_A^\top F_A \pi_i\right)\right\| \right)\\
& O \left( \frac{1}{|Y|} \halpha_{\min}^{-1}\cdot \frac{(\alpha_0+1) (\max_i (P\halpha)_i)}{   \halpha_{\min} \sigma_{\min}(P)}\cdot (1+\veps_1+\veps_2+\veps_3)\sqrt{ \log \frac{n}{\delta}}\right)
,\end{align*} with probability $1-13\delta$ (Lemma~\ref{lemma:concsumwhitevector}). Since there are $7$ terms in the third order tensor $\Triples^{\alpha_0}$, we have the bound with probability $1-91\delta$.

For $\Lambda_1 - \Lambda_2$, since $\h{W}_A F_A\Diag(\halpha)^{1/2}$ has spectral norm almost 1, by
Lemma~\ref{lemma:tensorofpi} the spectral
norm of the perturbation is at most
\begin{align*}
&\norm{\h{W}_A F_A\Diag(\halpha)^{1/2}}^3 \norm{\frac{1}{|Y|}\sum_{i\in Y} (\Diag(\halpha)^{-1/2}\pi_i)^{\otimes 3}
- \Ebb_{\pi\sim \Dir(\alpha)}(\Diag(\halpha)^{-1/2}\pi_i)^{\otimes 3}}\\ & \le O\left(\frac{1}{\halpha_{\min}\sqrt{n}}\cdot\sqrt{\log \frac{n}{\delta}}\right).
\end{align*}
For the final term $\Lambda_2-\Lambda_3$, the dominating term is
\[
(\h{W}_A - W_A) F_A\Diag(\halpha)^{1/2} \norm{\Lambda_3}
\le \veps_{W_A} \norm{\Lambda_3} \le
O\left(\frac{(\alpha_0+1) \sqrt{\max_i (P\halpha)_i }}{n^{1/2}   \halpha_{\min}^{3/2} \sigma_{\min}(P)}  (1+\veps_1+\veps_2+\veps_3)\sqrt{ \log \frac{n}{\delta}}\right)
\]
Putting all these together, the third term $\norm{\Lambda_2-\Lambda_3}$ dominates. We know with probability at least $1-100\delta$, the perturbation in the tensor
is at most \[O\left(\frac{(\alpha_0+1) \sqrt{\max_i (P\halpha)_i }}{n^{1/2}   \halpha_{\min}^{3/2} \sigma_{\min}(P)}(1+\veps_1+\veps_2+\veps_3)\sqrt{ \log \frac{n}{\delta}}\right).\]
\eprfof\\


\subsection{Whitening Matrix Perturbations}\label{app:whitenperturb}

Consider rank-$k$ SVD of
$ |X|^{-1/2}(G^{\alpha_0}_{X, A})^\top_{k-svd}= \hat{U}_A \hat{D}_A \hat{V}_A^\top,$ and the whitening matrix is given by  $\hat{W}_A:=\hat{U}_A\hat{D}_A^{-1}$ and thus $ |X|^{-1}\hat{W}_A^\top (G^{\alpha_0}_{X, A})^\top_{k-svd} (G^{\alpha_0}_{X, A})_{k-svd} \hat{W}_A = I$. Now  consider the singular value decomposition of
\[|X|^{-1} \hat{W}_A^\top \Ebb[(G^{\alpha_0}_{X, A})^\top|\Pi]\cdot
\Ebb[(G^{\alpha_0}_{X, A})|\Pi]\hat{W}_A = \Phi \tl{D} \Phi^\top.\] $\hat{W}_A$ does not whiten the exact moments in general. On the other hand,
consider
\beq\label{eqn:WA} W_A := \hat{W}_A \Phi_A \tl{D}_A^{-1/2} \Phi_A^\top. \eeq
Observe that $W_A$ whitens $|X|^{-1/2}\Ebb[(G^{\alpha_0}_{X, A})|\Pi]$
\[|X|^{-1} W_A^\top  \Ebb[(G^{\alpha_0}_{X, A})^\top|\Pi] \Ebb[(G^{\alpha_0}_{X, A})|\Pi] W_A =(\Phi_A \tl{D}_A^{-1/2} \Phi_A^\top)^\top \Phi_A \tl{D}_A \Phi_A^\top \Phi_A \tl{D}_A^{-1/2} \Phi_A^\top = I  \]
Now the ranges of $W_A$ and $\hat{W}_A$ may differ and we control the perturbations below.

Also note that $\hat{R}_{A,B}$, $\hat{R}_{A,C}$ are  given by
 \begin{align}\label{eqn:tildeW-app} \hat{R}_{AB}&:=  |X|^{-1} \h{W}_B^\top(G^{\alpha_0}_{X,B})_{k-svd}^\top (G^{\alpha_0}_{X,A})_{k-svd} \hat{W}_A.\\
 R_{AB} &:=  |X|^{-1} W_B^\top\Ebb[(G^{\alpha_0}_{X,B})^\top|\Pi]\cdot \Ebb[ G^{\alpha_0}_{X,A}|\Pi]\cdot W_A. \end{align}
Recall
 $\epsilon_G$  is given by \eqref{eqn:epsilonG}, and
 $\sigma_{\min}\left( \Ebb[G^{\alpha_0}_{X,A}|\Pi]\right)$ is given in \eqref{lemma:specnormexact} and $|A|=|B|=|X|=n$.

\begin{lemma}[Whitening matrix perturbations]\label{lemma:WAerror}With probability $1-\delta$,
\begin{align} \label{eqn:WAerror}\epsilon_{W_A} :=\|\Diag(\halpha)^{1/2}F_A^\top(\hat{W}_A - W_A) \| &=O\left( \frac{ 
(1-\veps_1)^{-1/2}\epsilon_G}{
\sigma_{\min}\left( \Ebb[G^{\alpha_0}_{X,A}|\Pi]\right)} \right)\\ \epsilon_{\tl{W}_B}:=\| \Diag(\halpha)^{1/2}F_B^\top(\h{W}_B\h{R}_{AB} - W_B R_{AB})\| &=O\left( \frac{
(1-\veps_1)^{-1/2}\epsilon_G}{
\sigma_{\min}\left( \Ebb[G^{\alpha_0}_{X,B}|\Pi]\right)} \right)\end{align} Thus, with probability $1-6\delta$,
\beq\epsilon_{W_A}=\epsilon_{\tl{W}_B}
= O\left(\frac{(\alpha_0+1) \sqrt{\max_i (P\halpha)_i }}{n^{1/2}   \halpha_{\min} \sigma_{\min}(P)}\cdot (1+\veps_1+\veps_2 +\veps_3)\right),\eeq where $\veps_1$, $\veps_2$ and $\veps_3$ are given by \eqref{eqn:veps1} and \eqref{eqn:veps}.
\end{lemma}

\paragraph{Remark: }Note that when partitions $X,A$ satisfy \eqref{eqn:dimcondition-app-app}, $\veps_1, \veps_2, \veps_3$ are small. When  $P$ is well conditioned and $\halpha_{\min}=\halpha_{\max}=1/k$, we have $\epsilon_{W_A}, \epsilon_{\tl{W}_B}  = O(k/\sqrt{n})$.\\


\bprf  Using the fact that $W_A= \h{W}_A \Phi_A \tl{D}_A^{-1/2} \Phi_A^\top$ or $\h{W}_A = W_A \Phi_A \tl{D}_A^{1/2} \Phi_A^\top$ we have that \begin{align*}
\|\Diag(\halpha)^{1/2}F_A^\top (\h{W}_A - W_A)\| &\leq \|\Diag(\halpha)^{1/2}F_A^\top W_A ( I - \Phi_A \tl{D}_A^{1/2} \Phi_A^\top)\| \\ &= \|\Diag(\halpha)^{1/2}F_A^\top W_A ( I - \tl{D}_A^{1/2})\| \\ &\leq  \|\Diag(\halpha)^{1/2}F_A^\top W_A (I - \tl{D}_A^{1/2})(I+ \tl{D}_A^{1/2})\|\\ &\leq \| \Diag(\halpha)^{1/2}F_A^\top W_A\| \cdot \| I - \tl{D}_A\|
\end{align*} using the fact that $\tl{D}_A$ is a diagonal matrix.

Now note that   $W_A$ whitens $|X|^{-1/2}\Ebb[G^{\alpha_0}_{X,A}|\Pi]=|X|^{-1/2} F_A\Diag(\alpha^{1/2}) \Psi_X,$ where $\Psi_X$ is defined in \eqref{eqn:psidef}. Further it is shown in Lemma~\ref{lemma:specnormexact} that $\Psi_X$ satisfies with probability $1-\delta$ that\[ \veps_1 := \|I -  |X|^{-1}\Psi_X\Psi_X^\top\|\le \bigO\left(\sqrt{\frac{(\alpha_0+1)}{\halpha_{\min}|X|}}\cdot \log \frac{k}{\delta}\right)\] Since $\veps_1\ll1$ when $X,A$ satisfy \eqref{eqn:dimcondition-app-app}. We have that $|X|^{-1/2}\Psi_X$  has singular values around $1$. Since $W_A$ whitens $|X|^{-1/2}\Ebb[G^{\alpha_0}_{X,A}|\Pi]$, we have
\[ |X|^{-1}  W_A^\top F_A \Diag(\alpha^{1/2}) \Psi_X \Psi_X^\top \Diag(\alpha^{1/2})F_A^\top W_A = I.\] Thus, with probability $1-\delta$,\[ \|\Diag(\halpha)^{1/2}F_A^\top W_A \| = O( (1-\veps_1)^{-1/2}).\]
Let $\Ebb[(G^{\alpha_0}_{X,A})|\Pi] = (G^{\alpha_0}_{X,A})_{k-svd} + \Delta$.
We have \begin{align*}\| I - \tl{D}_A\| &= \| I - \Phi_A \tl{D}_A \Phi_A^\top\| \\ &= \| I-|X|^{-1}\hat{W}_A^\top
\Ebb[(G^{\alpha_0}_{X,A})^\top|\Pi]\cdot \Ebb[(G^{\alpha_0}_{X,A})|\Pi]  \hat{W}_A\|\\
&= O\left(|X|^{-1}\| \h{W}_A^\top\left(\Delta^\top (G^{\alpha_0}_{X,A})_{k-svd} + \Delta (G^{\alpha_0}_{X,A})_{k-svd}^\top\right) \h{W}_A \|\right)\\
&= O\left(|X|^{-1/2}\| \h{W}_A^\top \Delta^\top \h{V}_A + \h{V}_A^\top \Delta \h{W}_A\|\right),\\
&=O\left(|X|^{-1/2}\|\h{W}_A\| \|\Delta\|\right) \\ &= O\left(|X|^{-1/2} \|W_A\| \epsilon_G\right),
\end{align*} since  $\|\Delta\|\leq \epsilon_G+ \sigma_{k+1}( G^{\alpha_0}_{X,A})\leq 2 \epsilon_G$, using Weyl's theorem for singular value perturbation
and  the fact that $\epsilon_G \cdot \|W_A\| \ll 1$ and
 $\| W_A\| = |X|^{1/2}/\sigma_{\min}\left( \Ebb[G^{\alpha_0}_{X,A}|\Pi]\right)$.

We now consider perturbation of $W_B R_{AB}$. By definition, we have that
\[ \Ebb[G^{\alpha_0}_{X,B} |\Pi]\cdot W_B R_{AB} = \Ebb[G^{\alpha_0}_{X,A} |\Pi]\cdot W_A.\] and
\[ \|W_B R_{AB}\| = |X|^{1/2}\sigma_{\min}(\Ebb[G^{\alpha_0}_{X,B} |\Pi] )^{-1}.\] Along the lines of previous derivation for $\epsilon_{W_A}$, let
\[ |X|^{-1}(\h{W}_B \h{R}_{AB})^\top \cdot \Ebb[(G^{\alpha_0}_{X,B})^\top |\Pi] \cdot \Ebb[G^{\alpha_0}_{X,B} |\Pi] \h{W}_B \h{R}_{AB} = \Phi_B \tl{D}_B \Phi_B^\top.\]
Again using the fact that $|X|^{-1}\Psi_X \Psi_X^\top\approx I$, we have
\[ \|\Diag(\halpha)^{1/2}F_B^\top W_B R_{AB} \|\approx \|\Diag(\halpha)^{1/2}F_A^\top W_A\|,\]
and the rest of the proof follows.\eprf\\

 
\subsection{Auxiliary Concentration Bounds}\label{app:auxconc} 
 
\begin{lemma}[Concentration of sum of whitened vectors]
\label{lemma:concsumwhitevector}Assuming all the partitions satisfy \eqref{eqn:dimcondition-app-app}, with probability $1-7\delta$, \begin{align*}\left\|
\sum_{i\in Y}\left(\h{W}^\top_A   G^\top_{i,A} - \h{W}_A^\top F_A \pi_i\right)\right\|&=  O(\sqrt{ |Y|\halpha_{\max}} \epsilon_{W_A})\\&=O\left(\frac{\sqrt{(\alpha_0+1)(\max_i (P\halpha)_i)}}
{ \halpha_{\min} \sigma_{\min}(P)}\cdot (1+ \veps_2+\veps_3)  \sqrt{\log n/\delta}\right)
,\\
\left\|
\sum_{i\in Y}
\left((\h{W}_B \h{R}_{AB})^\top (G^\top_{i,B} -
F_B \pi_i)\right)\right\|&=  O\left(\frac{\sqrt{(\alpha_0+1) (\max_i (P\halpha)_i)}}
{ \halpha_{\min} \sigma_{\min}(P)}\cdot (1+\veps_1+\veps_2+\veps_3)  \sqrt{\log n/\delta}\right)
.\end{align*}
\end{lemma}

\paragraph{Remark: }Note that when $P$ is well conditioned and $\halpha_{\min}=\halpha_{\max}=1/k$, we have  the above bounds as $  O(k)$. Thus, when it is normalized with $1/|Y|=1/n$, we have the bound as $O(k/n)$.\\

\bprf   Note that $\h{W}_A$ is computed using partition $X$ and $G_{i,A}$ is obtained from $i\in Y$. We have independence for edges across  different partitions $X$ and $Y$.
Let  $\Xi_{i}:=\h{W}^\top_A (
G^\top_{i,A} -   F_A \pi_i )$.
Applying matrix Bernstein's inequality to each of the variables, we have
\begin{align*} \|\Xi_{i}\| &\leq \|\h{W}_A \| \cdot \| G^\top_{i,A} -   F_A \pi_i\|
\\ &\leq \|\h{W}_A \| \sqrt{\|F_A\|_1},\end{align*}
from Lemma~\ref{lemma:adj}.
The variances are given by
\begin{align*} \|\sum_{i\in Y}\Ebb[ \Xi_{i}\Xi_{i}^\top|\Pi]\|&\leq \sum_{i\in Y} \h{W}_A^\top \Diag(F_A\pi_i) \h{W}_A,
\\ &\leq \| \h{W}_A\|^2 \|F_Y\|_1 \\
&=O\left(\frac{|Y|}{|A|}\cdot \frac{(\alpha_0+1) (\max_i (P\halpha)_i) }
{ \halpha_{\min}^2 \sigma_{\min}^2(P) }\cdot (1+ \veps_2+\veps_3) \right)
,\end{align*} with probability $1- 2\delta$
from \eqref{eqn:entrybound} and \eqref{eqn:l1normbound}, and $\veps_2, \veps_3$ are given by \eqref{eqn:veps}.
Similarly, $\|\sum_{i\in Y}\Ebb[ \Xi_{i}^\top \Xi_{i}|\Pi]\|\leq \| \h{W}_A\|^2 \|F_Y\|_1$. Thus, from matrix Bernstein's inequality,
we have with probability $1-3\delta$
\begin{align*}  \| \sum_{i\in Y} \Xi_{i}\| &= O( \|\h{W}_A\|
\sqrt{\max(\|F_A\|_1, \|F_X\|_1)}).
\\ &= O\left(\frac{\sqrt{(\alpha_0+1) (\max_i (P\halpha)_i)}}
{ \halpha_{\min} \sigma_{\min}(P)}\cdot (1+ \veps_2+\veps_3)  \sqrt{\log n/\delta}\right)\end{align*}


On similar lines, we have the result for $B$ and $C$, and also   use the independence assumption on edges in various partitions.
\eprf\\

We now show that not only the sum of whitened vectors concentrates, but that each individual whitened vector $\h{W}_A^\top G^\top_{i,A}$ concentrates, when $A$ is large enough.

\begin{lemma}[Concentration of a random whitened vector]
\label{lem:randomwhitenedvector}
Conditioned on $\pi_i$, with probability at least $1/4$,
$$ \norm{\h{W}_A^\top G_{i,A}^\top - W_A^\top F_A\pi_i} \le
O(\veps_{W_A}\halpha_{\min}^{-1/2}) =
\tl{O}\left(\frac{\sqrt{(\alpha_0+1)(\max_i (P\halpha)_i)}}
{n^{1/2}\halpha_{\min}^{3/2}\sigma_{\min}(P)}\right).$$
\end{lemma}

\paragraph{Remark:} The above result is not a high probability event since we employ Chebyshev's inequality to establish it. However, this is not an issue for us, since we will employ it to show that out of $\Theta(n)$ whitened vectors, there exists at least one good initialization vector corresponding to each eigen-direction, as required in Theorem~\ref{thm:robustpower} in Appendix~\ref{app:tensorpower}. See  Lemma~\ref{lemma:dirichletinit} for details.
%
%

\begin{proof}
We have
$$\norm{\h{W}_A^\top G_{i,A}^\top - W_A^\top F_A\pi_i}
\le \norm{(\h{W}_A-W_A)^\top  F_A \pi_i} + \norm{\h{W}_A^\top  (G_{i,A}^\top  - F_A\pi_i)}.$$
The first term is satisfies satisfies with probability $1-3\delta$
\begin{align*}
\|  (\h{W}_A^\top -W_A^\top) F_A \pi_i\| &\leq  \epsilon_{W_A} \halpha_{min}^{-1/2}
\\ &=O\left(\frac{(\alpha_0+1)  \halpha_{\max}^{1/2} \sqrt{ (\max_i (P\halpha)_i)}}{n^{1/2}   \halpha_{\min}^{3/2} \sigma_{\min}(P)}\cdot (1+\veps_1+\veps_2 +\veps_3)\right) \end{align*}
Now we  bound the second term.
Note that $G_{i,A}^\top $ is independent of $\h{W}^\top _A$, since they are related to    disjoint subset of edges.
The whitened neighborhood vector can be viewed as a sum of vectors:\[\h{W}_A^\top G_{i,A}^\top =\sum_{j\in A} G_{i,j} (\h{W}_A^\top )_j
= \sum_{j\in A} G_{i,j} (\h{D}_A \h{U}_A^\top )_j = \h{D}_A \sum_{j\in A} G_{i,j} (\h{U}_A^\top )_j.\]
Conditioned on $\pi_i$ and $F_A$, $G_{i,j}$ are Bernoulli variables with probability $(F_A\pi_i)_j$. The goal is to compute the variance of the sum, and then use
Chebyshev's inequality noted in Proposition~\ref{prop:chebyshev}. 

Note that the variance  is given by \[ \|\Ebb[(G^\top_{i,A}-F_A \pi_i)^\top\h{W}_A \h{W}_A^\top  (G^\top_{i,A}-F_A \pi_i)]\| \leq  \|\h{W}_A\|^2 \sum_{j\in A} (F_A\pi_i)_j \norm{(\h{U}_A^\top)_j}^2.\] We now bound the variance.
By Wedin's theorem, we know the span
of columns of $\h{U}_A$ is $O(\eps_G/\sigma_{\min}(G^{\alpha_0}_X,A)) = O(\eps_{W_A})$ close to the span of columns of $F_A$.
The span of columns of $F_A$ is the same as the span of rows in $\Pi_A$.
In particular, let $Proj_\Pi$ be the projection matrix of the span of rows in $\Pi_A$, we have \[\norm{\h{U}_A \h{U}_A^\top - Proj_\Pi} \le O(\eps_{W_A}).\]
Using the spectral norm bound, we have the Frobenius norm \[\norm{\h{U}_A \h{U}_A^\top - Proj_\Pi}_F \le O(\eps_{W_A}\sqrt{k})\] since they are rank $k$ matrices. This implies that
\[\sum_{j\in A} \left(\norm{(\h{U}_A^\top )_j} - \norm{Proj_\Pi^j}\right)^2=O(\eps_{W_A}^2 k).\]Now
\[ \| Proj_\Pi^j\| \leq \frac{\|\pi_j\|}{\sigma_{\min}(\Pi_A)}= \bigO\left(\sqrt{\frac{(\alpha_0+1)}{n\halpha_{\min}}}\right),\]
from Lemma~\ref{lemma:specnormexact}


Now we can bound the variance of the vectors
$\sum_{j\in A} G_{i,j} (\h{U}_A^\top )_j$, since the variance of $G_{i,j}$ is bounded by $(F_A \pi_i)_j$ (its probability), and
the variance of the vectors is at most
\begin{align*}
\sum_{j \in A} (F_A \pi_i)_j \norm{(\h{U}_A^\top )_j}^2 & \le 2 \sum_{j \in A} (F_A \pi_i)_j \norm{Proj_\Pi^j}^2 + 2  \sum_{j\in A}(F_A \pi_i)_j \left(\norm{(\h{U}_A^\top )_j} - \norm{Proj_\Pi^j}\right)^2\\
& \le 2 \sum_{j \in A} (F_A \pi_i)_j \max_{j\in A} \left(\norm{Proj_\Pi^j}^2\right) +  \max_{i,j} P_{i,j}
\sum_{j\in A} \left(\norm{(\h{U}_A^\top )_j} - \norm{Proj_\Pi^j}\right)^2 \\
& \le
\bigO\left(\frac{|F_A|_1(\alpha_0+1)}{n\halpha_{\min}}\right)
\end{align*}
%


Now Chebyshev's inequality implies that with probability at least
$1/4$ (or any other constant), \[\norm{\sum_{j\in A} (G_{i,j}-F_A\pi_i) (\h{U}_A^\top )_j}^2 \le
O\left(\frac{|F_A|_1(\alpha_0+1)}{n\halpha_{\min}}\right).\] And thus, we have
\[\h{W}_A^\top (G_{i,A}-F_A\pi_i) \le \sqrt{\frac{|F_A|_1(\alpha_0+1)}{n\halpha_{\min}}}\cdot \norm{\h{W}_A^\top }
\le O\left(\eps_{W_A}\halpha_{\min}^{-1/2}\right).\] Combining the two terms, we have the result.
\end{proof}

Finally, we establish the following perturbation bound between empirical and expected tensor under the Dirichlet distribution, which is used in the proof of Theorem~\ref{theorem:tensorperturb}.

\begin{lemma}[Concentration of third moment tensor under Dirichlet distribution]
\label{lemma:tensorofpi}
With probability $1-\delta$, for $\pi_i \simiid \Dir(\alpha)$,
\begin{align*}\norm{\frac{1}{|Y|}\sum_{i\in Y} (\Diag(\halpha)^{-1/2}\pi_i)^{\otimes 3}
- \Ebb_{\pi\sim \Dir(\alpha)}(\Diag(\halpha)^{-1/2}\pi)^{\otimes 3}} &\le O\left( \cdot \frac{1}{\halpha_{\min}\sqrt{n}}\sqrt{\log \frac{n} {\delta}}\right)\\& = \tl{O}\left(\frac{1}{\halpha_{min}\sqrt{n}}\right)\end{align*}
\end{lemma}


\begin{proof} The spectral norm of this tensor cannot be larger than the spectral norm of a $k\times k^2$ matrix
that we obtain be ``collapsing'' the last two dimensions (by definitions of norms). Let $\phi_i :=
\Diag(\hat{\alpha})^{-1/2}\pi_i$ and the ``collapsed'' tensor is   the matrix $\phi_i (\phi_i\otimes \phi_i)^\top $
(here we view $\phi_i\otimes \phi_i$ as a vector in $\R^{k^2}$). We apply Matrix Bernstein on
the matrices $Z_i = \phi_i (\phi_i\otimes\phi_i)^\top $.
Now  \[\norm{\sum_{i\in Y} \E[Z_iZ_i^\top ]} \le |Y| \max \norm{\phi}^4 \norm{\E[\phi\phi^\top ]} \le |Y|\halpha_{\min}^{-2}\] since
$\norm{\E[\phi\phi^\top ]} \le 2$.
For the other variance term $\norm{\sum_{i\in Y}\E[Z_i^\top Z_i]}$, we have
\[\norm{\sum_{i\in Y}\E[Z_i^\top Z_i]} \le |Y|\halpha_{\min} \norm{\E[(\phi\otimes\phi)(\phi\otimes\phi)^\top ]}.\]
It remains to bound the norm of
$\E[(\phi\otimes\phi)(\phi\otimes\phi)^\top ]$. We have
\[ \|\E[(\phi\otimes\phi)(\phi\otimes\phi)^\top ]\|=\sup\left( \| \Ebb[M^2]\|,\,\, \st\, M= \sum_{i,j} N_{i,j} \phi_i \phi_j^\top, \,\,\|N\|_F=1\right).\]
by definition. We now group the terms of $\Ebb[M^2]$ and bound them separately. \begin{align}\nn M^2 =& \sum_i N_{i,i}^2 \phi_i\phi_i^\top \|\phi_i\|^2 + \sum_{i\neq j} N_{i,j}^2 \phi_i \phi_j^\top \inner{\phi_i, \phi_j} \\&+ \sum_{i\neq j\neq a} N_{i,i} N_{j,a} \phi_i \phi_a^\top \inner{\phi_i, \phi_j} + \sum_{i\neq j\neq a \neq b} N_{i,j} N_{a,b} \phi_i \phi_b^\top \inner{\phi_j, \phi_a}\label{eqn:fourterms}\end{align} We bound the terms individually now.
  
$\|\phi(i)\|^4$ terms: By properties of Dirichlet distribution we know \[\E[\|\phi(i)\|^4] = \Theta(\halpha_i^{-1}) \le O(\halpha_{\min}^{-1}).\]Thus, for the first term in \eqref{eqn:fourterms}, we have 
\[ \sup_{N: \|N\|_F=1}\|\sum_i\Ebb[ N_{i,i}^2 \phi_i\phi_i^\top \|\phi_i]\|^2 \| =O(\halpha_{\min}^{-1}).\]
 
$\|\phi(i)\|^3\cdot \|\phi(j)\|$ terms: 
We have$$ \|\Ebb[\sum_{i,j} N_{i,i}N_{i,j} \phi(i)^3\phi(j)]\|\leq \Ebb[\|\phi_i\|^2 \cdot \|\phi_j\|] \le O(\sqrt{\sum_{i,j}(N_{i,i}^2
\hat{\alpha}(j)) \sum_{i,j} N_{i,j}^2 \hat{\alpha}(i)^{-1}}) \le O(\halpha_{\min}^{-1/2}).$$

$\|\phi(i)\|^2\cdot \|\phi(j)\|^2$ terms: the total number of such terms is $\bigO(k^2)$ and we have\[\Ebb[\|\phi(i)\|^2\cdot \|\phi(j)\|^2]= \Theta(1),\] and thus  the
Frobenius norm of these set of terms is smaller than $\bigO(k)$ 

$\|\phi(i)\|^2\cdot \|\phi(j)\|\cdot \|\phi(a)\|$ terms:   there are $\bigO(k^3)$ such terms, and we have
\[\|\Ebb[\phi(i)\|^2\cdot \|\phi(j)\|\cdot \|\phi(a)]\|=\Theta(\hat{\alpha}(i_2)^{1/2}\hat{\alpha}(i_3)^{1/2}).\] The Frobenius norm of this part of matrix is bounded by
$$
O\left(\sqrt{\sum_{i,j,a \in [k]} \hat{\alpha}(j)\hat{\alpha}(a)}\right) \le
O(\sqrt{k}) \sqrt{\sum_{j}\sum_{a}\halpha_{j}\halpha_{a}} \le  O(\sqrt{k}).$$

the rest: the sum is \[\E[\sum_{i\neq j\neq a\neq b}
N_{i,j}N_{a,b}\hat{\alpha}(i)^{1/2}\hat{\alpha}(j)^{1/2}\hat{\alpha}(a)^{1/2}\hat{\alpha}(b)^{1/2}].\] It is easy to
break the bounds into the product of two sums ($\sum_{i,j}$ and $\sum_{a,b}$) and then bound each one by
Cauchy-Schwartz, the result is 1.

Hence the variance term in Matrix Bernstein's inequality can be bounded by $\sigma^2 \le \bigO(n\halpha_{\min}^{-2})$,
each term has norm at most $\halpha_{\min}^{-3/2}$. When $\halpha_{\min}^{-2} < n$ we know the variance term
dominates and
the spectral norm of the difference is at most $O(\halpha_{\min}^{-1}n^{-1/2}\sqrt{\log n/\delta})$ with probability $1-\delta$.

\end{proof}

\subsection{Basic Results on Spectral Concentration of Adjacency Matrix}\label{app:adjperturb}

Let $n :=\max(|A|, |X|)$.
\begin{lemma}[Concentration of $G^{\alpha_0}_{X,A}$]When $\pi_i \sim \Dir(\alpha)$, for $i \in V$, with probability $1-4\delta$,
\beq \label{eqn:epsilonG}\epsilon_{G}:=\| G^{\alpha_0}_{X,A} - \Ebb[(G^{\alpha_0}_{X,A})^\top|
\Pi]\| = O\left( \sqrt{ (\alpha_0+1) n\cdot(\max_i (P\halpha)_i)(1+\veps_2)\log \frac{n}{\delta}}\right) \eeq
\end{lemma}

\bprf From definition of $G^{\alpha_0}_{X,A}$, we have
\begin{align*} \epsilon_G &\leq \sqrt{\alpha_0+1} \| G_{X,A} - \Ebb[G_{X,A}|\Pi]\| + (\sqrt{\alpha_0+1} -1) \sqrt{|X|} \|\mu_{X,A} - \Ebb[\mu_{X,A}|\Pi]\|.  \end{align*}
We have concentration for $\mu_{X,A}$   and  adjacency submatrix $G_{X,A}$ from Lemma~\ref{lemma:adj}.
\eprf\\

We now provide concentration bounds for adjacency sub-matrix $G_{X,A}$ from partition $X$ to $A$ and the corresponding mean vector.
Recall that $\Ebb[\mu_{X\rightarrow A}|F_A, \pi_X] = F_A \pi_X$ and $\Ebb[\mu_{X\rightarrow A}|F_A] = F_A\halpha$.

\begin{lemma}[Concentration of adjacency submatrices]\label{lemma:adj}When $\pi_i \simiid \Dir(\alpha)$ for $i\in V$, with probability $1-2\delta$,
\begin{align} \| G_{X,A} - \Ebb[G_{X,A}|\Pi]\| &= O\left( \sqrt{n\cdot(\max(\max_i (P\halpha)_i, \max_i (P^\top\halpha)_i))(1+\veps_2)\log \frac{n}{\delta}}\right).\label{eqn:adj}
\\  \label{eqn:mean-app}\|\mu_A - \Ebb[\mu_A|\Pi]\| &= O\left( \frac{1}{|X|}\sqrt{n\cdot(\max(\max_i (P\halpha)_i, \max_i (P^\top\halpha)_i))(1+\veps_2)\log \frac{n}{\delta}}\right),\end{align}where $\veps_2$ is given by \eqref{eqn:veps}.\end{lemma}


\bprf Recall $\Ebb[G_{X,A}|\Pi] = F_A \Pi_X$ and
$G_{A, X} = \Ber(F_A\Pi_X)$ where $\Ber(\cdot)$ denotes the Bernoulli random matrix with independent entries. Let \[Z_i := (G^\top_{i,A}- F_A \pi_i) e_i^\top .\] We have
$ G^\top_{X,A} - F_A \Pi_X = \sum_{i \in X} Z_i$. We apply  matrix Bernstein's inequality.

We compute the variances $\sum_i \Ebb[Z_i Z_i^\top|\Pi]$ and $\sum_i \Ebb[Z^\top_i Z_i|\Pi]$. We have that
$\sum_i  \Ebb[Z_i Z_i^\top|\Pi]$ only the diagonal terms are non-zero due to independence of Bernoulli variables, and \beq\label{eqn:entrybound}\Ebb[Z_i Z_i^\top| \Pi]\leq \Diag( F_A \pi_i)\eeq entry-wise. Thus,
\begin{align}\nn\|\sum_{i\in X}\Ebb[ Z_i Z_i^\top|\Pi]\|
&\leq \max_{a\in A} \sum_{i\in X, b\in [k]} F_A(a,b) \pi_i(b) \\ \nn&= \max_{a\in A}
\sum_{i\in X, b\in [k]} F_A(a,b) \Pi_X(b,i) \\ \nn
&\leq \max_{c\in [k]} \sum_{i\in X, b\in [k]} P(b,c) \Pi_X(b,i) \\ &= \|P^\top \Pi_X\|_\infty .\label{eqn:l1normbound}\end{align}
Similarly $\sum_{i\in X}\Ebb[Z^\top_i Z_i]= \sum_{i\in X}\Diag(\Ebb[\|G^\top_{i,A} - F_A \pi_i \|^2])\leq \|P^\top \Pi_X\|_\infty $.  On lines of Lemma~\ref{lemma:spectdirichlet}, we have  $\|P^\top \Pi_X\|_\infty=O(|X| \cdot(\max_i (P^\top\halpha)_i))$  when $|X|$ satisfies \eqref{eqn:dimcondition-app-app}.

We now bound $\|Z_i\|$. First note that the entries in $G_{i,A}$ are independent and we can use the vector Bernstein's inequality to bound $\|G_{i,A}-F_A\pi_i\|$. We have
$\max_{j\in A} |G_{i,j}-(F_A\pi_i)_j|\leq 2 $ and $\sum_j\Ebb[G_{i,j}-(F_A\pi_i)_j]^2 \leq\sum_j (F_A \pi_i)_j\leq \|F_A\|_1$. Thus with probability $1-\delta$, we have
\[\|G_{i,A}-F_A\pi_i\| \leq (1+\sqrt{8\log(1/\delta)})\sqrt{\|F_A\|_1} + 8/3\log(1/\delta).\]
Thus, we have the bound that $\|\sum_i Z_i\| = O(\max(\sqrt{\|F_A\|_1},\sqrt{\|P^\top \Pi_X\|_\infty}))$.
The concentration of the mean term follows from this result.\eprf\\

We now provide spectral bounds on $\Ebb[(G^{\alpha_0}_{X,A})^\top|
\Pi]$. Define \beq \label{eqn:psidef} \psi_i :=
\Diag(\hat{\alpha})^{-1/2}(\sqrt{\alpha_0+1} \pi_i - (\sqrt{\alpha_0+1} - 1) \mu).\eeq
Let $\Psi_X $ be the matrix with columns $\psi_i$, for $i\in X$. We have \[\Ebb[(G^{\alpha_0}_{X,A})^\top|
\Pi] =F_A \Diag(\hat{\alpha})^{1/2} \Psi_X,\] from definition of $\Ebb[(G^{\alpha_0}_{X,A})^\top|
\Pi]$.

\begin{lemma}[Spectral   bounds]\label{lemma:specnormexact}  With probability $1-\delta$, \beq\label{eqn:veps1} \veps_1 := \|I -  |X|^{-1}\Psi_X\Psi_X^\top\|\le \bigO\left(\sqrt{\frac{(\alpha_0+1)}{\halpha_{\min}|X|}}\cdot \log \frac{k}{\delta}\right)\eeq  With probability $1-2\delta$, \begin{align*}\|\Ebb[(G^{\alpha_0}_{X,A})^\top|
\Pi]\|&=O\left(\|P\| \halpha_{\max} \sqrt{|X| |A|( 1+\veps_1 +\veps_2)} \right) \\  \sigma_{\min}\left(\Ebb[(G^{\alpha_0}_{X,A})^\top|
\Pi]\right)&= \Omega\left(\halpha_{\min}\sqrt{\frac{|A| |X|}{\alpha_0+1} ( 1-\veps_1 - \veps_3)} \cdot \sigma_{\min}(P)\cdot\right), \end{align*}
where \beq\label{eqn:veps} \veps_2 :=O\left(\left( \frac{ 1   }{|A|\halpha^2_{\max}}\log \frac{k}{\delta}\right)^{1/4}\right), \quad\veps_3 :=O\left( \left(\frac{(\alpha_0+1)^2}{|A|\halpha^2_{\min}}\log \frac{k}{\delta}\right)^{1/4}\right) .\eeq\end{lemma}

\paragraph{Remark: }When partitions $X,A$ satisfy \eqref{eqn:dimcondition-app-app},    $\veps_1, \veps_2, \veps_3$  are small. \\

 \bprf  Note that $\psi_i$ is a random vector with norm
bounded by $\bigO(\sqrt{(\alpha_0+1)/\halpha_{min}})$ from Lemma~\ref{lemma:spectdirichlet} and $\E[\psi_i\psi_i^\top ] = I$.
We now prove \eqref{eqn:veps1}. using Matrix Bernstein Inequality. Each matrix $\psi_i\psi_i^\top /|X|$ has spectral norm
at most $\bigO((\alpha_0+1)/\halpha_{min}|X|)$. The variance $\sigma^2$ is bounded by $$\norm{\frac{1}{|X|^2}\E[\sum_{i\in X}
\norm{\psi_i}^2\psi_i\psi_i^\top ]} \le \norm{\frac{1}{|X|^2} \max \norm{\psi_i}^2 \E[\sum_{i\in X} \psi_i\psi_i^\top ]}
\le \bigO((\alpha_0+1)/\halpha_{\min}|X|).$$
Since $\bigO((\alpha_0+1)/\alpha_{\min} |X|) < 1$, the variance dominates in Matrix
Bernstein's inequality.

Let $B:=|X|^{-1} \Psi_X \Psi_X^\top$. We have with probability $1-\delta$,  \begin{align*} \sigma_{\min}(\Ebb[(G^{\alpha_0}_{X,A})^\top|
\Pi]) &=\sqrt{|X|\sigma_{\min}( F_A\Diag(\hat{\alpha})^{1/2} B
\Diag(\hat{\alpha})^{1/2}F_A^\top)}, \\ &= \Omega(\sqrt{\halpha_{\min}|X| (1-\epsilon_1)}\cdot
\sigma_{\min}(F_A)) .\end{align*}   From Lemma~\ref{lemma:spectdirichlet}, with probability $1-\delta$,  \[\sigma_{\min}(F_A)\geq \left(\sqrt{\frac{|A|\halpha_{\min}}{\alpha_0+1}}-O((|A|\log k/\delta)^{1/4})\right) \cdot \sigma_{\min}(P).\]
Similarly other results follow.\eprf\\ 
\subsection{Properties of Dirichlet Distribution}\label{sec:dirichlet}

In this section, we list various properties of Dirichlet distribution.

\subsubsection{Sparsity Inducing Property}

We first note that the Dirichlet distribution $\Dir(\alpha)$ is sparse depending on values of $\alpha_i$, which is shown in~\cite{Matus-sparse}.

\begin{lemma}
    \label{lemma:dirsparse}
    Let reals $\tau \in (0,1]$, $\alpha_i > 0$, $\alpha_0:=\sum_i \alpha_i$ and integers $1\leq s \leq k$ be given.
    Let $(X_i,\ldots,X_k)\sim \Dir(\alpha)$.
    Then
    \[
        \Pr\big[
            |\{i : X_i \geq \tau\}| \leq s
        \big]
        \geq 1
        - \tau^{-\alpha_0} e^{-(s+1)/3}
        - e^{-4(s+1)/9},
    \] when $s+1 < 3k$.
\end{lemma}

We now show that we obtain good initialization vectors under Dirichlet distribution.

Arrange the $\halpha_j$'s in ascending order, i.e. $\halpha_1=\halpha_{\min} \leq \halpha_2 \ldots\leq \halpha_k = \halpha_{\max}$.
Recall that columns vectors $\hat{W}_A^\top G^\top_{i,A}$, for $i \notin A$, are used as initialization vectors to the tensor power method.  We say that $u_i:=\frac{\hat{W}_A^\top G^\top_{i,A}}{\|\hat{W}_A^\top G^\top_{i,A}\|}$
 is a $(\gamma, R_0)$-good initialization vector corresponding to $j \in [k]$  if
\beq   \left|\inner{u_i, \Phi_{j}}\right| \geq R_0,
\quad \left|\inner{u_i, \Phi_{j}}\right|  - \max_{m< j}\left|\inner{u_i, \Phi_{m}}\right| \geq \gamma \left|\inner{u_i, \Phi_{j}}\right|,\eeq
where $\Phi_j:= \halpha^{1/2}_j(\tl{F}_A)_j$, where $(\tl{F}_A)_j$ is the $j^{\tha}$ column of  $\tl{F}_A := W_A^\top F_A$. Note that the $\{\Phi_j\}$ are orthonormal and are the eigenvectors to be estimated by the tensor power method.

%

\begin{lemma}[Good initialization vectors under   Dirichlet distribution]\label{lemma:dirichletinit}When $\pi_i\simiid \Dir(\alpha)$,   and $\alpha_j< 1$, let
\beq \label{eq:defDelta} \Delta := O\left(\frac{\zeta \rho}{\sqrt{n}r_0}\right).\eeq
 For $j\in [k]$, there is at least one  $(\gamma-\frac{2\Delta}{r_0-\Delta}, r_0- \Delta)$-good vector corresponding to each $\Phi_j$, for $j\in [k]$, among $\{u_i\}_{i\in [n]}$ with probability $1-9\delta$, when
\beq \label{eqn:nboundinit}n = \tl{\Omega}\left(\alpha_{\min}^{-1}   e^{r_0 \halpha_{\max}^{1/2} (\alpha_0+ c_1\sqrt{k\alpha_0})} (2k)^{r_0 c_2}\log(k/\delta)\right) ,\eeq where $c_1:=(1+\sqrt{8\log 4}) $  and $c_2:=4/3(\log 4)$, when \beq \label{eqn:gammaregime}(1-\gamma)r_0 \halpha_{\min}^{1/2}  (\alpha_0+(1+\sqrt{8 \log 4}) \sqrt{k \alpha_0}+ 4/3 (\log 4) \halpha_{\min}^{-1/2}\log 2k) >1.\eeq When $\alpha_0<1$, the bound can be improved for $r_0 \in (0.5,(\alpha_0+1)^{-1})$ and $1-\gamma\geq  \frac{1-r_0}{r_0}$ as
\beq n>\frac{(1+\alpha_0)(1-r_0 \halpha_{\min})}{\halpha_{\min}(\alpha_{\min}+1 - r_0(\alpha_0+1))} \log (k/\delta).\label{eqn:nboundinitimproved}\eeq \end{lemma}

\paragraph{Remark when $\alpha_0\geq 1$, $\alpha_0=\Theta(1)$: }When $r_0$ is chosen as $r_0= \alpha_{\max}^{-1/2} (\sqrt{\alpha_0}+ c_1\sqrt{k})^{-1}$,  the term  $ e^{r_0 \halpha_{\max}^{1/2} (\alpha_0+ c_1\sqrt{k\alpha_0})} = e$, and we require
\beq \label{eqn:nboundinit-r0}n = \tl{\Omega}\left(\alpha_{\min}^{-1}  k^{0.43}\log(k/\delta)\right) ,\quad r_0= \alpha_{\max}^{-1/2} (\sqrt{\alpha_0}+ c_1\sqrt{k})^{-1}, \eeq by substituting $c_2/c_1= 0.43$.  Moreover,   \eqref{eqn:gammaregime} is satisfied for the above choice of $r_0$ when $\gamma =\Theta(1)$.

In this case we also need $\Delta < r_0/2$, which implies
\beq \zeta = O\left(\frac{\sqrt{n}}{\rho k\halpha_{\max}}\right)\eeq

\paragraph{Remark when $\alpha_0< 1$: }In this regime, \eqref{eqn:nboundinitimproved} implies that we require $n = \Omega(\halpha_{\min}^{-1})$. Also, $r_0$ is a constant, we just need $\zeta = O(\sqrt{n}/\rho)$.

\bprf
Define $\tl{u}_i:= W_A^\top F_A \pi_i/ \|W_A^\top F_A\pi_i\|$, when whitening matrix $W_A$ and $F_A$ corresponding to exact statistics are input.

We first observe that if $\tl{u}_i$ is $(\gamma, r_0)$ good, then
$u_i$ is $(\gamma - \frac{2\Delta}{r_0-\Delta}, r_0 - \Delta)$ good.

When $\tl{u}_i$ is $(\gamma, r_0)$ good, note that 
$W_A^\top F_A \pi_i \ge \halpha_{\max}^{-1/2} r_0$ since
$\sigma_{\min}(W_A^\top F_A) =\halpha_{\max}^{-1/2}$ and $\norm{\pi_i} \ge r_0$.
Now with probability $1/4$, conditioned on $\pi_i$, we have the event $\Bc(i)$,
\begin{align*}
\Bc(i):=\{   \norm{u_i -\tl{u}_i}\leq\Delta\}, \end{align*} where $\Delta$ is given by\begin{align*}
  \Delta=  \tl{O}\left(\frac{\halpha_{\max}^{0.5}\sqrt{(\alpha_0+1)(\max_i (P\halpha)_i)}}
{r_0n^{1/2}\halpha_{\min}^{1.5}\sigma_{\min}(P)}\right) 
\end{align*}
from  Lemma~\ref{lem:randomwhitenedvector}. Thus, we have $\Pbb[\Bc(i)|\pi_i] \geq 1/4$, i.e. $\Bc(i)$ occurs with probability $1/4$ for any realization of $\pi_i$.


If we perturb a $(\gamma, r_0)$ good vector by $\Delta$ (while maintaining unit norm), then
it is still $(\gamma - \frac{2\Delta}{r_0-\Delta}, r_0 - \Delta)$ good.


We now show that the set $\{\tl{u}_i\}$ contains good initialization vectors when $n$ is large enough. Consider
$Y_i \sim \Gamma(\alpha_i , 1)$, where $\Gamma(\cdot, \cdot)$ denotes the Gamma distribution and  we have $Y/\sum_i Y_i \sim \Dir(\alpha)$. We first compute the probability that  $\tl{u}_i:= W_A^\top F_A \pi_i/ \|W_A^\top F_A\pi_i\|$ is a $(r_0, \gamma)$-good vector with respect to $j=1$ (recall that $\halpha_1=\halpha_{\min}$).
The desired event is\beq \label{eqn:desiredevent}
\Ac_1:=(\halpha_{1}^{-1/2}Y_{1} \geq r_0 \sqrt{\sum_{j}\halpha_{j}^{-1} Y^2_j})\cap (\halpha_{1}^{-1/2}Y_{1} \geq \frac{1}{1-\gamma} \max_{j > 1}\halpha_{j}^{-1/2} Y_j )\eeq
We have
\begin{align}\nn\Pbb\left[\Ac_1 \right]  &\geq \Pbb\left[ (\halpha_{\min}^{-1/2}Y_{1} \geq r_0 \sqrt{\sum_{j}\halpha_{j}^{-1} Y^2_j})\cap (Y_{1} \geq \frac{1}{1-\gamma} \max_{j > 1}  Y_j )\right] \\  \nn &\geq
\Pbb\left[(\halpha_{\min}^{-1/2}Y_1 >r_0 t)\bigcap( \sum_{j} \halpha_{j}^{-1} Y_j^2 \leq t^2)\bigcap_{j> 1}( Y_1 \leq (1-\gamma)r_0t\halpha_{\min}^{1/2}   )\right], \quad \mbox{for some }t  \\ \nn
&\geq \Pbb\left[\halpha_{\min}^{-1/2}Y_1 > r_0t  \right] \Pbb\left[
 \sum_{j}\halpha_{j}^{-1}  Y_j^2 \leq t^2 \Big|\halpha_j^{-1/2} Y_j \leq  (1-\gamma)r_0t\halpha_{\min}^{1/2}  \right] \Pbb\left[ \max_{j>1} Y_j \leq (1-\gamma)r_0t\halpha_{\min}^{1/2} \right]\\ \nn
 &\geq  \Pbb\left[\halpha_{\min}^{-1/2}Y_1 > r_0  t\right] \Pbb\left[
 \sum_{j} \halpha_{j}^{-1} Y_j^2 \leq t^2  \right] \Pbb\left[ \max_{j> 1} Y_j \leq (1-\gamma)r_0t \halpha_{\min}^{1/2}\right]
\end{align}
When $\alpha_j\leq 1$, we have \[ \Pbb[\cup_j Y_j \geq \log 2k] \leq  0.5,\] since $P(Y_j \geq t) \leq t^{\alpha_j-1}e^{-t} \leq e^{-t}$ when $t >1$ and $\alpha_j\leq 1$.
Applying vector Bernstein's inequality, we have with probability $0.5-e^{-m}$ that
\[ \| \Diag(\halpha_j^{-1/2}) (Y- \Ebb(Y))\|_2
\leq (1+\sqrt{8m}) \sqrt{k \alpha_0}+ 4/3 m \halpha_{\min}^{-1/2}\log 2k,\] since $\Ebb[  \sum_{j}\halpha_j^{-1}\Var(Y_j)] = k  \alpha_0$  since $\halpha_j = \alpha_j /\alpha_0$ and $\Var(Y_j)=\alpha_j$. Thus, we have
\[ \| \Diag(\halpha_j^{-1/2})  Y\|_2
\leq \alpha_0+(1+\sqrt{8m}) \sqrt{k \alpha_0}+ 4/3 m \halpha_{\min}^{-1/2}\log 2k,\]
since $\| \Diag(\halpha_j^{-1/2})\Ebb(Y)\|_2=
\sqrt{\sum_j \halpha_j^{-1}\alpha_j^2} = \alpha_0$. Choosing $m = \log 4$, we have with probability $1/4$ that
\begin{align}\label{eqn:choiceoft} \| \Diag(\halpha_j^{-1/2})  Y\|_2
\leq t&:= \alpha_0+(1+\sqrt{8\log 4}) \sqrt{k \alpha_0}+ 4/3(\log 4) \halpha_{\min}^{-1/2}\log 2k,\\ &= \alpha_0 + c_1 \sqrt{k\alpha_0} + c_2 \halpha_{\min}^{-1/2}\log 2k.\end{align}


We now have
\[ \Pbb\left[\halpha_{\min}^{-1/2} Y_1 >  r_0  t\right] \geq \frac{\alpha_{\min}}{4C} \left( r_0t \halpha_{\min}^{1/2} \right)^{\alpha_{\min}-1} e^{- r_0t\halpha_{\min}^{1/2} } ,\]   from Lemma~\ref{prop:dirichletgamma}.

Similarly,
\begin{align*}
 \Pbb\left[ \max_{j\neq 1} Y_j \leq \halpha_{\min}^{1/2} (1-\gamma)r_0t\right] \geq 1- \sum_j \left((1-\gamma)r_0t \halpha_{\min}^{1/2} \right)^{\sum_j \alpha_j -1} e^{- (1-\gamma)r_0 \halpha_{\min}^{1/2}  t} \geq 1- k e^{-(1-\gamma)r_0 \halpha_{\min}^{1/2} t},
\end{align*} assuming that $(1-\gamma)r_0 \halpha_{\min}^{1/2}  t >1$.

Choosing $t$ as in \eqref{eqn:choiceoft}, we have the probability of the event in \eqref{eqn:desiredevent} is greater than
\begin{align*} & \frac{\alpha_{\min} }{16C} \left( 1- \frac{ e^{-(1-\gamma)r_0 \halpha_{\min}^{1/2} (\alpha_0+ c_1\sqrt{k\alpha_0})}}{2(2k)^{(1-\gamma)r_0 c_2-1 }}\right) \frac{ e^{-r_0 \halpha_{\min}^{1/2} (\alpha_0+ c_1\sqrt{k\alpha_0})}}{(2k)^{r_0 c_2}}\left(r_0\halpha_{\min}^{1/2}(\alpha_0 + c_1 \sqrt{k\alpha_0} + c_2 \halpha_{\min}^{-1/2}\log 2k)\right)^{\alpha_{\min}-1}
\end{align*}Similarly the (marginal) probability of events $\Ac_2$  can be bounded from below by replacing $\alpha_{\min}$ with $\alpha_2$ and so on. Thus, we have
\begin{align*} \Pbb[\Ac_m]&=\tl{\Omega} \left( \alpha_{\min} \frac{ e^{-r_0 \halpha_{\max}^{1/2} (\alpha_0+ c_1\sqrt{k\alpha_0})}}{(2k)^{r_0 c_2}}\right),\end{align*} for all $m\in[k]$.


Thus, we have each of the  events $\Ac_1(i)\cap \Bc(i), \Ac_2(i)\cap \Bc(i), \ldots, \Ac_k\cap \Bc(i)$ occur at least once in $i\in [n]$ i.i.d. tries with probability
\begin{align*} &1-\Pbb\left[\bigcup_{j\in[k]}(\bigcap_{i\in [n]} (\Ac_j(i)\cap \Bc(i))^c)\right]\\ &\geq 1-\sum_{j\in [k]}\Pbb\left[\bigcap_{i\in [n]} (\Ac_j(i)-\Bc(i))^c\right]\\ &\geq 1- \sum_{j\in [k]} \exp\left[- n \Pbb(\Ac_j\cap \Bc)\right],\\   &\geq
1-k \exp\left[-n \tl{\Omega} \left( \alpha_{\min} \frac{ e^{-r_0 \halpha_{\max}^{1/2} (\alpha_0+ c_1\sqrt{k\alpha_0})}}{(2k)^{r_0 c_2}}\right)\right]\end{align*} where $\Ac_j(i)$ denotes the event that $\Ac_1$ occurs for $i^{\tha}$ trial and we have that $\Pbb[\Bc|\Ac_j] \geq 0.25$ since $\Bc$ occurs in any trial with probability $0.25$ for any realization of $\pi_i$ and the events $\Ac_j$ depend only on $\pi_i$. We use that $1-x \leq e^{-x}$ when $x\in [0,1]$.
Thus, for the event to occur with probability $1-\delta$, we require
\[ n = \tl{\Omega}\left(\alpha_{\min}^{-1}   e^{r_0 \halpha_{\max}^{1/2} (\alpha_0+ c_1\sqrt{k\alpha_0})} (2k)^{r_0 c_2}\log(1/\delta)\right) .\]

\paragraph{Improved Bound when $\alpha_0<1$: }We can improve the above bound   by directly working with the Dirichlet distribution. Let $\pi\sim \Dir(\alpha)$.   The desired event corresponding to $j =1$ is given by \[ \Ac_1= \left(\frac{\halpha_1^{-1/2}\pi_1}{\|\Diag(\halpha_i^{-1/2})\pi\|}\geq r_0\right) \bigcap_{i >1} \left(\pi_1 \geq \frac{\pi_i}{ 1-\gamma}\right).\] Thus, we have \begin{align*}\Pbb[\Ac_1] &\geq \Pbb\left[\left(\pi_1\geq r_0\right)\bigcap_{i>1} \left(\pi_i \leq (1-\gamma) r_0\right) \right]\\ &\geq
\Pbb[\pi_1\geq r_0] \Pbb \left(\bigcap_{i>1}\pi_i \leq (1-\gamma) r_0| \pi_1\geq r_0\right),
\end{align*} since $\Pbb \left(\bigcap_{i>1}\pi_i \leq (1-\gamma) r_0| \pi_1 \geq r_0\right) \geq \Pbb \left(\bigcap_{i>1}\pi_i \leq (1-\gamma) r_0\right)$.
 By properties of Dirichlet distribution, we know $\E[\pi_i] = \halpha_i$
and $\E[\pi_i^2] =\halpha_i \frac{\alpha_i+1}{\alpha_0+1}$. Let $p:= \Pr[\pi_1 \ge r_0]$. We have
\begin{align*}
\E[\pi_i^2] & = p \E[\pi_i^2|\pi_i \ge r_0] + (1-p) \E[\pi_i^2 | \pi_i < r_0]\\
& \le p + (1-p) r_0 \E[\pi_i | \pi_i < r_0]\\
& \le p + (1-p) r_0 \E[\pi_i]
\end{align*}Thus, $p \geq \frac{\halpha_{\min}(\alpha_{\min}+1-r_0(\alpha_0+1))}{(\alpha_0+1)
(1- r_0 \halpha_{\min})}$, which is useful when $r_0(\alpha_0+1)< 1$.
 Also when $\pi_1 \geq r_0$, we have that $\pi_i \leq 1-r_0$ since $\pi_i \geq 0$ and $\sum_i \pi_i =1$. Thus, choosing $1-\gamma = \frac{1-r_0}{r_0}$, we have the other conditions for $\Ac_1$ are satisfied. Also, verify that we have $\gamma<1$ when  $r_0 > 0.5$ and this is feasible when $\alpha_0<1$.\eprf\\

We now prove a result that the  entries
of $\pi_i$, which are marginals of the Dirichlet distribution, are
likely to be small in the sparse regime of the Dirichlet parameters.  Recall that the marginal distribution of $\pi_i$ is distributed as $B(\alpha_i, \alpha_0-\alpha_i)$, where $B(a,b)$ is the beta distribution and
\[ \Pbb[Z=z]\propto z^{a-1} (1-z)^{b-1}, \quad Z\sim B(a,b).\]


\begin{lemma}[Marginal Dirichlet distribution in sparse regime]
\label{lem:dirichletsparse}
For $Z\sim B(a,b)$, the following results hold:

\paragraph{Case $b \le 1$, $C\in [0,1/2]$: }
\begin{align}
\Pr[Z \ge C] & \le 8\log (1/C)\cdot  \frac{a}{a+b} \\
\E[Z\cdot \delta(Z\le C)] & \le C \cdot \E[Z] = C\cdot \frac{a}{a+b}
\end{align}

\paragraph{Case $b \ge 1$, $C \le (b+1)^{-1}$: }we have
\begin{align}
\Pr[Z \ge C] & \le a \log (1/C) \\
\E[Z\cdot \delta(Z\le C)] & \le 6 a C
\end{align}
\end{lemma}

\paragraph{Remark: }The guarantee for $b\ge 1$ is worse and this agrees with the intuition that the Dirichlet vectors are more spread out (or less sparse)  when $b=\alpha_0-\alpha_i$ is large.

\begin{proof}We have
\begin{align*}
\E[Z\cdot \delta(Z\le C)] & = \int_{0}^C \frac{1}{B(a,b)} x^a (1-x)^{b-1} dx \\
& \le \frac{(1-C)^{b-1}}{B(a,b)} \int_0^C x^a dx\\
& = \frac{(1-C)^{b-1} C^{a+1}}{(a+1)B(a,b)}
\end{align*}

For $\E[Z\cdot \delta(Z\ge C)]$, we have,
\begin{align*}
\E[Z\cdot \delta(Z\ge C)] & = \int_C^1 \frac{1}{B(a,b)} x^a (1-x)^{b-1} dx  \\
& \ge \frac{C^a}{B(a,b)} \int_C^1 (1-x)^{b-1} dx\\
& = \frac{(1-C)^b C^a} {b B(a,b)}
\end{align*}

The ratio between these two is at least
$$
\frac{\E[Z\cdot \delta(Z\ge C)]}{\E[Z\cdot \delta(Z\le C)]} \ge \frac{(1-C)(a+1)}{bC}
\ge \frac{1}{C}.
$$
The last inequality holds when $a,b<1$ and $C < 1/2$.
The sum of the two is exactly $\E[Z]$, so when
$C < 1/2$ we know $\E[Z\cdot \delta(Z\le C)] < C\cdot \E[Z]$.

Next we bound the probability $\Pr[Z\ge C]$.
Note that $\Pr[Z\ge 1/2] \le 2\E[Z] = \frac{2a}{a+b}$ by Markov's inequality.
Now we show $\Pr[Z \in [C,1/2]]$ is not much larger than $\Pr[Z \ge 1/2]$ by bounding the integrals.
$$ A = \int_{1/2}^1 x^{a-1}(1-x)^{b-1} dx
\ge \int_{1/2}^1 (1-x)^{b-1}dx = (1/2)^b/b.$$
\begin{align*}
B = \int_{C}^{1/2} x^{a-1}(1-x)^{b-1} &\le (1/2)^{b-1} \int_{C}^{1/2} x^{a-1}dx\\
& \le  (1/2)^{b-1} \frac{0.5^a - C^a}{a} \\
& \le (1/2)^{b-1} \frac{1 - (1-a\log 1/C)}{a} \\
& = (1/2)^{b-1} \log (1/C).
\end{align*}
The last inequality uses the fact that $e^x \ge 1+x$ for all $x$. Now
$$\Pr[Z\ge C] =  (1+\frac{B}{A}) \Pr[Z\ge 1/2] \le (1+2b\log (1/C)) \frac{2a}{a+b}\le 8 \log (1/C) \cdot \frac{a}{a+b} $$ and we have the result.

\paragraph{Case 2: }
When $b\ge 1$, we have an alternative bound.
We use the fact that if $X\sim \Gamma(a,1)$ and
$Y\sim \Gamma(b,1)$ then $Z \sim X/(X+Y)$.
Since $Y$ is distributed as $\Gamma(b,1)$, its
PDF is $\frac{1}{\Gamma(b)} x^{b - 1} e^{-x}$.
This is proportional to the PDF of $\Gamma(1)$ ($e^{-x}$)
multiplied by a increasing function $x^{b-1}$.

Therefore we know $\Pr[Y \ge t] \ge \Pr_{Y'\sim \Gamma(1)} [Y'\ge t] = e^{-t}$.

Now we use this bound to compute the probability that $Z \le 1/R$ for all $R\ge 1$.

This is equivalent to

\begin{align*}
\Pr[ \frac{X}{X+Y} \le \frac{1}{R}] & = \int_{0}^\infty Pr[X=x] Pr[Y\ge (R-1)X] dx \\
& \ge \int_0^\infty \frac{1}{\Gamma(a)} x^{a-1}e^{-Rx}dx \\
& = R^{-a}\int_0^\infty \frac{1}{\Gamma(a)} y^{a-1}e^{-y}dy \\
& = R^{-a}
\end{align*}
In particular, $\Pr[Z\le C] \ge C^{a}$, which means $\Pr[Z\ge C] \le 1-C^a \le a\log (1/C)$.


For $\E[Z\delta(Z<C)]$, the proof is similar as before:
$$P = \E[Z\delta(Z<C)] = \int_0^C \frac{1}{B(a,b)}x^a (1-x)^bdx
\le \frac{C^{a+1}}{B(a,b)(a+1)}$$

$$Q = \E[Z\delta(Z\ge C)] = \int_C^1\frac{1}{B(a,b)}x^a (1-x)^bdx \ge \frac{C^a(1-C)^{b+1}}{B(a,b)(b+1)}$$

Now $\E[Z\delta(Z\le C)] \le \frac{P}{Q} \E[Z]
\le 6aC$ when $C < 1/(b+1)$.
\end{proof}

\subsubsection{Norm Bounds}

\begin{lemma}[Norm Bounds under Dirichlet distribution]\label{lemma:spectdirichlet}For $\pi_i \simiid \Dir(\alpha)$ for $i\in A$, with probability $1-\delta$, we have \begin{align*}\sigma_{\min}(\Pi_A)&\geq
\sqrt{\frac{|A|\halpha_{\min}}{\alpha_0+1}} -
O((|A| \log k/\delta)^{1/4} ), \\ \norm{\Pi_A} &\leq
\sqrt{|A|\halpha_{\max}}+O((|A| \log k/\delta)^{1/4} ),\\ \kappa(\Pi_A)
&\le \sqrt{\frac{(\alpha_0+1)\halpha_{\max}}{\halpha_{\min}}} +O((|A| \log k/\delta)^{1/4} ).\end{align*} This implies that $\norm{F_A} \le \norm{P}\sqrt{|A|\halpha_{\max}}$, $\kappa(F_A)\le\bigO(\kappa(P)\sqrt{(\alpha_0+1)
\halpha_{\max}/\halpha_{min}})$. Moreover, with probability $1-\delta$
\beq \| F_A\|_1 \leq |A| \cdot\max_i ( P  \halpha)_i + O\left( \|P\| \sqrt{|A|\log \frac{|A|}{\delta}}\right) \eeq \end{lemma}

\paragraph{Remark: }When  $|A| = \Omega\left( \log \frac{k}{\delta}\left(\frac{\alpha_0+1}{ \halpha_{\min}}\right)^2\right)$, we have $ \sigma_{\min}(\Pi_A)=\Omega(
\sqrt{\frac{|A|\halpha_{\min}}{\alpha_0+1}})$ with probability $1-\delta$ for any fixed $\delta\in (0,1)$.\\

\bprf Consider $\Pi_A\Pi_A^\top  =\sum_{i\in A} \pi_i \pi_i^\top$.   \begin{align*}\frac{1}{|A|}\Ebb[\Pi_A \Pi_A^\top] =& \Ebb_{\pi\sim Dir(\alpha)}[\pi \pi^\top]\\ =& \frac{\alpha_0}{\alpha_0+1} \halpha \halpha^\top +\frac{1}{\alpha_0+1} \Diag(\halpha),\end{align*}  from Proposition~\ref{prop:dirichletmoment}. The first term is positive semi-definite so the eigenvalues of the sum are at least the eigenvalues of the second component. Smallest eigenvalue of second component gives lower bound on $\sigma_{\min}(\Ebb[\Pi_A \Pi_A^\top])$. The spectral norm of the first component is bounded by  $\frac{\alpha_0}{\alpha_0+1}\norm{\hat{\alpha}} \le \frac{\alpha_0}{\alpha_0+1}
\halpha_{\max}$, the spectral norm of second component is $\frac{1}{\alpha_0+1} \alpha_{\max}$. Thus
$\norm{\E[\Pi_A \Pi_A^\top]} \le |A| \cdot \halpha_{\max}$.

Now applying Matrix Bernstein's inequality to $\frac{1}{|A|}\sum_i \left( \pi_i \pi^\top_i - \Ebb[ \pi \pi^\top]\right)$. We have that the variance is $O(1/|A|)$. Thus with probability $1-\delta$,
\[\left\| \frac{1}{|A|}\left(\Pi_A \Pi_A^\top - \Ebb[\Pi_A \Pi_A^\top]\right)\right\|=O\left(\sqrt{\frac{\log (k/\delta)}{|A|}}\right).\]
  For the  result on $F$, we use the property that for any two matrices $A,B$, $\norm{AB} \le \norm{A}\norm{B}$ and $\kappa(AB) \le \kappa(A)\kappa(B)$.

To show bound on $\|F_A\|_1$, note that each column of $F_A$ satisfies $\Ebb[(F_A)_i] = \inner{\halpha, (P)_i} 1^\top $, and thus $\|\Ebb[F_A] \|_1 \leq |A| \max_i (P  \halpha)_i$. Using Bernstein's inequality, for each column  of $F_A$, we have, with probability $1-\delta$,
\[  \left| \,\| (F_A)_i\|_1 - |A|\inner{\halpha, (P)^i}\right| = O\left( \|P\| \sqrt{|A|\log \frac{|A|}{\delta}}\right),\] by applying Bernstein's inequality, since  $|\inner{\halpha, (P)^i}| \leq \|P\|,$ and thus we have $\sum_{i\in A}\|\Ebb[(P)^j  \pi_i \pi^\top_i ((P)^j)^\top]\|,$ and  $\sum_{i\in A}\|\Ebb[ \pi_i^\top((P)^j)^\top (P)^j \pi_i ]\| \leq |A|\cdot \|P\|.$
\eprf\\

\subsubsection{Properties of Gamma and Dirichlet Distributions}

Recall Gamma distribution $\Gamma(\alpha, \beta)$ is a distribution on nonnegative real values with density function
 $\frac{\beta^\alpha}{\Gamma(\alpha)} x^{\alpha-1}e^{-\beta x}$.

\begin{proposition}[Dirichlet and Gamma distributions]\label{prop:dirichletgamma}
The following facts are known for Dirichlet distribution and Gamma distribution.

\begin{enumerate}
\item Let $Y_i \sim \Gamma(\alpha_i, 1)$ be independent random variables, then the vector $(Y_1, Y_2, ..., Y_k)/\sum_{i=1}^k Y_k$
is distributed as $Dir(\alpha)$.
\item The $\Gamma$ function satisfies Euler's reflection formula: $\Gamma(1-z)\Gamma(z) \le \pi/\sin \pi_z$.
\item The $\Gamma(z) \ge 1$ when $0 < z < 1$.
\item There exists a universal constant $C$ such that $\Gamma(z) \le C/z$ when $0 < z < 1$.
    \item  For $Y\sim \Gamma(\alpha,1)$ and $t> 0$ and $\alpha\in (0,1)$,  we have
    \begin{equation}  \frac{\alpha}{4C}t^{\alpha-1} e^{-t} \leq \Pr[Y\geq t]  \leq t^{\alpha-1} e^{-t},
        \label{eq:gamma_cdf_bounds}
    \end{equation}     and for any $\eta, c > 1$, we have \beq\label{eqn:gammabound} \Pbb[Y> \eta t
|Y\geq t] \geq (c \eta)^{\alpha -1} e^{-(\eta-1) t}.\eeq
\end{enumerate}
\end{proposition}

\bprf The  bounds in \eqref{eq:gamma_cdf_bounds} is derived using the fact that $1\leq \Gamma(\alpha) \leq C/\alpha$ when $\alpha\in (0,1)$ and
$$\int_t^\infty \frac{1}{\Gamma(\alpha_i)} x^{\alpha_i-1}e^{-x}dx \le \frac{1}{\Gamma(\alpha_i)} \int_t^\infty
t^{\alpha_i-1}e^{-x}dx \le t^{\alpha_i-1} e^{-t}, $$
and
$$\int_t^\infty \frac{1}{\Gamma(\alpha_i)} x^{\alpha_i-1}e^{-x}dx \ge \frac{1}{\Gamma(\alpha_i)} \int_t^{2t}
x^{\alpha_i-1}e^{-x}dx \ge  \alpha_i/C \int_t^{2t} (2t)^{\alpha_i-1}e^{-x}dx \ge \frac{ \alpha_i}{4C}t^{\alpha_i-1}e^{-t}
.$$  \eprf\\

\begin{proposition}[Moments  under Dirichlet distribution]
\label{prop:dirichletmoment}
Suppose $v\sim Dir(\alpha)$, the moments of $v$ satisfies the following formulas:
\begin{eqnarray*}
\E[v_i] & = & \frac{\alpha_i}{\alpha_0} \\
\E[v_i^2] & = & \frac{\alpha_i(\alpha_i+1)}{\alpha_0(\alpha_0+1)} \\
\E[v_iv_j] & = & \frac{\alpha_i\alpha_j}{\alpha_0(\alpha_0+1)}, \quad i\neq j. \\
\end{eqnarray*}

More generally, if $a^{(t)} = \prod_{i=0}^{t-1} (a+i)$, then we have
$$
\E[\prod_{i=1}^k v_i^{(a_i)}] = \frac{\prod_{i=1}^k \alpha_i^{(a_i)}}{\alpha_0^{(\sum_{i=1}^k a_i)}}.
$$
\end{proposition}

\subsection{Standard Results}

\paragraph{Bernstein's Inequalities: }
One of the key tools we use is the standard matrix Bernstein inequality~\cite[thm. 1.4]{tropp2012user}.

\begin{proposition}[Matrix Bernstein Inequality]
Suppose $Z = \sum_j W_j$ where
\begin{enumerate}\itemsep 0pt
\item $W_j$ are independent random matrices with dimension $d_1\times d_2$,
\item $\E[W_j] = 0$ for all $j$,
\item $\norm{W_j} \le R$ almost surely.
\end{enumerate}

Let $d = d_1+d_2$, and $\sigma^2 = \max\left\{\norm{\sum_j\E[W_jW_j^\top ]},\norm{\sum_j\E[W_j^\top W_j]} \right\}$, then
we have

$$
\Pr[\norm{Z} \ge t] \le d \cdot exp\left\{\frac{-t^2/2}{\sigma^2 + Rt/3}\right\}.
$$
\end{proposition}

\begin{proposition}[Vector Bernstein Inequality]\label{prop:vectorbernstein}
Let $z = (z_1, z_2, ..., z_n) \in \R^n$ be a random vector with independent entries, $\E[z_i] = 0$,
$\E[z_i^2] = \sigma_i^2$, and $\Pr[|z_i| \le 1] = 1$. Let $A = [a_1|a_2|\cdots|a_n] \in \R^{m\times n}$ be
a matrix, then

$$
\Pr[\norm{Az} \le (1+\sqrt{8t})\sqrt{\sum_{i=1}^n \norm{a_i}^2 \sigma_i^2} +(4/3)\max_{i\in[n]}
\norm{a_i} t] \ge 1-e^{-t}.
$$
\end{proposition}

\paragraph{Vector Chebyshev inequality: }We will require a vector version of the Chebyshev inequality~\cite{ferentios1982tcebycheff}.

\begin{proposition}\label{prop:chebyshev}
Let $z = (z_1, z_2, ..., z_n) \in \R^n$ be a random vector with independent entries, $\E[z_i] = \mu$,
$\sigma :=\| \Diag( \Ebb[(z-\mu)^\top (z-\mu)])\|$. Then we have that
\[ \Pbb[\| z-\mu\|>t \sigma] \leq t^{-2}.\]
\end{proposition}

\paragraph{Wedin's theorem: }
We make use of Wedin's theorem to control subspace perturbations.

\begin{lemma}[Wedin's theorem; Theorem 4.4, p.~262 in~\cite{SS90}.] \label{lemma:wedinB}
Let $A, E \in \R^{m \times n}$ with $m \geq n$ be given.
Let $A$ have the singular value decomposition
\[
\left[ \begin{array}{c} U_1^\top \\ U_2^\top \\ U_3^\top \end{array} \right]
A \left[ \begin{array}{cc} V_1 & V_2 \end{array} \right]
=
\left[ \begin{array}{cc} \Sigma_1 & 0 \\ 0 & \Sigma_2 \\ 0 & 0 \end{array}
\right]
.
\]
Let $\tl A := A + E$, with analogous singular value decomposition
$(\tl U_1, \tl U_2, \tl U_3, \tl \Sigma_1, \tl \Sigma_2, \tl V_1 \tl V_2)$.
Let $\Phi$ be the matrix of canonical angles between $\range(U_1)$ and
$\range(\tl U_1)$, and $\Theta$ be the matrix of canonical angles between
$\range(V_1)$ and $\range(\tl V_1)$.
If there exists $\delta, \alpha > 0$ such that
$\min_i \sigma_i(\tl \Sigma_1) \geq \alpha + \delta$ and
$\max_i \sigma_i(\Sigma_2) \leq \alpha$, then
$$
\max\{\|\sin \Phi\|_2, \|\sin \Theta\|_2\} \leq \frac{\|E\|_2}{\delta}.
$$
\end{lemma}

\end{document}